\documentclass[accepted]{uai2022} 

\usepackage[american]{babel}

\usepackage{mathtools} 
\usepackage{booktabs} 
\usepackage{tikz} 

\usepackage[numbers]{natbib}

\makeatletter
\newcommand{\rmnum}[1]{\romannumeral #1}
\newcommand{\Rmnum}[1]{\expandafter\@slowromancap\romannumeral #1@}
\makeatother

\usepackage{booktabs} 
\usepackage{tikz} 

\usepackage{caption}
\usepackage{subcaption}
\usepackage{multicol}
\usepackage{wrapfig}
\usepackage{placeins}
\usepackage{hyperref} 

\usepackage{algorithm}
\usepackage{algorithmic}
\usepackage{hyperref}       
\usepackage{amsthm}
\usepackage{amsfonts}       
\usepackage{mathtools}
\usepackage{float} 
\usepackage{bold-extra}
\usepackage{textcase}

\theoremstyle{definition}
\newtheorem{theorem}{Theorem}
\newtheorem{lemma}{Lemma}
\newtheorem{claim}{Claim}
\newtheorem{corollary}{Corollary}
\newtheorem{definition}{Definition}

\theoremstyle{remark}
\newtheorem{remark}{\textbf{Remark}}

\title{Is Simple Uniform Sampling Effective for Center-Based Clustering with Outliers: When and Why?}

\author[1,2]{Jiawei Huang}
\author[1]{Wenjie Liu}
\author[1]{Hu Ding}
\affil[1]{%
School of Computer Science and Technology, University of Science and Technology of China.
}
\affil[2]{%
Department of Computer Science, City University of Hong Kong. 

\{hjw0330, lwj1217\}@mail.ustc.edu.cn, huding@ustc.edu.cn 
}

\begin{document}
\maketitle

\begin{abstract}
Real-world datasets often contain outliers, and the presence of outliers can make the clustering problems to be much more challenging. In this paper, we propose a simple uniform sampling framework for solving several representative center-based clustering with outliers problems:  $k$-center/median/means clustering with outliers. Our analysis is fundamentally different from the previous (uniform and non-uniform) sampling based ideas. To explain the effectiveness of uniform sampling in theory, we introduce a measure of  ``significance''  and prove that the performance of our framework depends on the significance degree of the given instance. In particular, the sample size can be independent of the input data size $n$ and the dimensionality $d$, if we assume the given instance is  ``significant'', which is in fact a   reasonable assumption in practice. Due to its simplicity, the uniform sampling approach also enjoys several significant advantages over the non-uniform sampling approaches in practice. To the best of our knowledge, this is the first work that systematically studies the effectiveness of uniform sampling from both theoretical and experimental aspects. 
\end{abstract}

\section{Introduction}
\label{sec-intro}
\noindent  Clustering has many important applications in real world~\citep{jain2010data}. An important type of clustering problems is called ``center-based clustering'', such as the well-known {\em $k$-center/median/means} clustering problems~\citep{awasthi2014center}. 
In general, a center-based clustering problem aims to find $k$ cluster centers so as to minimize the induced clustering cost.
But in practice our datasets often contain outliers which can seriously destroy the final clustering results~\citep{georgogiannis2016robust}. 
A key obstacle is that the outliers can be arbitrarily located in the space. {\em Outlier removal} is a topic that has been extensively studied  before~\citep{chandola2009anomaly}. The existing outlier removal methods (like DBSCAN~\citep{ester1996density,bhattacharjee2021survey}) yet cannot simply replace the center-based clustering with outliers approaches, since we often need to use the obtained cluster centers to represent the data or perform other tasks, {\em e.g.,} data compression and data selection for large-scale machine learning~\citep{DBLP:conf/iclr/SenerS18,coleman2020selection}.  

Clustering with outliers can be viewed as a generalization of the vanilla clustering problems, however, the presence  of arbitrary outliers makes the problems to be much more challenging. In particular, the problem of clustering with outliers is often a challenging combinatorial optimization problem~\citep{charikar2001algorithms,gupta2017local}. 
Even for their approximate solutions, most existing quality-guaranteed algorithms have super-linear time complexities. 
So a number of sampling methods have been proposed for reducing their complexities. Namely, we take a small sample (uniformly or non-uniformly) from the input and run an existing approximation algorithm on the sample.

A number of  \textbf{non-uniform sampling methods} have been proposed, such as the {\em $k$-means++} seeding~\citep{gupta2017local,DBLP:journals/corr/abs-2003-02433,NEURIPS2019_73983c01,DBLP:conf/uai/DeshpandeKP20} and the successive sampling method~\citep{DBLP:conf/nips/ChenA018}. However, these non-uniform sampling approaches may suffer several drawbacks in practice. For instance, they need to read the input dataset in multiple passes with high computational complexities, and/or may have to discard more than $z$ outliers ($z$ is the pre-specified number of outliers). For example, the methods~\citep{gupta2017local,DBLP:journals/corr/abs-2003-02433} need to run the $k$-means++ seeding procedure $\Omega(k+z)$ rounds, and discard $O(kz\log n)$ and $O(kz)$ outliers respectively. 

Compared with the non-uniform sampling methods, \textbf{uniform sampling} enjoys several significant advantages, {\em e.g.,} it is very easy to implement in practice.  
We need to emphasize that, beyond the clustering quality and asymptotic complexity, \textbf{the simplicity of implementation is also a major concern in {\em algorithm engineering}, especially for dealing with big data}~\citep{DBLP:conf/gi/Sanders14}. 
However, uniform sampling usually requires a large sample size depending 
on $n/z$ and the dimensionality $d$, which could be very high ({\em e.g.,} $d$ could be high and $z$ could be much smaller than $n$).  For example, as the result proposed by~\citet{huang2018epsilon}, the sample size should be at least 
$\Omega\big((\frac{n}{\delta z})^2kd\log k\big)$, if $\delta\in (0,1)$ and we allow to discard  $(1+\delta)z$ outliers (which is slightly larger than the pre-specified number $z$). If $z=\sqrt{n}$, their sample size will be even larger than $n$.

Actually, for the vanilla $k$-median/means clustering (without outliers) problems, ~\citet[Section 4 and 5]{meyerson2004k} considered using uniform sampling to solve the practical case that each optimal cluster is assumed to have size $\Omega(\frac{\epsilon n}{k})$, where $\epsilon$ is a fixed parameter in $(0,1)$. But unfortunately their result cannot be easily extended to the version with outliers. As  studied in the same paper~\citep[Section 6]{meyerson2004k}, it was shown that the sample size relies on $\frac{n}{z}$, and moreover, the solution has a large error on the number of discarded outliers (it needs to discard $>16z$ outliers).

\textbf{Our contributions and key ideas.} In this paper, we aim to provide the theoretical analysis to explain the effectiveness of uniform sampling in practice. 
Our  algorithmic framework  actually is very simple, where we just need to take a small sample uniformly at random from the input and run an existing approximation algorithm (as the black-box algorithm) on the sample. Our results are partly inspired by the work from~\citet{meyerson2004k}, but as discussed above, we need to go deeper and develop significantly new ideas for analyzing  the influence of outliers. 
We briefly introduce our ideas and results below. 

Suppose $\{C^*_1, C^*_2, \cdots, C^*_k\}$ are the $k$ optimal clusters.  
We consider the ratio between two  values: the size of the smallest optimal cluster $\min_{1\leq j\leq k}|C^*_j|$\footnote{$|C^*_j|$ indicates the number of points in $C^*_j$.} and the number of outliers $z$.  
We show this ratio actually plays a key role to the effectiveness of uniform sampling.   
In real applications, a cluster usually represents a certain scale of population and it is rare to have a cluster with the size much smaller than the number of outliers. 
 So it is  reasonable to assume that each cluster $C^*_j$ has a size at least comparable to $z$, {\em i.e.,} $\min_{1\leq j\leq k}|C^*_j|/z=\Omega(1)$ (note this assumption allows the ratio to be a value smaller than $1$, say $0.5$). 
 
 Under such a realistic assumption, our framework can output $k+O(\log k)$ cluster centers that yield a $4$-approximation for $k$-center clustering with outliers. 
In some scenarios, we may insist on returning exactly $k$ cluster centers.  We prove that, if $\frac{\min_{1\leq j\leq k}|C^*_j|}{z}>1$, our framework can return exactly $k$ cluster centers that yield a $(c+2)$-approximate solution, through running an existing $c$-approximation $k$-center with outliers algorithm (with some $c\geq 1$) on the sample. The framework  also yields similar results for $k$-median/means clustering with outliers.  

We prove that the sample size can be independent of the ratio $n/z$ and the dimensionality $d$, which is fundamentally different from the previous results on uniform sampling~\citep{charikar2003better,huang2018epsilon,DBLP:conf/esa/DingYW19}. Our sample size only depends on $k$, $\min_{1\leq j\leq k}|C^*_j|$, and several parameters controlling the clustering error and success probability. 
Moreover, different from the previous methods which often have the errors on the number of discarded outliers, \textbf{our method allows to discard exactly $z$ outliers}.
Also, if we only require to output the cluster centers, our uniform sampling approach has the sublinear time complexity that is independent of the input size.

\begin{remark}
It is worth noting that \citet{gupta2018approximation} proposed a similar uniform sampling approach to solve $k$-means clustering with outliers. 
However, the analysis and results of~\citet{gupta2018approximation} are quite different from ours. Also the assumption in~\citep{gupta2018approximation} is stronger: it requires that each optimal cluster has size roughly $\Omega(\frac{z}{\gamma^2}\log k)$ ({\em i.e.,} $\frac{\min_{1\leq j\leq k}|C^*_j|}{z}=\Omega(\frac{\log k}{\gamma^2})$) where $\gamma$ is a small parameter in $(0,1)$.
\end{remark}

\subsection{Other related works}
\label{sec-prior}
We overview the  related works below. 

\textbf{$k$-center clustering with outliers.} 
~\citet{charikar2001algorithms} proposed a $3$-approximation algorithm for $k$-center clustering with outliers in arbitrary metrics. 
The time complexity of their algorithm is at least quadratic in data size, since it needs to read all the pairwise distances. A following streaming $(4+\epsilon)$-approximation algorithm was proposed by ~\citet{mccutchen2008streaming}. 
~\citet{DBLP:conf/icalp/ChakrabartyGK16} showed a $2$-approximation algorithm for metric $k$-center clustering with outliers based on the LP relaxation techniques. 
~\citet{BHI} showed a coreset based approach but having an exponential time complexity if $k$ is not a constant. 
 Recently, ~\citet{DBLP:conf/esa/DingYW19} provided a greedy algorithm that yields a bi-criteria approximation (returning more than $k$ clusters) based on the idea from~\citep{G85}; independently, ~\citet{NEURIPS2019_73983c01} also proposed a similar greedy bi-criteria approximation algorithm. Both their results yield $2$ (or $2+\epsilon$)-approximation for the radius and need to return $O(\frac{k}{\epsilon})$ cluster centers. 
Several non-uniform and uniform sampling methods were  studied in~\citep{charikar2003better,DBLP:journals/corr/abs-1802-09205,huang2018epsilon}. 


\textbf{$k$-median/means clustering with outliers.} The early work of $k$-means clustering with outliers dates back to 1990s, in which ~\citet{10.2307/2242558} used the ``trimming''  idea to formulate the robust model for $k$-means clustering. \citet{georgogiannis2016robust} provided a theoretical analysis of the robustness of $k$-means clustering with respect to outliers. 
There are a bunch of  algorithms with provable guarantees that have been proposed  for $k$-means/median clustering with outliers~\citep{charikar2001algorithms,chen2008constant,krishnaswamy2018constant,friggstad2018approximation}, but they are difficult to implement due to their high complexities. Several heuristic but practical  algorithms without provable guarantees have also been studied, such as \citep{chawla2013k,ott2014integrated}. \citet{liu2019clustering} provided a new objective function through combining clustering and outlier removal. \citet{paul2021uniform} recently proposed a robust center-based clustering method based on ``Median-of-Means'' which can be solved via gradient-based algorithms. 

Moreover,  several sampling based methods  were proposed to speed up the existing algorithms. 
By using the local search method, \citet{gupta2017local} provided a $274$-approximation algorithm for $k$-means clustering with outliers; they showed that the well known $k$-means++ method~\citep{arthur2007k} can be used to reduce the complexity. Furthermore, ~\citet{NEURIPS2019_73983c01} and ~\citet{DBLP:conf/uai/DeshpandeKP20} respectively showed that the quality can be improved by modifying the $k$-means++ seeding. 
Partly inspired by the successive sampling method of~\citet{mettu2004optimal}, ~\citet{DBLP:conf/nips/ChenA018} proposed a novel summary construction algorithm to reduce input data size.   \citet{DBLP:journals/corr/abs-2003-02433} provided a sampling method by combining $k$-means++ and uniform sampling. \citet{charikar2003better} and ~\citet{meyerson2004k} provided different uniform sampling approaches respectively; ~\citet{huang2018epsilon}, ~\citet{DBLP:conf/icml/DingW20} also presented the uniform sampling approaches in Euclidean space. Recently ~\citet{huang2022near} obtained a robust coreset for the clustering with outliers problem based on the framework of \citet{braverman2022power}. 

\textbf{Sublinear time algorithms.} Uniform sampling is closely related to sublinear algorithms. A number of uniform sampling based sublinear time algorithms have been studied for the problems like clustering~\citep{DBLP:conf/stoc/Indyk99,mishra2001sublinear,czumaj2004sublinear} and property testing~\citep{DBLP:journals/jacm/GoldreichGR98}. 

\subsection{Preliminaries}
\label{sec-pre}

In this paper, we follow the common definition for ``uniform sampling'' in most previous articles, that is,  we take a sample from the input \textbf{independently and uniformly at random}. 

We use $||\cdot||$ to denote the Euclidean norm. Let the input be a point set $P \subset \mathbb{R}^d$ with $|P|=n$. Given a set of points $H\subset \mathbb{R}^d$ and a positive integer $z<n$, we  let $\texttt{dist}(p, H)\coloneqq\min_{q\in H}||p-q||$ for any point $p\in\mathbb{R}^d$, and define the following notations: 
\begin{itemize}
\item $\Delta^{-z}_{\infty}(P, H)=\min\{\max_{p\in P'}\texttt{dist}(p, H)\mid P'\subset P, |P'|=n-z\}$; 
\item  $\Delta^{-z}_{l}(P, H) = \min\{\frac{1}{|P'|}\sum_{p\in P'}\big(\texttt{dist}(p, H)\big)^l \mid P' \subset  P, |P'|=n-z\}$,   $l \in \{ 1, 2 \}$. 
\end{itemize}


The following definition follows the previous articles (mentioned in Section~\hyperref[sec-prior]{1.1}) on center-based clustering with outliers. As emphasized earlier, we suppose that the outliers  can be arbitrarily located in the space.

\begin{definition}[\textbf{$k$-Center/Median/Means clustering with outliers}]
\label{def-outlier}
Given a set $P$ of $n$ points in $\mathbb{R}^d$ with two positive integers $k<n$ and $z<n$, the problem of $k$-center ({\em resp.,} $k$-median, $k$-means) clustering with outliers is to find $k$ cluster centers $C=\{c_1, \cdots, c_k\}\subset \mathbb{R}^d$, such that the objective function $\Delta^{-z}_{\infty}(P, C)$ ({\em resp.,} $\Delta^{-z}_{1}(P, C)$, $\Delta^{-z}_{2}(P, C)$) is minimized.
\end{definition}

\begin{remark}
\label{rem-def}
Definition~\ref{def-outlier} can be simply modified for arbitrary metric space $(X, \mu)$, where $X$ contains $n$ vertices and $\mu(\cdot, \cdot)$ is the distance function: the Euclidean distance ``$||p-q||$'' is replaced by $\mu(p, q)$; the cluster centers $\{c_1, \cdots, c_k\}$ should be chosen from $X$. 

\end{remark}

In this paper, we always use $P_{\mathtt{opt}}$, a subset of $P$ with size $n-z$, to denote the subset yielding the optimal solution with respect to the objective functions in Definition~\ref{def-outlier}. Additionally, let $\{C^*_1, \cdots, C^*_k\}$ be the $k$ optimal clusters forming $P_{\mathtt{opt}}$. Also, we introduce the following definition for analyzing our algorithms. 


\begin{definition}[\textbf{$\bf{(\epsilon_1, \epsilon_2)}$-Significant instance}]
\label{def-sig}
Let $\epsilon_1, \epsilon_2>0$. Given an instance of $k$-center ({\em resp.,} $k$-median, $k$-means) clustering with outliers as described in Definition~\ref{def-outlier}, if $\min_{1\leq j\leq k} |C^*_j|\geq\frac{\epsilon_1}{k}n$ and $z=\frac{\epsilon_2}{k}n$, we say that it is an $(\epsilon_1,\epsilon_2)$-significant instance.
\end{definition}

\begin{remark}
We assume that  $\min_{1\leq j\leq k} |C^*_j|$ has a lower bound $\frac{\epsilon_1}{k}n$ in Definition~\ref{def-sig} (this assumption was also proposed by~\citet{meyerson2004k} before). 
The ratio $\frac{\epsilon_1}{\epsilon_2}$ ($\leq\frac{\min_{1\leq j\leq k} |C^*_j|}{z}$), measures  the ``significance'' degree of the clusters to outliers. Specifically, the higher the ratio, the more significant  the clusters. In practice, a cluster usually represents a certain scale of population and it is rare to have a cluster with the size much smaller than the number of outliers. So it is reasonable to assume $\frac{\epsilon_1}{\epsilon_2}=\Omega(1)$ in real scenarios.  

\end{remark}
 
 To help our analysis, we introduce the following two important implications  of Definition~\ref{def-sig}. 

\begin{lemma}
\label{lem-imp1}
Given an $(\epsilon_1, \epsilon_2)$-significant instance $P$ as described in Definition~\ref{def-sig}, we select a set $S$ of points from $P$ uniformly at random. Let $\eta, \delta\in(0,1)$.  Then we have: 
\begin{enumerate} 
\item If $|S|\geq \frac{k}{\epsilon_1}\log\frac{k}{\eta}$, with probability at least $1-\eta$, $S\cap C^*_j\neq \emptyset$ for any $1\leq j\leq k$.  
\item  If $|S|\geq\frac{3k}{\delta^2\epsilon_1}\log\frac{2k}{\eta}$, with probability at least $1-\eta$, $|S\cap C^*_j|\in (1\pm\delta)\frac{|C^*_j|}{n} |S|$ for any $1\leq j\leq k$.
\end{enumerate}
\end{lemma}


Lemma~\ref{lem-imp1} can be directly obtained through the following claim. We just replace $\eta$ by $\eta/k$ in Claim~\ref{cla-sample}, because we need to take the union bound over all the $k$ clusters.

\begin{claim}
\label{cla-sample}
Let $U$ be a set of elements and $V\subseteq U$ with $\frac{|V|}{|U|}=\tau>0$. Given $\eta, \delta\in(0,1)$, we uniformly select a set $S$ of elements from $U$ at random. Then we have:
\begin{itemize}
\item (\rmnum{1}) if $|S|\geq \frac{1}{\tau}\log\frac{1}{\eta}$, with probability at least $1-\eta$, $S$ contains at least one element from $V$;
\item (\rmnum{2}) if $|S|\geq\frac{3}{\delta^2\tau}\log\frac{2}{\eta}$, with probability at least $1-\eta$, we have $\big||S\cap V|-\tau |S|\big|\leq \delta\tau |S|$.
\end{itemize}
\end{claim}
\begin{proof}
Actually, (\rmnum{1}) is a folklore result that has been presented in several papers before (such as~\cite{DX14}). Since each sampled element falls in $V$ with probability $\tau$, we know that the sample $S$ contains at least one element from $V$ with probability $1-(1-\tau)^{|S|}$. Therefore, if we want $1-(1-\tau)^{|S|}\geq 1-\eta$, $|S|$ should be at least $\frac{\log 1/\eta}{\log 1/(1-\tau)}\leq\frac{1}{\tau}\log\frac{1}{\eta}$.

(\rmnum{2}) can be proved by using the Chernoff bound~\citep{alon2004probabilistic}. Define $|S|$ random variables $\{y_1, \cdots, y_{|S|}\}$: for each $1\leq i\leq |S|$, $y_i=1$ if the $i$-th sampled element falls in $V$, otherwise, $y_i=0$. So $E[y_i]=\tau$ for each $y_i$. As a consequence, we have
\begin{align}
\textbf{Pr}\Big[\big|\sum^{|S|}_{i=1}y_i-\tau |S|\big|\leq \delta\tau|S|\Big]\geq 1-2e^{-\frac{\delta^2\tau}{3}|S|}. 
\end{align}
If $|S|\geq\frac{3}{\delta^2\tau}\log\frac{2}{\eta}$, with probability at least $1-\eta$, $\big|\sum^{|S|}_{i=1}y_i \! - \! \tau |S|\big| \! \leq \! \delta\tau|S|$ ({\em i.e.,} $\big||S\cap V| \! - \! \tau |S|\big|\leq \delta\tau |S|$).
\end{proof}

Also, we know that the expected number of outliers contained in the sample $S$ is $\frac{\epsilon_2}{k}|S|$. So we immediately have the following result by using the Markov's inequality.

\begin{lemma}
\label{lem-imp2}
Given an $(\epsilon_1, \epsilon_2)$-significant instance $P$ as described in Definition~\ref{def-sig}, we select a set $S$ of points from $P$ uniformly at random. Let $\eta\in(0,1)$. Then, with probability at least $1-\eta$, $\big|S\setminus P_{\mathtt{opt}}\big|\leq  \frac{\epsilon_2}{k\eta}|S|$.
\end{lemma}


\section{$k$-Center clustering with outliers}
\label{sec-kcenter}
In this section, we  focus on the problem of $k$-center clustering with outliers in Euclidean space, and the results also hold for   abstract metric space by using exactly the same idea. 

\vspace{0.05in}
\textbf{High-level idea.} The two algorithms of Section~\hyperref[sec-kc-first]{2.1} and Section~\hyperref[sec-nips-kcenter-2]{2.2} both follow the simple uniform sampling framework: take a small  sample $S$ from the input, and run an existing black-box algorithm $\mathcal{A}$ on $S$ to obtain the solution. In Algorithm~\ref{alg-kc1}, the sample size $|S|$ is relatively smaller, and thus $S$ contains only a small number of outliers, which is roughly $O(\log k)$, from $P\setminus P_{\mathtt{opt}}$. Therefore, we can run a $\big(k+O(\log k)\big)$-center clustering algorithm (as the algorithm $\mathcal{A}$) on $S$, so as to achieve a constant factor approximate solution (in terms of the radius). In Algorithm~\ref{alg-kc2}, under the assumption $\frac{\epsilon_1}{\epsilon_2}>1$, we can enlarge the sample size $|S|$ and safely output only $k$ (instead of $k+O(\log k)$) cluster centers, also by running a black-box algorithm $\mathcal{A}$. The obtained approximation ratio is $c+2$ if   $\mathcal{A}$ is a $c$-approximation algorithm with some $c\geq 1$. For example, if we apply the $3$-approximation algorithm from~\citet{charikar2001algorithms}, our Algorithm~\ref{alg-kc2} will yield a $5$-approximate solution. 


\subsection{The First Algorithm}
\label{sec-kc-first}

For ease of presentation, we let $r_{\mathtt{opt}}$ be the optimal radius of the instance $P$, {\em i.e.,} each optimal cluster $C^*_j$ is covered by a ball with radius $r_{\mathtt{opt}}$. For any point $p\in \mathbb{R}^d$ and any value $r\geq 0$, we use $\texttt{Ball}(p, r)$ to denote the ball centered at $p$ with radius $r$. 

\begin{theorem}
\label{the-kc1}
In Algorithm~\ref{alg-kc1}, the number of returned cluster centers $|H|=k+k'= k+\frac{1}{\eta}\frac{\epsilon_2}{\epsilon_1}\log \frac{k}{\eta}$. Also, with probability at least $(1-\eta)^2$, $\Delta^{-z}_{\infty}(P, H)\leq 4r_{\mathtt{opt}}$. 
\end{theorem}

\begin{remark}
\label{rem-the-kc1}
The sample size $|S|=\frac{k}{\epsilon_1}\log\frac{k}{\eta}$ depends on $k$, $\epsilon_1$, and $\eta$ only. If $\frac{\epsilon_2}{\epsilon_1}=O(1)$ and $\eta$ is assumed to be a fixed constant in $(0,1)$, $k'$ will be $O(\log k)$. Also, the runtime of the algorithm~\citep{G85} used in Step 2 is $O((k+k')|S|d) =O(\frac{k^2}{\epsilon_1}(\log k) d)$, which is independent of the input size $n$. 
\end{remark}

\begin{algorithm}[hbt!]
   \caption{\textsc{Uni-$k$-Center Outliers \Rmnum{1}}}
   \label{alg-kc1}
\begin{algorithmic}
  \STATE {\bfseries Input:} An $(\epsilon_1, \epsilon_2)$-significant instance $P$ of $k$-center clustering with $z$ outliers, and $|P|=n$; a parameter $\eta\in (0,1)$. 
   \STATE
\begin{enumerate}
\item Sample a set $S$ of $\frac{k}{\epsilon_1}\log\frac{k}{\eta}$ points uniformly at random from $P$.
\item Let $k'=\frac{1}{\eta}\frac{\epsilon_2}{k}|S|$, and solve the $(k+k')$-center clustering problem on $S$ by using the $2$-approximation algorithm~\citep{G85}. 
\end{enumerate}
  \STATE {\bfseries Output}  $H$, which is the set of $k+k'$ cluster centers returned in Step 2. 
\end{algorithmic}
\end{algorithm}


\begin{proof}(\textbf{of Theorem~\ref{the-kc1}})
%
First,  it is straightforward to know that $|H|=k+k'= k+\frac{1}{\eta}\frac{\epsilon_2}{\epsilon_1}\log \frac{k}{\eta}$. 
Below, we assume that the sample $S$ contains at least one point from each $C^*_j$, and at most $k'=\frac{\epsilon_2}{k\eta}|S|$ points from $P\setminus P_{\mathtt{opt}}$. These events happen with probability at least $(1-\eta)^2$ according to Lemma~\ref{lem-imp1} and Lemma~\ref{lem-imp2}. It is worth noting that the events  described in Lemma~\ref{lem-imp1} and Lemma~\ref{lem-imp2} are not completely independent. For example, if the event (1) of Lemma~\ref{lem-imp1} occurs, then $\big|S\setminus P_{\mathtt{opt}}\big|$ should be at most $\frac{\epsilon_2}{k\eta}(|S|-k)$ instead of $\frac{\epsilon_2}{k\eta}|S|$. But since $\frac{\epsilon_2}{k\eta}|S|>\frac{\epsilon_2}{k\eta}(|S|-k)$, we can still say ``$\big|S\setminus P_{\mathtt{opt}}\big|\leq  \frac{\epsilon_2}{k\eta}|S|$'' with probability at least $1-\eta$. So we can safely claim that the overall probability is at least $(1-\eta)^2$.

Since the sample $S$ contains at most $k'$ points from $P\setminus P_{\mathtt{opt}}$ and $P_{\mathtt{opt}}$ can be covered by $k$ balls with radius $r_{\mathtt{opt}}$, we know that $S$ can be covered by $k+k'$ balls with radius $r_{\mathtt{opt}}$. Thus, if we perform the $2$-approximation $(k+k')$-center clustering algorithm~\citep{G85} on $S$, the obtained balls should have radius no larger than $2r_{\mathtt{opt}}$. Let $H=\{h_1, \cdots, h_{k+k'}\}$ and $\mathbb{B}_S=\{\texttt{Ball}(h_l, r)\mid 1\leq l\leq k+k'\}$ be those balls covering $S$ with $r\leq 2r_{\mathtt{opt}}$. Also, for each $1\leq j\leq k$, since $S\cap C^*_j\neq\emptyset$, there exists one ball of $\mathbb{B}_S$, say $\texttt{Ball}(h_{l_j}, r)$, covers at least one point, say $p_j$, from $C^*_j$. For any point $p\in C^*_j$, we have $||p-p_j||\leq 2 r_{\mathtt{opt}}$ (by the triangle inequality) and $||p_j-h_{l_j}||\leq r\leq 2 r_{\mathtt{opt}}$; therefore, 
\begin{eqnarray}
 ||p-h_{l_j}||\leq ||p-p_j||+||p_j-h_{l_j}||\leq 4 r_{\mathtt{opt}}.
 \end{eqnarray} 
 
 \begin{figure}[h]
 \begin{center}
    \includegraphics[width=0.45\linewidth]{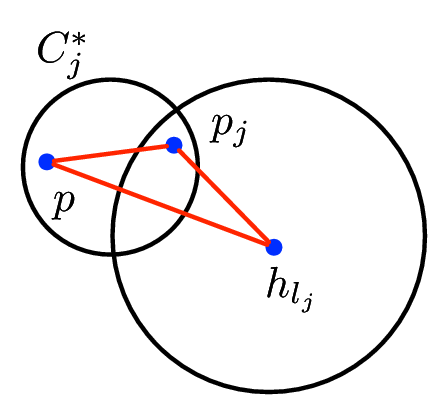}  
    \end{center}
  \vspace{-10pt}
  \caption{$||p-h_{l_j}||\leq ||p-p_j||+||p_j-h_{l_j}||\leq 4 r_{\mathtt{opt}}$.}     
   \label{fig-th1}
\end{figure}

 See Figure~\ref{fig-th1} for an illustration. Overall, $P_{\mathtt{opt}}=\cup^k_{j=1}C^*_j$ is covered by the union of the balls $\cup^{k+k'}_{l=1}\texttt{Ball}(h_l, 4 r_{\mathtt{opt}})$, {\em i.e.,} $\Delta^{-z}_{\infty}(P, H)\leq 4r_{\mathtt{opt}}$. 
 \end{proof}

\textbf{An ``extreme'' example for Theorem~\ref{the-kc1}.}  We present an example to show that the value of $k'$ in Step~2 cannot be reduced. Namely, the clustering quality could be arbitrarily bad if we run $(k+k'')$-center clustering on $S$ with $k''<k'$. Let $P$ be an $(\epsilon_1, \epsilon_2)$-significant instance in $\mathbb{R}^d$, where each optimal cluster $C^*_j$ is a set of $|C^*_j|$ overlapping points located at its cluster center $o^*_j$. Let $x>0$ and $y \gg x$. We assume (1) $||o^*_{j_1}-o^*_{j_2}||=  x$, $\forall j_1\neq j_2$; (2) $||q_1-q_2||\geq y$, $\forall q_1, q_2\in P\setminus P_{\mathtt{opt}}$; (3) $||o^*_j-q||\geq y$, $\forall 1\leq j\leq k, q\in P\setminus P_{\mathtt{opt}}$. 
Obviously, the optimal radius $r_{\mathtt{opt}}=0$. Suppose we obtain a sample $S$ satisfying $S\cap C^*_j\neq\emptyset$ for any $1\leq j\leq k$ and $|S\setminus P_{\mathtt{opt}}|=\frac{1}{\eta}\frac{\epsilon_2}{k}|S|=k'$. Given a number $k''<k'$, we run $(k+k'')$-center clustering on $S$. Since the points of $S$ take $k+k'$ distinct locations in the space,  any $(k+k'')$-center clustering on $S$ will yield a radius  at least $x/2>0$ (because $y \gg x$, it forces to select the points from $S\setminus P_{\mathtt{opt}}$ as the cluster centers, and therefore there must exist two points of $S\cap P_{\mathtt{opt}}$ falling into one cluster); thus the approximation ratio is at least  $\frac{x/2}{0}=+\infty$. 

\subsection{The Second Algorithm}
\label{sec-nips-kcenter-2}

We present the second algorithm (Algorithm~\ref{alg-kc2}) and analyze its quality in this section. 

\begin{theorem}
\label{theorem-kc2}
If $\frac{\epsilon_1}{\epsilon_2}>\frac{1}{\eta(1-\delta)}$, with probability at least $(1-\eta)^2$, Algorithm~\ref{alg-kc2} returns $k$ cluster centers that achieve a $(c+2)$-approximation for  $k$-center clustering with $z$ outliers, {\em i.e.,} $\Delta^{-z}_{\infty}(P, H)\leq (c+2)r_{\mathtt{opt}}$.
\end{theorem}
\begin{remark}
\label{rem-the-kc2}
 As an example, if we set $\eta=\delta=1/2$, the algorithm works for any instance with $\frac{\epsilon_1}{\epsilon_2}>\frac{1}{\eta(1-\delta)}=4$. Actually, as long as $\frac{\epsilon_1}{\epsilon_2}>1$ ({\em i.e.,} $\min_{1\leq j\leq k}|C^*_j|>z$), we can always find the appropriate values for $\eta$ and $\delta$ to satisfy $\frac{\epsilon_1}{\epsilon_2}>\frac{1}{\eta(1-\delta)}$, {\em e.g.,} we can set $\eta=\sqrt{\frac{\epsilon_2}{\epsilon_1}}$ and $\delta<1-\sqrt{\frac{\epsilon_2}{\epsilon_1}}$. Obviously, if $\frac{\epsilon_1}{\epsilon_2}$ is close to $1$, the success probability $(1-\eta)^2$ could be small. To boost the success probability, we repeat the algorithm multiple times and select the best one in our experiments of Section~\ref{sec-exp}. By repeatedly running the algorithm, we can achieve a constant success probability, as stated in Corollary~\ref{cor-kc2-repeat}. This also implies an important observation: \textbf{the larger the ratio $\frac{\epsilon_1}{\epsilon_2}$ is , the more effectively uniform sampling performs. }
\end{remark}

\begin{corollary}
    \label{cor-kc2-repeat}
    By executing Algorithm~\ref{alg-kc1} $O(\frac{1}{(1-\eta)^2})$ times, with constant probability, there exists at least one time where the returned $k$ centers $H$ satisfy $\Delta^{-z}_{\infty}(P, H)\leq (c+2)r_{\mathtt{opt}}$. 
\end{corollary}

\begin{proof}(\textbf{of Theorem~\ref{theorem-kc2}}) Similar to the proof of Theorem~\ref{the-kc1}, we assume that $|S\cap C^*_j|\in(1\pm\delta)\frac{|C^*_j|}{n} |S|$ for each $C^*_j$, and $S$ has at most $\hat{z}=\frac{\epsilon_2}{k\eta}|S|$ points from $P\setminus P_{\mathtt{opt}}$.
Let $\mathbb{B}_S=\{\texttt{Ball}(h_l, r)\mid 1\leq l\leq k\}$ be the set of $k$ balls returned in Step 2 of Algorithm~\ref{alg-kc2}. Since $S\cap P_{\mathtt{opt}}$ can be covered by $k$ balls with radius $r_{\mathtt{opt}}$ and $|S\setminus P_{\mathtt{opt}}|\leq \hat{z}$, the optimal radius for the instance $S$ with $\hat{z}$ outliers should be at most $r_{\mathtt{opt}}$. 
Consequently, $r\leq c r_{\mathtt{opt}}$. Moreover, we have 
\begin{eqnarray}
|S\cap C^*_j| &\geq& (1-\delta)\frac{|C^*_j|}{n} |S| \nonumber\\
&\geq& (1-\delta)\frac{\epsilon_1}{k}|S| \nonumber \\ 
& >& (1-\delta)\frac{\epsilon_2}{\eta(1-\delta)k}|S| \nonumber\\
&=& \frac{\epsilon_2}{\eta k}|S|=\hat{z} \label{for-kc2-1}
\end{eqnarray}  
for any $1\leq j\leq k$, where the last inequality comes from the assumption $\frac{\epsilon_1}{\epsilon_2}>\frac{1}{\eta(1-\delta)}$. Thus, if we perform $k$-center clustering with $\hat{z}$ outliers on $S$, the obtained $k$ balls must cover at least one point from each $C^*_j$ (since $|S\cap C^*_j|>\hat{z}$ from (\ref{for-kc2-1})). Through a similar manner in the proof of Theorem~\ref{the-kc1}, we know that $P_{\mathtt{opt}}=\cup^k_{j=1}C^*_j$ is covered by the union of the balls $\cup^{k}_{l=1}\texttt{Ball}(h_l, r+2 r_{\mathtt{opt}})$, {\em i.e.,}
\begin{eqnarray}
 \Delta^{-z}_{\infty}(P, H)\leq r+2 r_{\mathtt{opt}}\leq (c+2)r_{\mathtt{opt}}.
\end{eqnarray}
\end{proof}

\textbf{Runtime:} Similar to Algorithm~\ref{alg-kc1}, the sample size $|S|=\frac{3k}{\delta^2\epsilon_1}\log\frac{2k}{\eta}$ depends on $k$, $\epsilon_1$, and the parameters $\delta$ and $\eta$ only. 
The  runtime  depends on the complexity of the subroutine $c$-approximation algorithm used in Step~2. For example, the algorithm of~\citet{charikar2001algorithms} takes 
 $O\big(|S|^2d+k|S|^2\log |S|\big)$ 
 time in $\mathbb{R}^d$.

\begin{algorithm}[tb]
   \caption{\textsc{Uni-$k$-Center Outliers \Rmnum{2}}}
   \label{alg-kc2}
\begin{algorithmic}
  \STATE {\bfseries Input:}  An $(\epsilon_1, \epsilon_2)$-significant instance $P$ of $k$-center clustering with $z$ outliers, and $|P|=n$; two parameters $\eta,\delta\in (0,1)$. 
   \STATE
\begin{enumerate}
\item Sample a set $S$ of $\frac{3k}{\delta^2\epsilon_1}\log\frac{2k}{\eta}$ points uniformly at random from $P$.
\item Let $\hat{z}=\frac{1}{\eta}\frac{\epsilon_2}{k}|S|$, and solve the $k$-center clustering with $\hat{z}$ outliers problem on $S$ by using any $c$-approximation algorithm with $c\geq1$ ({\em e.g.,} the $3$-approximation algorithm~\citep{charikar2001algorithms}). 
\end{enumerate}
  \STATE {\bfseries Output} $H$, which is the set of $k$ cluster centers returned in Step~2.
\end{algorithmic}
\end{algorithm}

\section{$k$-Median/Means clustering with outliers}
\label{sec-kmedian}
For $k$-means clustering with outliers, 
we apply the similar uniform sampling ideas as Algorithm~\ref{alg-kc1} and~\ref{alg-kc2}. 
However, the analyses are more complicated here. For ease of understanding, we present our high-level idea first. 
Also,  we provide the extensions for  $k$-median clustering with outliers and their counterparts in arbitrary metric space in appendix.

\vspace{0.1in}
\textbf{High-level idea.} Recall $\{C^*_1, C^*_2, \cdots, C^*_k\}$ are the $k$ optimal clusters. 
 Denote by $O^*=\{o^*_1, \cdots, o^*_k\}$ the mean points of $\{C^*_1, \cdots, C^*_k\}$, respectively. 
 We  define the following transformation on $P_{\mathtt{opt}}$ to help us analyzing the clustering errors. 
For each point in $C^*_j$, we translate it to $o^*_j$; overall, we generate a new set of $n-z$ points located at $\{o^*_1, \cdots o^*_k\}$, where each $o^*_j$ has $|C^*_j|$ overlapping points. 
\textbf{(1)} For any point $p\in C^*_j$ with $1\leq j\leq k$, denote by $\tilde{p}$ its transformed point. 
\textbf{(2)} For any $U\subseteq P_{\mathtt{opt}}$, denote by $\tilde{U}$ its transformed point set. Since the transformation forms $k$ ``stars'' (see Figure~\ref{fig-key}), we call it ``\textbf{star shaped transformation}''.


\begin{figure}[h]
  \vspace{-0.05in}
\begin{center}
    \includegraphics[width=0.66\linewidth]{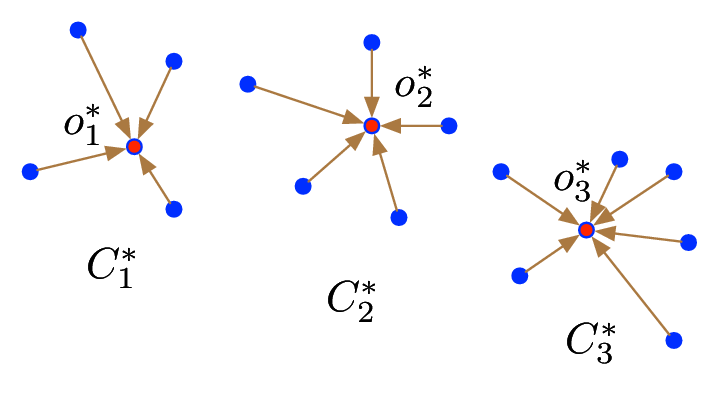}  
    \end{center}
  \vspace{-0.3in}
  \caption{The transformation from $P_{\mathtt{opt}}$ to $\tilde{P}_{\mathtt{opt}}$.}     
   \label{fig-key}
     \vspace{-0.05in}
\end{figure}


 Let $S$ be a sufficiently large random sample from $P$.  We first show that $S\cap C^*_j$ can well approximate $C^*_j$ for each $1\leq j\leq k$. Informally, 


%
%

\vspace{-0.05in}
\begin{equation}
 \hspace{-0.3in}\left.
 \begin{aligned}
  \frac{|S\cap C^*_j|}{|S|}     &\approx \frac{|C^*_j|}{n};  \\
 \frac{1}{|S  \! \cap \!  C^*_j|}\sum_{q\in S\cap C^*_j}||q \! - \! o^*_j||^2   &\approx\frac{1}{|C^*_j|}\sum_{p\in  C^*_j}||p \! - \! o^*_j||^2. 
               \end{aligned}
  \hspace{0in}\right  \}  \hspace{-0.23in} 
  \label{for-key}
\end{equation}
\vspace{-0.05in}



%
Let $S_{\mathtt{opt}}=S\cap P_{\mathtt{opt}}$. By using (\ref{for-key}), we can prove that the clustering costs of $\tilde{P}_{\mathtt{opt}}$ and $\tilde{S}_{\mathtt{opt}}$ are close (after the normalization) for any given set of cluster centers in the Euclidean space. Similar to Algorithm~\ref{alg-kc1}, we compute the $(k+k')$-means clustering on the sample $S$ in Algorithm~\ref{alg-km}, where $k'$ is roughly $O(\frac{\log k}{\xi^2})$ with a parameter $\xi\in (0,1)$. Let $H$ be the returned set of $k+k'$ cluster centers. 
Then, we can use $\tilde{S}_{\mathtt{opt}}$ and $\tilde{P}_{\mathtt{opt}}$ as the ``bridges'' to connect $S$ and $P$, so as to prove the theoretical quality guarantee of Algorithm~\ref{alg-km}. 
%

In the second algorithm (similar to  Algorithm~\ref{alg-kc2}), we run a $k$-means with $\hat{z}$ outliers algorithm on the sample $S$ and return exactly $k$ (rather than $k+k'$) cluster centers. 
Let $S_{\mathtt{in}}$ be the set of obtained $|S|-\hat{z}$ inliers of $S$. If $\frac{\epsilon_1}{\epsilon_2}>1$, we can prove that $|S_{\mathtt{in}}\cap C^*_j|\approx |S\cap C^*_j|$ for each $1\leq j\leq k$. Therefore, we can replace ``$S$'' by ``$S_{\mathtt{in}}$'' in (\ref{for-key}) and prove a similar quality guarantee. 

 We present Algorithm~\ref{alg-km} and Algorithm~\ref{alg-km2} to realize the above ideas, and provide the main theorems below. In Theorem~\ref{the-km}, $\mathcal{L}$ denotes the maximum diameter of the $k$ clusters $C^*_1, \cdots, C^*_k$, {\em i.e.,} $\mathcal{L}$$=\max_{1\leq j\leq k}\max_{p, q\in C^*_j}$ $||p-q||$. 
 Actually, our result can be viewed as an extension of the sublinear time $k$-means clustering algorithms~\citep{mishra2001sublinear,czumaj2004sublinear} (who also have the additive clustering cost errors) to the case with outliers. We need to emphasize that the additive error is unavoidable even for the case without outliers, if we require the sample complexity to be independent of the input size~\citep{mishra2001sublinear,czumaj2004sublinear}. Though the $k$-median clustering with outliers algorithm in~\citet[Section 6]{meyerson2004k} does not yield an additive error, as mentioned in Section~\ref{sec-intro}, it needs to discard more than $16z$ outliers and the sample size depends on the ratio $n/z$.

\begin{algorithm}[tb]
   \caption{\textsc{Uni-$k$-Means Outliers \Rmnum{1}}}
   \label{alg-km}
\begin{algorithmic}
  \STATE {\bfseries Input:} An $(\epsilon_1, \epsilon_2)$-significant instance $P$ of $k$-means clustering with $z$ outliers in $\mathbb{R}^d$, and $|P|=n$; three parameters $\eta, \delta, \xi\in (0,1)$. 
   \STATE
\begin{enumerate}
\item Take a uniform sample $S$ of $\max\{\frac{3k}{\delta^2\epsilon_1}\log\frac{2k}{\eta},$ $ \frac{k}{2\xi^2\epsilon_1(1-\delta)}\log\frac{2k}{\eta}\}$ points   from $P$.
\item Let $k'=\frac{1}{\eta}\frac{\epsilon_2}{k}|S|$, and solve the $(k+k')$-means clustering on $S$ by using any $c$-approximation algorithm with $c\geq 1$.  
\end{enumerate}
  \STATE {\bfseries Output} $H$, which is the set of $k+k'$ cluster centers returned in Step 2. 
\end{algorithmic}
\end{algorithm}

\begin{theorem}
\label{the-km}
Let $0<\delta, \eta, \xi<1$. 
With probability at least $(1-\eta)^3$, the set of cluster centers $H$ returned by Algorithm~\ref{alg-km} yields a clustering cost $\Delta^{-z}_{2}(P, H)\leq\alpha \Delta^{-z}_{2}(P, O^*)+\beta\xi \mathcal{L}^2$, where $\alpha=\big(2+(4+4c)\frac{1+\delta}{1-\delta}\big)$ and $\beta=(4+4c)\frac{1+\delta}{1-\delta}$.
\end{theorem}

\begin{remark}
\label{rem-km}
\textbf{(\rmnum{1})} In Step 2 of Algorithm~\ref{alg-km}, we can apply an $O(1)$-approximation $k$-means algorithm ({\em e.g.,} \citep{kanungo2004local}). If we assume $1/\delta$ and $1/\eta$ are fixed constants, 
then  the sample size $|S|=O( \frac{k}{\xi^2\epsilon_1}\log k)$, and both the factors $\alpha$ and $\beta$ are $O(1)$, {\em i.e.,} 
\begin{eqnarray}
\Delta^{-z}_{2}(P, H)\leq O(1)\cdot\Delta^{-z}_{2}(P, O^*) + O(\xi)\cdot \mathcal{L}^2. 
\end{eqnarray}
The additive error $O(\xi)\cdot \mathcal{L}^2$ converges to $0$ when $\xi$ goes to $0$. Moreover, the number of returned cluster centers $|H|=k+O(\frac{\log k}{\xi^2})$ if $\frac{\epsilon_1}{\epsilon_2}=\Omega(1)$. 

\textbf{(\rmnum{2})} It is also easy to see that the extreme example proposed at the end of Section~\hyperref[sec-kc-first]{2.1} also works for Algorithm~\ref{alg-km}, {\em i.e.,}  the value of $k'$ cannot be reduced. 
 \end{remark}

Before proving Theorem~\ref{the-km}, we need to introduce the following lemmas.

\begin{lemma}
\label{lem-km-sample1}
We fix a cluster $C^*_j$. Given $\eta, \xi\in (0,1)$, if one uniformly selects a set $T$ of $\frac{1}{2\xi^2}\log\frac{2}{\eta}$ or more points at random from $C^*_j$, then
\begin{align}
\bigg|\frac{1}{|T|}\sum_{q\in T}||q-o^*_j||^2-\frac{1}{|C^*_j|}\sum_{p\in C^*_j}||p-o^*_j||^2\bigg|\leq \xi \mathcal{L}^2
\end{align}
with probability at least $1-\eta$.
\end{lemma}
Lemma~\ref{lem-km-sample1} can be obtained via the Hoeffding's inequality (each $||q-o^*_j||^2$ can be viewed as a random variable between $0$ and $\mathcal{L}^2$)~\citep{alon2004probabilistic}.

\begin{lemma}
\label{lem-km-sample2}
If one uniformly selects a set $S$ of $\max\{\frac{3k}{\delta^2\epsilon_1}\log\frac{2k}{\eta}, $ $\frac{k}{2\xi^2\epsilon_1(1-\delta)}\log\frac{2k}{\eta}\}$ points at random from $P$, then
\begin{align}
\sum_{q\in S\cap C^*_j}&||q-o^*_j||^2\nonumber\\
&\leq (1+\delta)\frac{|S|}{n}\big(\sum_{p\in C^*_j}||p-o^*_j||^2+\xi |C^*_j| \mathcal{L}^2\big)
\end{align}
for each $1\leq j\leq k$, with probability at least $(1-\eta)^2$.
\end{lemma}
\begin{proof}
Suppose 
$$|S|=\max\{\frac{3k}{\delta^2\epsilon_1}\log\frac{2k}{\eta}, \frac{k}{2\xi^2\epsilon_1(1-\delta)}\log\frac{2k}{\eta}\}.$$
According to Lemma~\ref{lem-imp1}, $|S|\geq\frac{3k}{\delta^2\epsilon_1}\log\frac{2k}{\eta}$ implies  
\begin{eqnarray}
|S\cap C^*_j|\geq (1-\delta)\frac{|C^*_j|}{n}|S|\geq (1-\delta)\frac{\epsilon_1}{k}|S| \label{for-lem-km-sample3-2}
\end{eqnarray}
for each $1\leq j\leq k$, with probability at least $1-\eta$. Below, we assume (\ref{for-lem-km-sample3-2}) occurs. Further, $|S|\geq\frac{k}{2\xi^2\epsilon_1(1-\delta)}\log\frac{2k}{\eta}$ implies 
\begin{eqnarray}
(1-\delta)\frac{\epsilon_1}{k}|S|\geq \frac{1}{2\xi^2}\log\frac{2k}{\eta}. \label{for-lem-km-sample3-3}
\end{eqnarray}
Combining (\ref{for-lem-km-sample3-2}) and (\ref{for-lem-km-sample3-3}), we have $|S\cap C^*_j|\geq\frac{1}{2\xi^2}\log\frac{2k}{\eta}$. Consequently, through Lemma~\ref{lem-km-sample1} we obtain
\begin{align}
\bigg|\frac{1}{|S\cap C^*_j|}\sum_{q\in S\cap C^*_j}||q-o^*_j||^2-\frac{1}{|C^*_j|}\sum_{p\in C^*_j}||p-o^*_j||^2\bigg|\leq \xi\mathcal{L}^2 \label{for-km-sample3-1}
\end{align}
for each $1\leq j\leq k$, with probability at least $1-\eta$ ($\eta$ is replaced by $\eta/k$ in  Lemma~\ref{lem-km-sample1} for taking the union bound). From (\ref{for-km-sample3-1}) we have
\begin{align}
&\sum_{q\in S\cap C^*_j}||q-o^*_j||^2\leq |S\cap C^*_j|\big(\frac{1}{|C^*_j|}\sum_{p\in C^*_j}||p-o^*_j||^2+\xi \mathcal{L}^2\big)\nonumber\\
&\leq(1+\delta)\frac{|C^*_j|}{n}|S|\big(\frac{1}{|C^*_j|}\sum_{p\in C^*_j}||p-o^*_j||^2+\xi \mathcal{L}^2\big)\nonumber\\
&=(1+\delta)\frac{|S|}{n}\big(\sum_{p\in C^*_j}||p-o^*_j||^2+\xi |C^*_j| \mathcal{L}^2\big),
\end{align}
where the second inequality comes from Lemma~\ref{lem-imp1}. The overall success probability $(1-\eta)^2$ comes from the success probabilities of (\ref{for-lem-km-sample3-2}) and (\ref{for-km-sample3-1}). So we complete the proof.
\end{proof}


For ease of presentation, we define a new notation that is used in the following lemmas. Given two point sets $X$ and $Y\subset\mathbb{R}^d$, we use $\texttt{\bf{Cost}}(X, Y)$ to denote the clustering cost of $X$ by taking $Y$ as the cluster centers, {\em i.e.,} $\texttt{\bf{Cost}}(X, Y)=\sum_{q\in X}(\texttt{dist}(q, Y))^2$. Obviously, $\Delta^{-z}_2(P, H)=\frac{1}{n-z}\texttt{\bf{Cost}}(P_{\mathtt{opt}}, H)$. Let $S_{\mathtt{opt}}=S\cap P_{\mathtt{opt}}$. 
Below, we prove the upper bounds of $\texttt{\bf{Cost}}(S_{\mathtt{opt}}, O^*)$, $\texttt{\bf{Cost}}(\tilde{S}_{\mathtt{opt}}, H)$, and $\texttt{\bf{Cost}}(\tilde{P}_{\mathtt{opt}}, H)$ respectively, and use these bounds to complete the proof of Theorem~\ref{the-km}. For convenience, we always assume that the events described in Lemma~\ref{lem-imp1}, Lemma~\ref{lem-imp2}, and Lemma~\ref{lem-km-sample2} all occur, so that we do not need to repeatedly state the success probabilities. 


\begin{lemma}
\label{lem-km-1}
$\texttt{\bf{Cost}}(S_{\mathtt{opt}}, O^*)\leq (1+\delta)\frac{|S|}{n}(n-z)\big(\Delta^{-z}_2(P, O^*)+\xi  \mathcal{L}^2\big)$.
\end{lemma}
\begin{proof}
First, we have 
\begin{align}
\texttt{\bf{Cost}}&(S_{\mathtt{opt}}, O^*)\!=\!\sum^k_{j=1}\sum_{q\in S\cap C^*_j}||q-o^*_j||^2\nonumber\\
&\!\leq\!(1\!+\!\delta)\frac{|S|}{n}\sum^k_{j=1}\big(\sum_{p\in C^*_j}||p-o^*_j||^2\!+\!\xi |C^*_j| \mathcal{L}^2\big)\label{for-lem-km-1-1}
\end{align} 
by Lemma~\ref{lem-km-sample2}. Further, since $\sum^k_{j=1}\sum_{p\in C^*_j}||p-o^*_j||^2=(n-z)\Delta^{-z}_2(P, O^*)$ and $\sum^k_{j=1}|C^*_j|=n-z$, by plugging them into (\ref{for-lem-km-1-1}), we obtain Lemma~\ref{lem-km-1}.
%
\end{proof}

\begin{lemma}
\label{lem-km-2}
$\texttt{\bf{Cost}}(\tilde{S}_{\mathtt{opt}}, H)\leq (2+2c)\cdot\texttt{\bf{Cost}}(S_{\mathtt{opt}}, O^*)$.
\end{lemma}
\begin{proof}
We fix a point $q\in S_{\mathtt{opt}}$, and assume that the nearest neighbors of $q$ and $\tilde{q}$ in $H$ are $h_{j_q}$ and $h_{\tilde{j}_q}$, respectively. Then, we have
\begin{align}
||\tilde{q}-h_{\tilde{j}_q}||^2&\leq ||\tilde{q}-h_{j_q}||^2\nonumber\\
&\leq 2||\tilde{q}-q||^2+2||q-h_{j_q}||^2.\label{for-km-2-1}
\end{align}
Therefore,
{\small
\begin{align}
\sum_{q\in S_{\mathtt{opt}}}||\tilde{q} \! - \! h_{\tilde{j}_q}||^2& \! \leq \!  2 \! \sum_{q\in S_{\mathtt{opt}}}||\tilde{q} \! - \! q||^2 \! + \! 2 \! \sum_{q\in S_{\mathtt{opt}}}||q \! - \! h_{j_q}||^2,\nonumber\\
\implies \texttt{\bf{Cost}}(\tilde{S}_{\mathtt{opt}}, H)&\leq 2\texttt{\bf{Cost}}(S_{\mathtt{opt}}, O^*)+2\texttt{\bf{Cost}}(S_{\mathtt{opt}}, H).\label{for-km-2-2}
\end{align} 
}
\normalsize
Moreover, since $S_{\mathtt{opt}}\subseteq S$ (because $S_{\mathtt{opt}}=S\cap P_{\mathtt{opt}}$) and $H$ yields a $c$-approximate clustering cost of the $(k+k')$-means clustering on $S$, we have 
\begin{align}
\texttt{\bf{Cost}}(S_{\mathtt{opt}}, H)\leq \texttt{\bf{Cost}}(S, H)\leq c\cdot W,\label{for-km-2-3}
\end{align}
where $W$ is the optimal clustering cost of $(k+k')$-means clustering on $S$. Let $S'$ be the $k'$ farthest points of $S$ to $O^*$, then the set $O^*\cup S'$ also  
forms a solution for $(k+k')$-means clustering on $S$; namely, $S$ is partitioned into $k+k'$ clusters where each point of $S'$ is a cluster having a single point. Obviously, such a clustering yields a clustering cost $(|S|-k')\Delta^{-k'}_2(S, O^*)$. Consequently,
\begin{align}
W\leq (|S|-k')\Delta^{-k'}_2(S, O^*).\label{for-km-2-4} 
\end{align}
Also, Lemma~\ref{lem-imp2}  shows that $S$ contains at most $k'$ points from $P\setminus P_{\mathtt{opt}}$, {\em i.e.,} $|S_{\mathtt{opt}}|\geq |S|-k'$. Thus, 
\begin{eqnarray}
\texttt{\bf{Cost}}(S_{\mathtt{opt}}, O^*)\geq (|S|-k')\Delta^{-k'}_2(S, O^*). 
\end{eqnarray}
Together with (\ref{for-km-2-2}), (\ref{for-km-2-3}), and (\ref{for-km-2-4}), we have $\texttt{\bf{Cost}}(\tilde{S}_{\mathtt{opt}}, H)\leq (2+2c)\cdot\texttt{\bf{Cost}}(S_{\mathtt{opt}}, O^*)$.
\end{proof}

\begin{lemma}
\label{lem-km-3}
$\texttt{\bf{Cost}}(\tilde{P}_{\mathtt{opt}}, H)\leq \frac{1}{1-\delta}\frac{n}{|S|}\texttt{\bf{Cost}}(\tilde{S}_{\mathtt{opt}}, H)$.
\end{lemma}
\begin{proof}
From the constructions of $\tilde{P}_{\mathtt{opt}}$ and $\tilde{S}_{\mathtt{opt}}$, we know that they are overlapping points locating at $\{o^*_1, \cdots, o^*_k\}$. From Lemma~\ref{lem-imp1}, we know $|S\cap C^*_j|\geq (1-\delta)\frac{|C^*_j|}{n}|S|$, {\em i.e.,} 
\begin{eqnarray}
|C^*_j|\leq \frac{1}{1-\delta}\frac{n}{|S|}|S\cap C^*_j| \text{ for $1\leq j\leq k$.} 
\end{eqnarray}
Overall, we have 
$\texttt{\bf{Cost}}(\tilde{P}_{\mathtt{opt}}, H)=\sum^k_{j=1}|C^*_j|\big(\texttt{dist}(o^*_j, H)\big)^2$ that is at most 
\begin{align}
 && \frac{1}{1-\delta}\frac{n}{|S|}\sum^k_{j=1}|S\cap C^*_j|\big(\texttt{dist}(o^*_j, H)\big)^2\nonumber\\
 && =  \frac{1}{1-\delta}\frac{n}{|S|}\texttt{\bf{Cost}}(\tilde{S}_{\mathtt{opt}}, H).\nonumber
 \end{align}
%
%
%
%
\end{proof}
Now, we are ready to prove Theorem~\ref{the-km}. 

\begin{proof}(\textbf{of Theorem~\ref{the-km}}) Note that $(n-z)\Delta^{-z}_2(P, H)$ actually is the $|H|$-means clustering cost of $P$ by removing the farthest $z$ points to $H$, and $|P_{\mathtt{opt}}|=n-z$. So we have 
\begin{eqnarray}
(n-z)\Delta^{-z}_2(P, H)\leq \texttt{\bf{Cost}}(P_{\mathtt{opt}}, H). \label{for-the-km-sig-1}
\end{eqnarray}
Further, by using a similar manner of (\ref{for-km-2-2}), we have $\texttt{\bf{Cost}}(P_{\mathtt{opt}}, H)\leq 2\texttt{\bf{Cost}}(P_{\mathtt{opt}}, O^*)+2\texttt{\bf{Cost}}(\tilde{P}_{\mathtt{opt}}, H)$.
 Therefore,
{\small
\begin{align}
&\Delta^{-z}_2(P, H) \leq \frac{1}{n-z}\texttt{\bf{Cost}}(P_{\mathtt{opt}}, H)    \\
&\leq \frac{2}{n-z}\Big(\texttt{\bf{Cost}}(P_{\mathtt{opt}}, O^*)+\texttt{\bf{Cost}}(\tilde{P}_{\mathtt{opt}}, H)\Big).\label{for-the-km-f4-1}
\end{align}
Moreover, 
\begin{align}
\texttt{\bf{Cost}}(\tilde{P}_{\mathtt{opt}}, H)&\leq  \frac{1}{1-\delta}\frac{n}{|S|}\texttt{\bf{Cost}}(\tilde{S}_{\mathtt{opt}}, H)  
\label{for-the-km-f1}
\\
&\leq \frac{1}{1-\delta}\frac{n}{|S|}(2+2c)\texttt{\bf{Cost}}(S_{\mathtt{opt}}, O^*) 
\label{for-the-km-f2}
\\
&\leq   \! \frac{1 \! + \! \delta}{1 \! - \! \delta}(2 \! + \! 2c)(n \! - \! z)\big(\Delta^{-z}_2(P, O^*) \! + \! \xi  \mathcal{L}^2\big),
\label{for-the-km-f3} 
\end{align}
}
where (\ref{for-the-km-f1}), (\ref{for-the-km-f2}), and (\ref{for-the-km-f3}) come from Lemma~\ref{lem-km-3}, Lemma~\ref{lem-km-2}, and Lemma~\ref{lem-km-1} respectively.
From the fact $ \texttt{\bf{Cost}}(P_{\mathtt{opt}}, O^*)=(n-z)\Delta^{-z}_2(P, O^*)$, (\ref{for-the-km-f4-1}) and (\ref{for-the-km-f3}) imply $\Delta^{-z}_2(P, H)\leq\big(2 \! + \! (4 \! + \! 4c \! )\frac{1 \! + \! \delta}{1 \! - \! \delta}\big)\Delta^{-z}_{2}(P \!,\! O^*)\!+ \! (4 \! + \! 4c)\frac{1 \! + \! \delta}{1 \! - \! \delta}\xi \mathcal{L}^2$.

The success probability $(1-\eta)^3$ comes from Lemma~\ref{lem-km-sample2} and Lemma~\ref{lem-imp2}  (note that Lemma~\ref{lem-km-sample2} already takes into account of the success probability of Lemma~\ref{lem-imp1} ). Thus, we complete the proof of Theorem~\ref{the-km}.
\end{proof}

 \begin{algorithm}[hbt!]
   \caption{\textsc{Uni-$k$-Means Outliers \Rmnum{2}}}
   \label{alg-km2}
\begin{algorithmic}
  \STATE {\bfseries Input:} An $(\epsilon_1, \epsilon_2)$-significant instance $P$ of $k$-means clustering with $z$ outliers in $\mathbb{R}^d$, and $|P|=n$; three parameters $\eta, \delta, \xi\in (0,1)$. 
   \STATE
\begin{enumerate}
\item Take a uniform sample $S$ of $\max\{\frac{3k}{\delta^2\epsilon_1}\log\frac{2k}{\eta},$ $ \frac{k}{2\xi^2\epsilon_1(1-\delta)}\log\frac{2k}{\eta}\}$ points   from $P$.
\item Let $\hat{z}=\frac{1}{\eta}\frac{\epsilon_2}{k}|S|$, and solve the $k$-means clustering with $\hat{z}$ outliers on $S$ by using any $c$-approximation algorithm with $c\geq 1$. 
\end{enumerate}
  \STATE {\bfseries Output}  $H$, which is the set of $k$ cluster centers returned in Step~2. 
\end{algorithmic}

\end{algorithm}

\begin{theorem}
\label{the-km2}
Let $0<\delta, \eta, \xi<1$, and $t=\eta(1-\delta)\frac{\epsilon_1}{\epsilon_2}$. 
Assume $t>1$.
With probability at least $(1-\eta)^3$, the set of cluster centers $H$ returned by Algorithm~\ref{alg-km2} yields a clustering cost $\Delta^{-z}_{2}(P, H)\leq\alpha \Delta^{-z}_{2}(P, O^*)+\beta\xi \mathcal{L}^2$, where $\alpha=\big(2+(4+4c)\frac{t}{t-1}\frac{1+\delta}{1-\delta}\big)$ and $\beta=(4+4c)\frac{t}{t-1}\frac{1+\delta}{1-\delta}$. 
\end{theorem}
\begin{remark}
Similar to Remark~\ref{rem-the-kc2}, as long as $\frac{\epsilon_1}{\epsilon_2}>1$, we can set $\eta=\sqrt{\frac{\epsilon_2}{\epsilon_1}}$ and $\delta<1-\sqrt{\frac{\epsilon_2}{\epsilon_1}}$ to keep $t>1$. Additionally, like Corollary~\ref{cor-kc2-repeat}, we can repeat the algorithm $\frac{1}{(1-\eta)^3}$ times to achieve a constant success probability. If $\frac{t}{t-1}=O(1)$ and $c=O(1)$, then the clustering cost of Theorem~\ref{the-km2} will be 
 \begin{eqnarray}
 \Delta^{-z}_{2}(P, H)\leq O(1)\cdot \Delta^{-z}_{2}(P, O^*)+O(\xi)\cdot \mathcal{L}^2.
 \end{eqnarray}
 Also, when $\frac{\epsilon_1}{\epsilon_2}$ is large, the success probability becomes high as well. This also agrees with our previous observation concluded in Remark~\ref{rem-the-kc2}, that is, the ratio $\frac{\epsilon_1}{\epsilon_2}$ is an important factor that affects the performance of the uniform sampling approach. 
 
 \end{remark}

Before proving Theorem~\ref{the-km2}, we introduce the following lemmas first. 
Suppose the $k$ clusters of $S$ obtained in Step (2) of Algorithm~\ref{alg-km2} are $S_1, S_2, \cdots, S_k$, and thus the inliers $S_{\mathtt{in}}=\cup^k_{j=1}S_j$. Similar to the proof of Theorem~\ref{the-km}, below we always assume that the events described in Lemma~\ref{lem-imp1}, Lemma~\ref{lem-imp2}, and Lemma~\ref{lem-km-sample2} all occur, so that we do not need to repeatedly state the success probabilities.

\begin{lemma}
\label{lem-km2-1}
$\frac{|C^*_j|}{|C^*_j\cap S_{\mathtt{in}}|}\leq \frac{n}{|S|}\frac{t}{(t-1)(1-\delta)}$ for each $1\leq j\leq k$. 
\end{lemma}
\begin{proof}
Since $\hat{z}=\frac{1}{\eta}\frac{\epsilon_2}{k}|S|$ and $|S\cap C^*_j|\geq (1-\delta)\frac{|C^*_j|}{n}|S|$ for each $1\leq j\leq k$ (by Lemma~\ref{lem-imp1}), we have 
\begin{align}
|S_{\mathtt{in}}\cap C^*_j|&\geq |S\cap C^*_j|-\hat{z}\geq (1-\delta)\frac{|C^*_j|}{n}|S|-\frac{1}{\eta}\frac{\epsilon_2}{k}|S|\nonumber\\
&=\Big(1-\frac{1}{\eta(1-\delta)}\frac{\epsilon_2}{k}\frac{n}{|C^*_j|}\Big)(1-\delta)\frac{|C^*_j|}{n}|S|\nonumber\\
&\geq\Big(1-\frac{1}{\eta(1-\delta)}\frac{\epsilon_2}{\epsilon_1}\Big)(1-\delta)\frac{|C^*_j|}{n}|S|\nonumber\\
&=(1-\frac{1}{t})(1-\delta)\frac{|C^*_j|}{n}|S|,
\end{align}
where the last inequality comes from $|C^*_j|\geq \frac{\epsilon_1}{k}n$. Thus 
$\frac{|C^*_j|}{|C^*_j\cap S_{\mathtt{in}}|}\leq \frac{n}{|S|}\frac{t}{(t-1)(1-\delta)}$.  
\end{proof}

\begin{lemma}
\label{lem-km2-2}
$\texttt{\bf{Cost}}(S_{\mathtt{in}}\cap P_{\mathtt{opt}}, H)\leq (1+\delta)\frac{|S|}{n}\cdot c\cdot \big(\texttt{\bf{Cost}}(P_{\mathtt{opt}}, O^*)+(n-z)\cdot \xi \mathcal{L}^2\big)$. 
\end{lemma}
\begin{proof}
Since $\hat{z}\geq |S\setminus P_{\mathtt{opt}}|=|S|-|S_{\mathtt{opt}}|$, {\em i.e.,} $|S_{\mathtt{opt}}|\geq |S|-\hat{z}$, we have
\begin{align}
(|S|-\hat{z})\Delta^{-\hat{z}}_2(S, O^*)&\leq |S_{\mathtt{opt}}|\cdot \Delta^{-\hat{z}}_2(S, O^*) \nonumber\\
&\leq\texttt{\bf{Cost}}(S_{\mathtt{opt}}, O^*) \nonumber\\
\implies \Delta^{-\hat{z}}_2(S, O^*)&\leq\frac{1}{|S|-\hat{z}}\texttt{\bf{Cost}}(S_{\mathtt{opt}}, O^*),\label{for-km2-2-2}
\end{align}
where the second inequality is due to the same reason of (\ref{for-the-km-sig-1}). Because $H$ is a $c$-approximation on $S$, 
\begin{align}
\Delta^{-\hat{z}}_2(S, H)&\leq c\cdot \Delta^{-\hat{z}}_2(S, O^*)\nonumber\\
&\leq \frac{c}{|S|-\hat{z}}\texttt{\bf{Cost}}(S_{\mathtt{opt}}, O^*), \label{for-km2-2-3}
\end{align}
where the second inequality comes from (\ref{for-km2-2-2}). Therefore,  
\begin{align}
& \texttt{\bf{Cost}}(S_{\mathtt{in}} \cap  P_{\mathtt{opt}}, H)\nonumber\\
&\leq \texttt{\bf{Cost}}(S_{\mathtt{in}}, H) = (|S|-\hat{z})\Delta^{-\hat{z}}_2(S, H)\nonumber \nonumber\\
&\leq c\cdot \texttt{\bf{Cost}}(S_{\mathtt{opt}}, O^*)\nonumber\\
& \leq (1+\delta)\frac{|S|}{n}  \cdot  c  \cdot   (n-z)\big(\Delta^{-z}_2(P, O^*)+\xi  \mathcal{L}^2\big)\nonumber\\
& =  (1 \! + \! \delta)\frac{|S|}{n} \! \cdot \! c \! \cdot \! \big(\texttt{\bf{Cost}}(P_{\mathtt{opt}}, \!  O^*) \! + \! (n \! - \! z) \! \cdot \!  \xi \mathcal{L}^2\big),
\end{align}
where the second and third inequalities comes from (\ref{for-km2-2-3}) and Lemma~\ref{lem-km-1}, respectively. So we complete the proof.
\end{proof}
Since $S_{\mathtt{in}}\cap P_{\mathtt{opt}}\subseteq S_{\mathtt{opt}}$, we immediately have the following lemma via Lemma~\ref{lem-km-1}.
\begin{lemma}
\label{lem-km2-3}
$\texttt{\bf{Cost}}(S_{\mathtt{in}}\cap P_{\mathtt{opt}}, O^*)\leq (1+\delta)\frac{|S|}{n}(n-z)\big(\Delta^{-z}_2(P, O^*)+\xi  \mathcal{L}^2\big)$.
\end{lemma}
Now, we are ready to prove Theorem~\ref{the-km2}. 

\begin{proof}(\textbf{of Theorem~\ref{the-km2}}) 
For convenience, let $S'_{\mathtt{in}}=S_{\mathtt{in}}\cap P_{\mathtt{opt}}$. Using the same manner of (\ref{for-km-2-2}), we have 
\begin{align}
\texttt{\bf{Cost}}(\tilde{S}'_{\mathtt{in}}, H)&\leq 2 \texttt{\bf{Cost}}(S'_{\mathtt{in}}, O^*)+2 \texttt{\bf{Cost}}(S'_{\mathtt{in}}, H);\label{for-km2-2-4}\\
\texttt{\bf{Cost}}(P_{\mathtt{opt}}, H)&\leq 2\texttt{\bf{Cost}}(P_{\mathtt{opt}}, O^*)+2\texttt{\bf{Cost}}(\tilde{P}_{\mathtt{opt}}, H). \label{for-km2-2-5}
\end{align}
Also, because $\texttt{\bf{Cost}} \! (\tilde{P}_{\mathtt{opt}}, \!  H)=\sum^k_{j=1}|C^*_j|(\texttt{dist}(o^*_j, \!  H))^2$ and $\texttt{\bf{Cost}}(\tilde{S}'_{\mathtt{in}}, H)$ $=\sum^k_{j=1}|C^*_j\cap S_{\mathtt{in}}|(\texttt{dist}(o^*_j, H))^2$, we have 
\begin{eqnarray}
\frac{\texttt{\bf{Cost}}(\tilde{P}_{\mathtt{opt}}, H)}{\texttt{\bf{Cost}}(\tilde{S}'_{\mathtt{in}}, H)}\leq \max_{1\leq j\leq k}\frac{|C^*_j|}{|C^*_j\cap S_{\mathtt{in}}|}.
\end{eqnarray}
As a consequence, 
\begin{align}
\texttt{\bf{Cost}}(\tilde{P}_{\mathtt{opt}}, H)&\leq \max_{1\leq j\leq k}\frac{|C^*_j|}{|C^*_j\cap S_{\mathtt{in}}|}\cdot \texttt{\bf{Cost}}(\tilde{S}'_{\mathtt{in}}, H)\nonumber\\
&\leq \frac{n}{|S|}\frac{t}{(t-1)(1-\delta)}\cdot \texttt{\bf{Cost}}(\tilde{S}'_{\mathtt{in}}, H) \label{for-km2-2-6}
\end{align}
where the last inequality comes from Lemma~\ref{lem-km2-1}. From (\ref{for-km2-2-4}), (\ref{for-km2-2-5}), and (\ref{for-km2-2-6}), we have
\begin{align}
\texttt{\bf{Cost}}(P_{\mathtt{opt}}, H)&\leq 2\texttt{\bf{Cost}}(P_{\mathtt{opt}}, O^*)  +  \nonumber\frac{4n}{|S|}\frac{t}{(t-1)(1-\delta)}\cdot \\
& \big(\texttt{\bf{Cost}}(S'_{\mathtt{in}}, O^*)+\texttt{\bf{Cost}}(S'_{\mathtt{in}}, H)\big).\!\!\label{for-km2-2-7}
\end{align}
By plugging the inequalities of Lemma~\ref{lem-km2-2} and Lemma~\ref{lem-km2-3} into (\ref{for-km2-2-7}), we can obtain Theorem~\ref{the-km2}. 
\end{proof}

\section{Experiments}
\label{sec-exp}
All the experiments were conducted  on a Ubuntu workstation with 2.40GHz Intel(R) Xeon(R) CPU E5-2680 and 256GB main memory. The algorithms were implemented in Matlab R2020b, and our code is available at \href{https://github.com/h305142/lightweight-clustering}{https://github.com/h305142/lightweight-clustering}

\textbf{Algorithms for testing.} We use several baseline algorithms  mentioned in Section~\ref{sec-prior}.  For $k$-center clustering with outliers, we consider  the $3$-approximation \textsc{\textbf{Charikar}}~\citep{charikar2001algorithms},  the $(4+\epsilon)$-approximation \textbf{\textsc{MK}} \citep{mccutchen2008streaming},  the $13$-approximation \textbf{\textsc{Malkomes}} \citep{malkomes2015fast}, and the non-uniform sampling algorithms \textbf{\textsc{BVX}}~\citep{NEURIPS2019_73983c01} and 
\textbf{\textsc{DYW}}~\citep{DBLP:conf/esa/DingYW19}. In our Algorithm~\ref{alg-kc2}, we apply \textsc{MK} as the black-box algorithm  in Step 2 (though  \textsc{Charikar} has a lower approximation ratio, we observe that  \textsc{MK} runs  faster and often achieves comparable clustering results in practice).

For $k$-means/median clustering with outliers, we consider the heuristic algorithm \textbf{\textsc{$k$-means$--$}}~\citep{chawla2013k} and three non-uniform sampling methods: the local search algorithm \textbf{\textsc{LocalSearch}}~\citep{gupta2017local} (we use the fast implementation~\citep{10.5555/3157096.3157103} for its  $k$-means++ seeding), the  data summary based algorithm \textbf{\textsc{DataSummary}}~\citep{DBLP:conf/nips/ChenA018}, and the recently  proposed hybrid sampling algorithm \textbf{\textsc{Hybrid}}~\citep{DBLP:journals/corr/abs-2003-02433}. In our Algorithm~\ref{alg-km} and  Algorithm~\ref{alg-km2}, we apply the $k$-means++~\citep{10.5555/3157096.3157103}  and \textsc{$k$-means$--$} respectively as the black-box algorithms in their Step 2. 

 {\textbf{Datasets.}} We generate the synthetic datasets in $\mathbb{R}^{100}$, where the points of each cluster follow a random Gaussian distribution (similar with the methods in~\citet{DBLP:conf/nips/ChenA018,DBLP:journals/corr/abs-2003-02433}). 
%
 For each synthetic dataset, we uniformly generate $z$ outliers at random outside the minimum enclosing balls of the obtained $k$ clusters. 

We also choose $4$ real datasets from {\em UCI machine learning repository}~\citep{Dua:2019}. 
(1) \textbf{Covertype} has $7$ clusters with $5.8\times 10^5$ points in $\mathbb{R}^{54}$; (2) \textbf{Kddcup} has $23$ clusters with $4.9\times 10^6$ points in $\mathbb{R}^{38}$; (3)  \textbf{Poking Hand} has $10$ clusters with $ 10^6$ points in $\mathbb{R}^{10}$; (4) \textbf{Shuttle} has $7$ clusters with $5.8\times 10^4$ points in $\mathbb{R}^9$.
We add $1\%$ outliers uniformly at random  outside the enclosing balls of the clusters as we did for  the synthetic datasets.


\textbf{Some implementation details.} When implementing our proposed algorithms (Algorithm~\ref{alg-kc1}-\ref{alg-km2}), it is not quite convenient to set the values for the parameters $\{\eta,\xi,\delta\}$ in practice. In fact, we only need to determine the sample size $|S|$ and $k'$ for running  Algorithm~\ref{alg-kc1} and Algorithm~\ref{alg-km} (and similarly, $|S|$ and $\hat{z}$ for running  Algorithm~\ref{alg-kc2} and Algorithm~\ref{alg-km2}). As an example, for Algorithm~\ref{alg-kc1}, it is sufficient to input $|S|$ and $k'$ only, since we just need to compute the $(k+k')$-center clustering on the sample $S$; moreover, it is more intuitive to directly set $|S|$ and $k'$  rather than to set the parameters $\{\eta,\xi,\delta\}$. 
Therefore, we  will conduct the experiments to observe the trends when varying the values $|S|$, $k'$, and $\hat{z}$.

Another practical issue for implementation is the success probability. For Algorithm~\ref{alg-kc2} and \ref{alg-km2},  when $\frac{\epsilon_1}{\epsilon_2}$ is close to $1$, the success probability could be low (as discussed in Remark~\ref{rem-the-kc2}). We can run the algorithm multiple times, and at least one time the algorithm will return a qualified solution with a higher probability (by simply taking the union bound). For example, if $\eta=0.8$ and we run Algorithm~\ref{alg-kc2} $50$ times, the success probability will be $1-(1-(1-0.8)^2)^{50}\approx 87\%$. Suppose we run the algorithm $m>1$ times and let $H_1, \cdots, H_{m}$ be the set of output candidates. 
The remaining issue is that how to select the one that achieves the smallest objective value among all the candidates. 
We directly scan the whole dataset in one pass. When reading a point $p$ from $P$, we calculate its distance to all the candidates, {\em i.e.,} $\mathtt{dist}(p, H_1), \cdots, \mathtt{dist}(p, H_m)$; after scanning the whole dataset, we have calculated the clustering costs $\Delta^{-z}_\infty(P, H_l)$ ({\em resp.,} $\Delta^{-z}_1(P, H_l)$ and $\Delta^{-z}_2(P, H_l)$) for $1\leq l\leq m$ and then   return the best one. Another benefit of this operation is that we can determine the clustering assignments for the data points simultaneously. When calculating $\mathtt{dist}(p, H_l)$ for $1\leq l\leq m$, we record the index of its nearest cluster center in $H_l$; finally, we can return its corresponding clustering assignment in the selected best candidate.


\textbf{Experimental design.} Our experiments include  three parts. \textbf{(1)} We fix  the sample size $|S|$ to be $5\times 10^{-3} n$, and study the clustering quality (such as clustering cost, precision, and purity) and the running time. \textbf{(2)} We study  the scalabilities of the algorithms with enlarging the data size $n$, and the stabilities of our proposed algorithms with varying the sampling ratio $|S|/n$. \textbf{(3)} Finally, we focus on other factors that may affect the practical performances.


\subsection{Clustering Quality and Running Time}

In this part, we fix the sample size $|S|$ to be $5\times 10^{-3} n$. For the synthetic datasets, we first set $n=10^5$, $k=8$, and $z=2\% n$. For Algorithm~\ref{alg-kc1} and~\ref{alg-km}, the value $k'$ should be $\frac{1}{\eta}\frac{\epsilon_2}{k}|S|$ which could be large (though it is only $O(\log k)$ in the asymptotic analysis). Although we have constructed the example  in Section~\ref{sec-kc-first} to show that $k'$ cannot be reduced in the worst case, we do not strictly follow this theoretical value in our experiments. Instead, we keep the ratio $\tau=\frac{k+k'}{k}$ to be $1, \frac{4}{3}, \frac{5}{3}$, and $2$ ({\em i.e.,} run the black-box algorithms from~\citet{G85,arthur2007k} $\tau k$ steps). For Algorithm~\ref{alg-kc2}  and Algorithm~\ref{alg-km2}, we set $\hat{z}=2\frac{\epsilon_2}{k}|S|$; 
 as discussed for the implementation details before, to boost the success probability we run Algorithm~\ref{alg-kc2}  ({\em resp.,} Algorithm~\ref{alg-km2}) $10$ times  and select the best candidate (as one trial); we count the total time of the $10$ runs  for one trial. 
For each instance, we take $20$ trials and report the average results  in Figure~\ref{fig-exp-dataset-sup}. 


\begin{figure*}
     \centering
     \begin{subfigure}[b]{0.38\textwidth}
         \centering
         \includegraphics[width=1\textwidth]{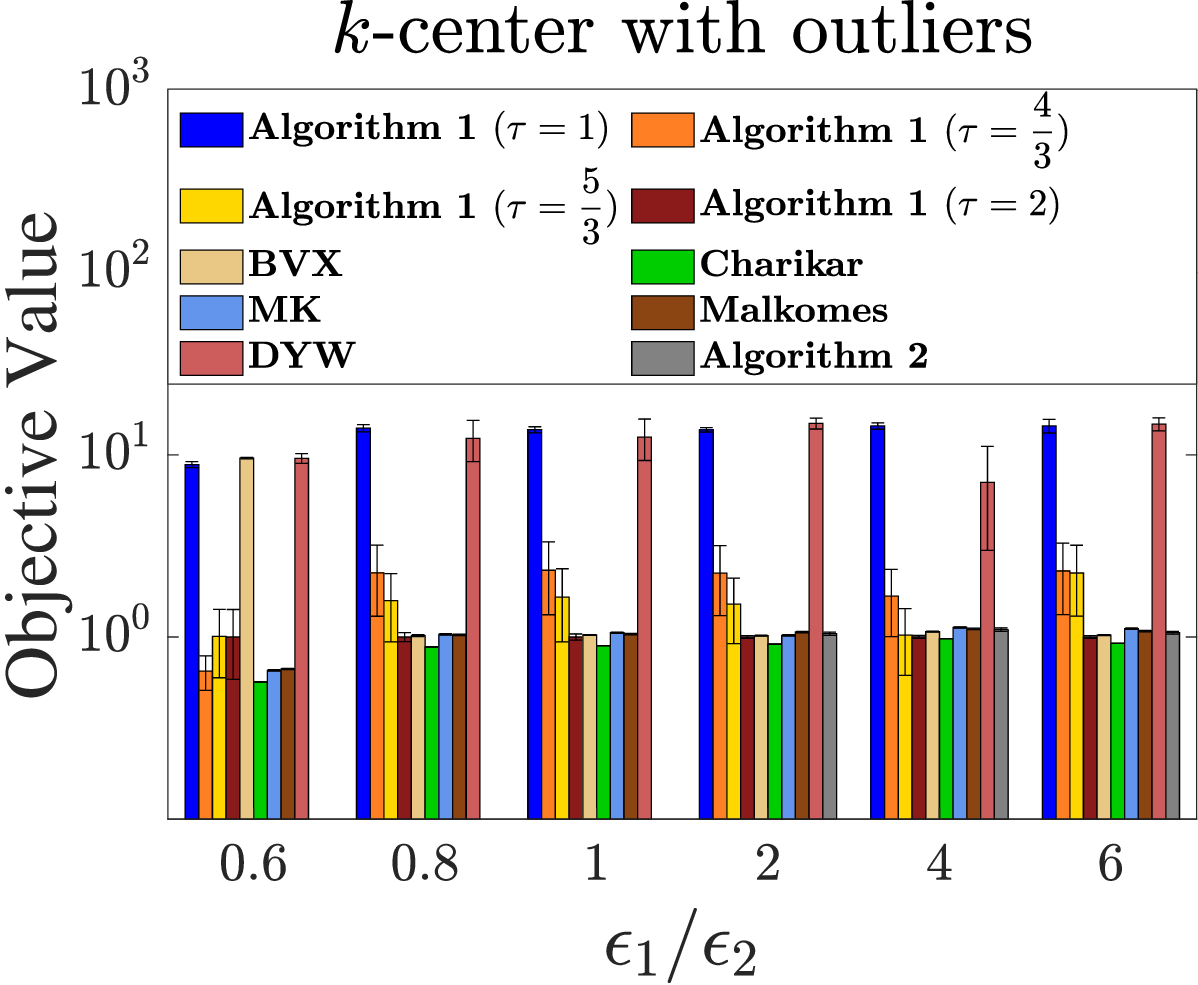}  
         \caption{Synthetic Datasets}
         \label{fig:subfig1}
     \end{subfigure}
     \hspace{0.2in}
     \begin{subfigure}[b]{0.38\textwidth}
         \centering
         \includegraphics[width=1\textwidth]{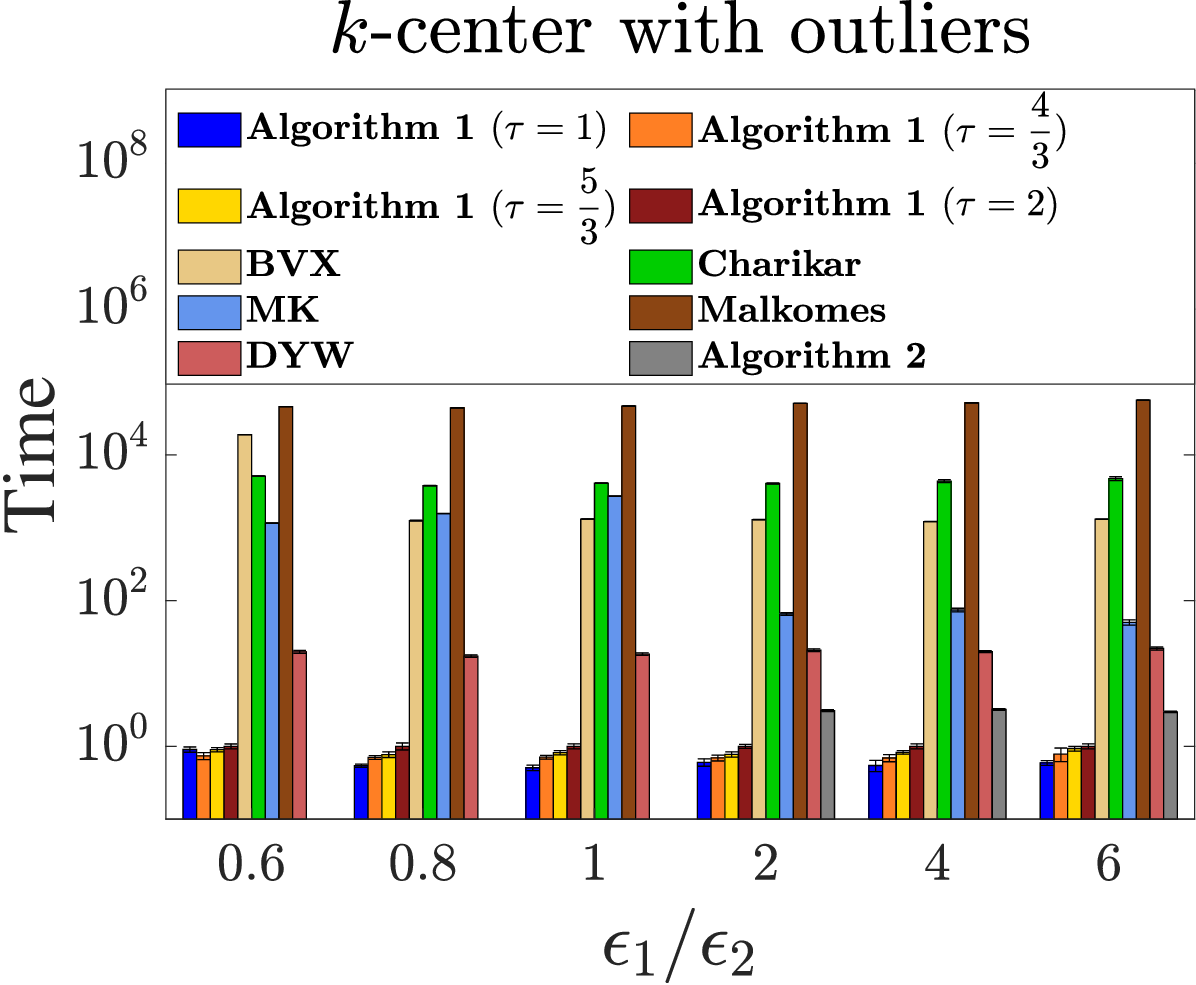}
         \caption{Synthetic Datasets}
         \label{fig:subfig2}
     \end{subfigure}
     \hfill

     \begin{subfigure}[b]{0.38\textwidth}
         \centering
         \includegraphics[width=1\textwidth]{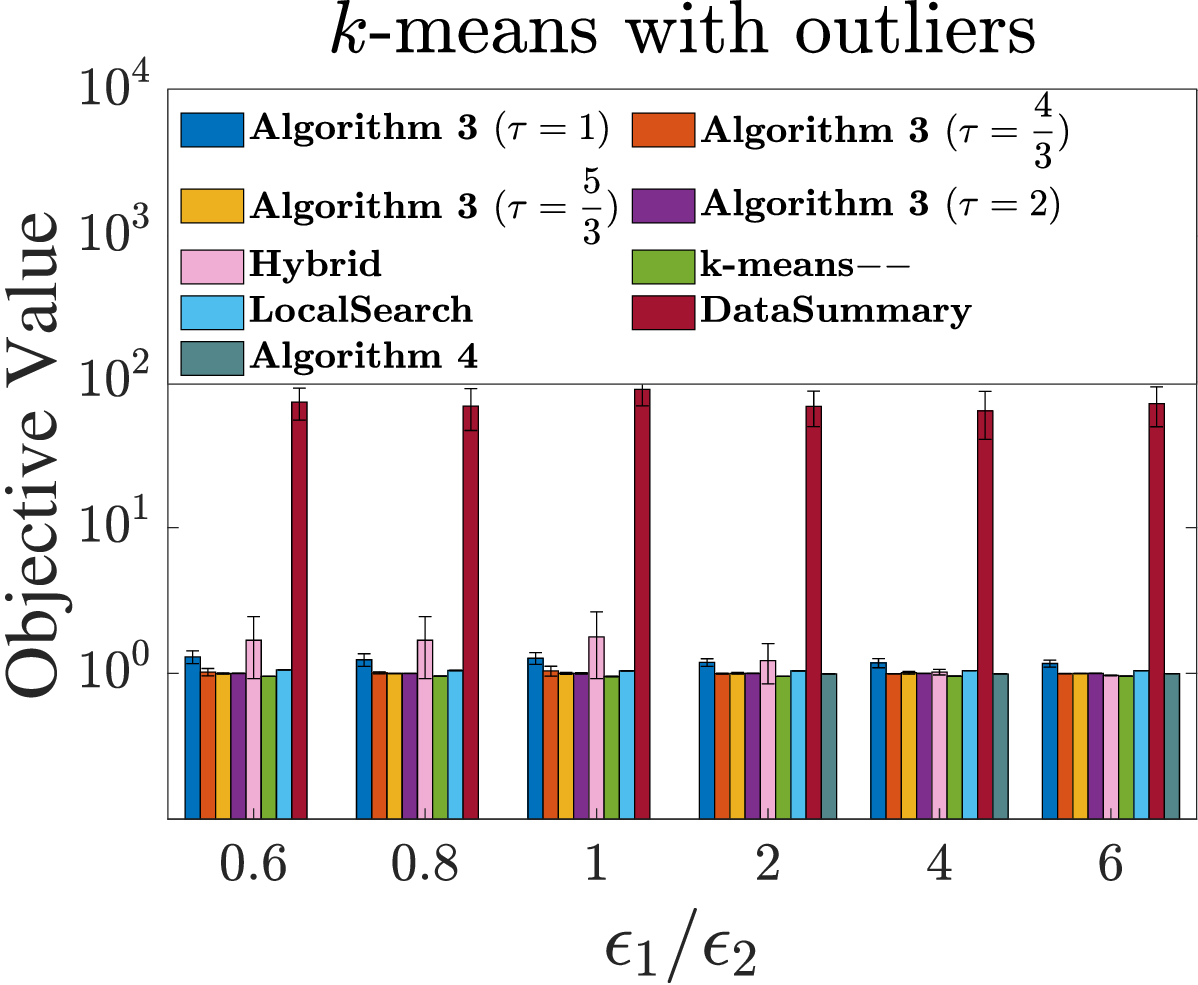}
         \caption{Synthetic Datasets}
         \label{fig:subfig3}
     \end{subfigure}
     \hspace{0.2in}
     \begin{subfigure}[b]{0.38\textwidth}
         \centering
         \includegraphics[width=1\textwidth]{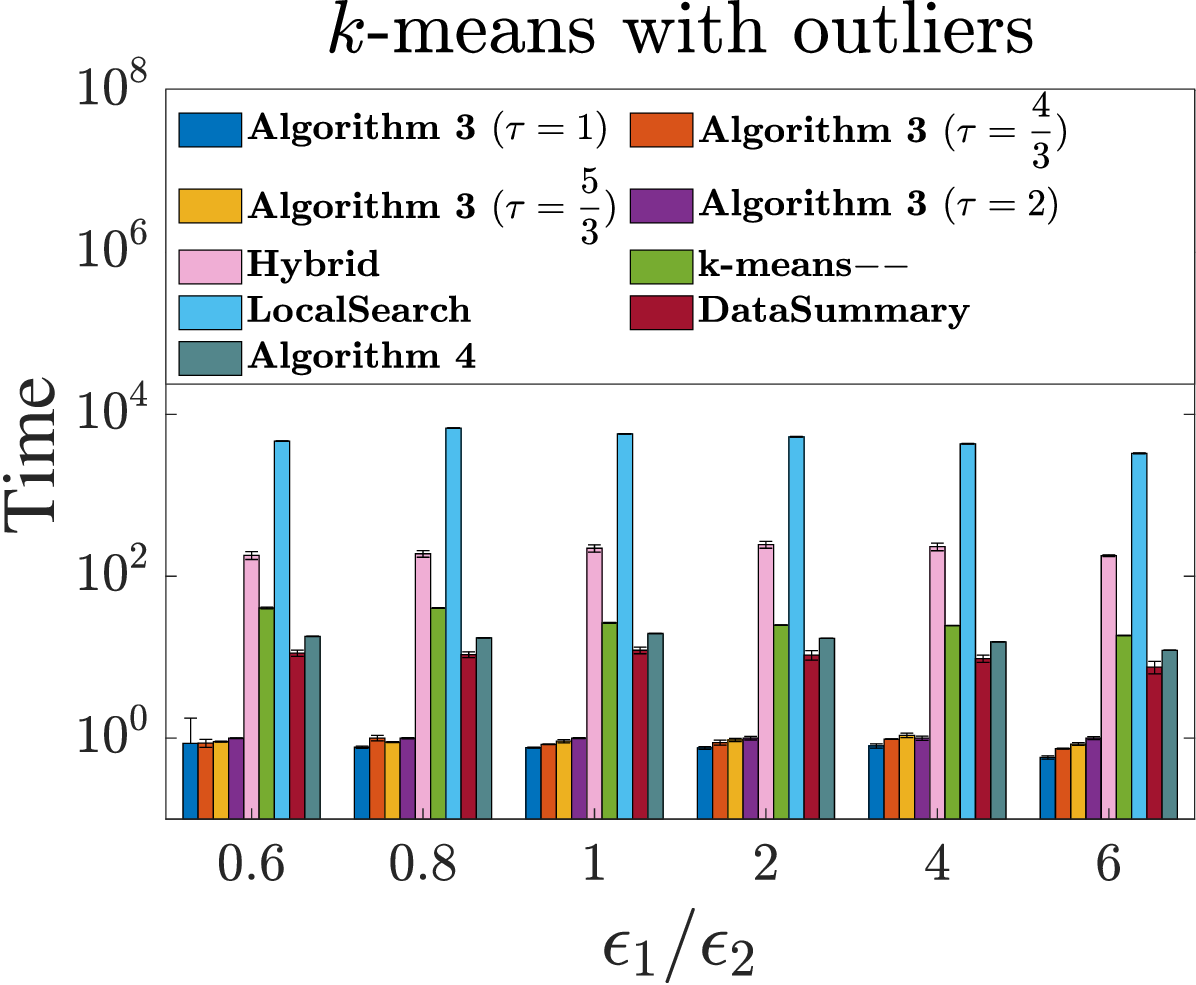} 
         \caption{Synthetic Datasets}
         \label{fig:subfig4}
     \end{subfigure}
     
     \begin{subfigure}[b]{0.38\textwidth}
         \centering
         \includegraphics[width=1\textwidth]{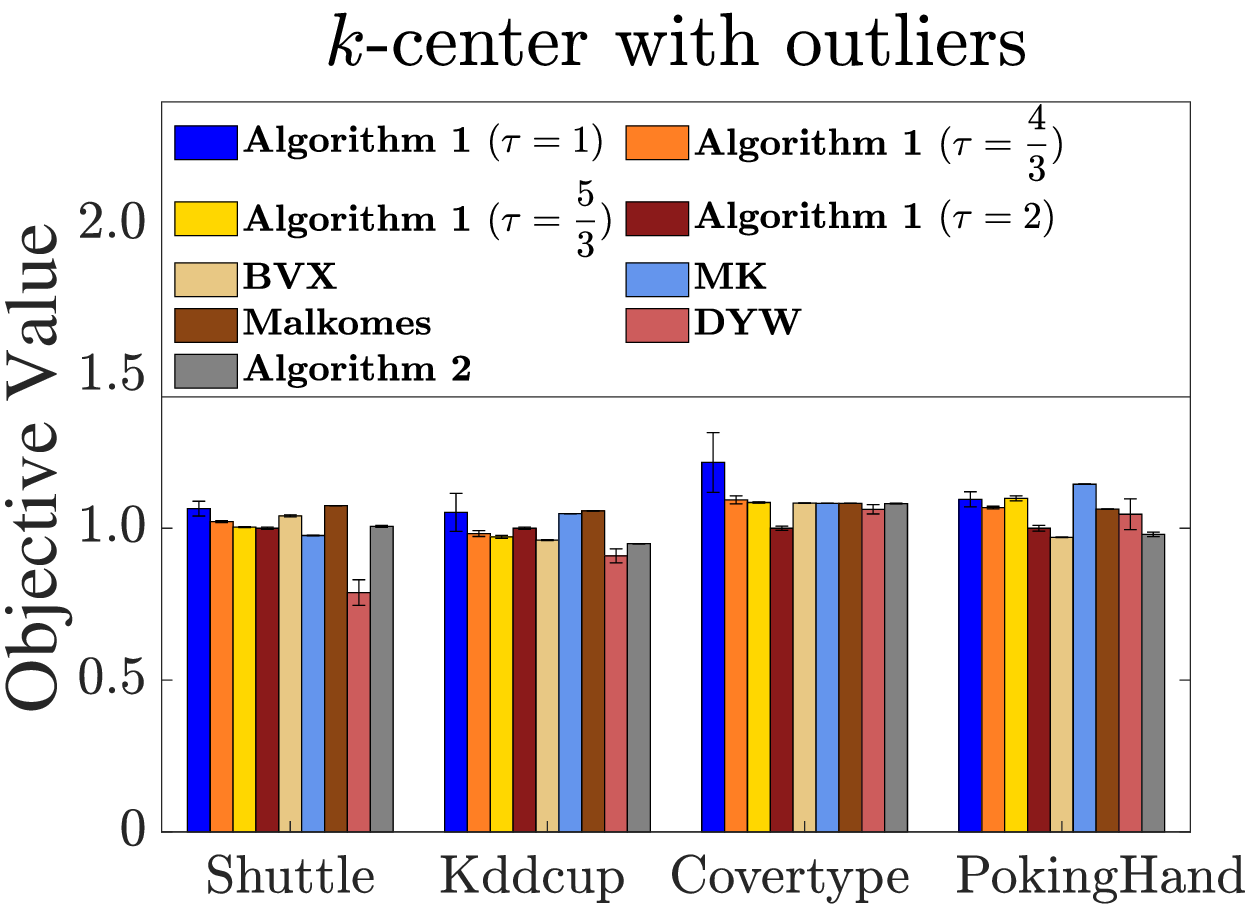} 
         \caption{Real Datasets}
         \label{fig:subfig5}
     \end{subfigure}
     \hspace{0.2in}
     \begin{subfigure}[b]{0.38\textwidth}
         \centering
         \includegraphics[width=1\textwidth]{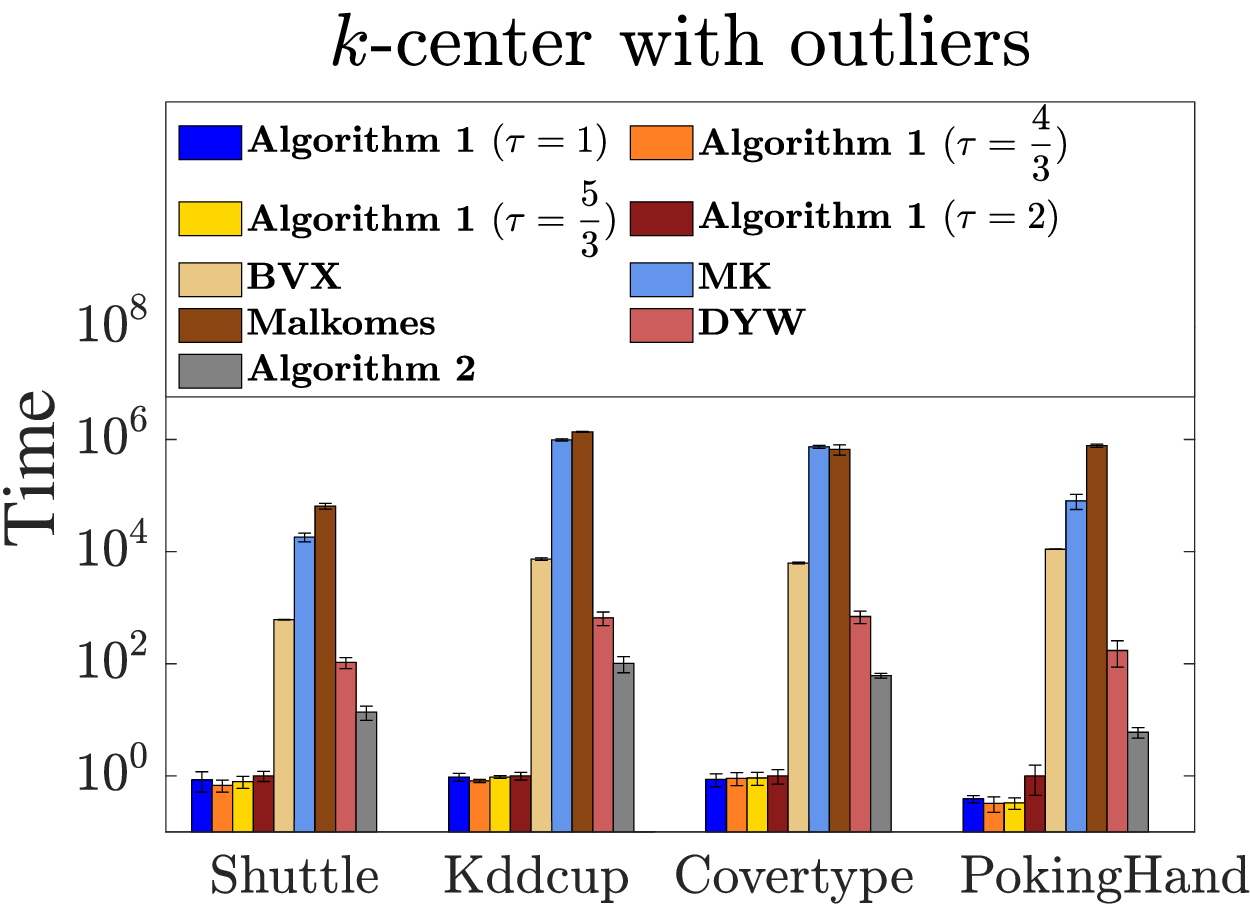}  
         \caption{Real Datasets}
         \label{fig:subfig6}
     \end{subfigure}
     
     \begin{subfigure}[b]{0.38\textwidth}
         \centering
         \includegraphics[width=1\textwidth]{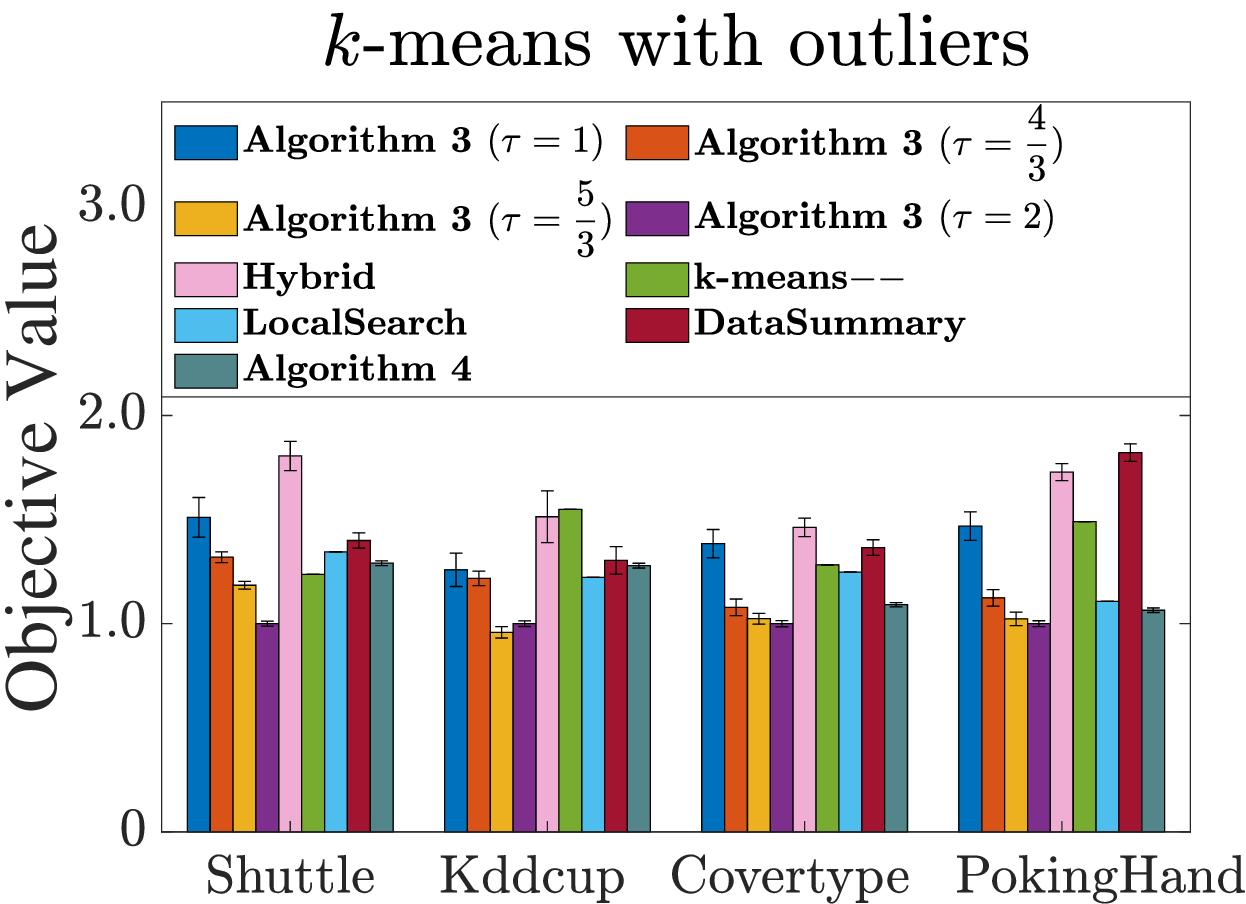}  
         \caption{Real Datasets}
         \label{fig:subfig7}
     \end{subfigure}
     \hspace{0.2in}
     \begin{subfigure}[b]{0.38\textwidth}
         \centering
         \includegraphics[width=1\textwidth]{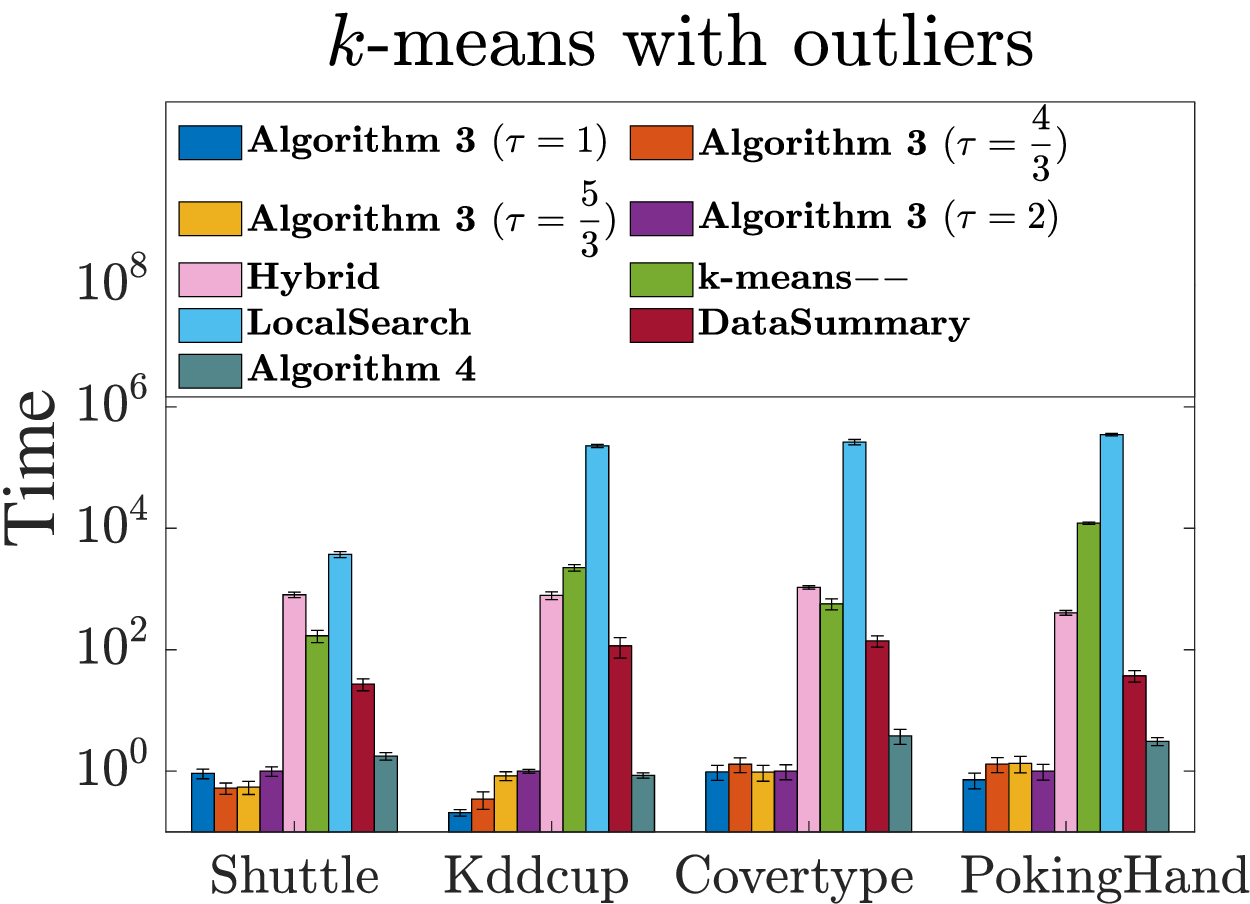}
         \caption{Real Datasets}
         \label{fig:subfig8}
     \end{subfigure}
    \vspace{-0.1in}

	 \caption{The normalized results (objective values and  running times) for $k$-center and $k$-means clustering with outliers in synthetic data (Figure~\ref{fig:subfig1} $\sim$ ~\ref{fig:subfig4}) and real data (Figure~\ref{fig:subfig5} $\sim$ ~\ref{fig:subfig8}). To illustrate the comparisons more clearly, we normalize all the results by dividing them over the results of Algorithm~\ref{alg-kc1} with $\tau = 2$ ({\em resp.,} Algorithm~\ref{alg-km} with $\tau = 2$ ) for $k$-center ({\em resp.,} $k$-means/median ) clustering with outliers. }
	 \label{fig-exp-dataset-sup}
\end{figure*}

For $k$-center clustering with outliers, our algorithms (Algorithm~\ref{alg-kc1} with $\tau\geq4/3$ and Algorithm~\ref{alg-kc2}) and the four baseline algorithms achieve similar objective values for most of the instances (we run Algorithm~\ref{alg-kc2} on the synthetic datasets with $\frac{\epsilon_1}{\epsilon_2}>1$ only; we do not run \textsc{Charikar} on the real datasets due to its high complexity). 
For $k$-means clustering with outliers, Algorithm~\ref{alg-km}  with $\tau\geq4/3$ and Algorithm~\ref{alg-km2} can achieve the results close to the best of the three baseline algorithms. Moreover, the running times of  our algorithms are significantly lower comparing with the baseline algorithms. 
Overall, we conclude that (1) Algorithm~\ref{alg-kc1} and Algorithm~\ref{alg-km} just need to return slightly more than $k$ cluster centers ({\em e.g.,} $\frac{4}{3} k$) for achieving low clustering cost; (2) Algorithm~\ref{alg-kc2} and  Algorithm~\ref{alg-km2} can achieve low clustering cost when $\frac{\epsilon_1}{\epsilon_2}\geq 2$.


To further evaluate their clustering qualities, we consider the measures \textbf{{\em precision}} and  \textbf{{\em purity}}, which have been widely used before~\citep{DBLP:books/daglib/0021593}; these two measures both aim to evaluate the difference between the obtained clusters and the ground truth. 
(1) The \textbf{precision} is the proportion of the ground-truth outliers found by the algorithm ({\em i.e.,} $\frac{|Out\cap Out_\mathtt{truth}|}{|Out_\mathtt{truth}|}$, where $Out$ is the set of returned outliers and $Out_\mathtt{truth}$ is the set of ground-truth outliers). (2) For each obtained cluster, we assign it to the ground-truth cluster which is most frequent in the obtained cluster, and the \textbf{purity} measures the accuracy of this assignment. 
Specifically, let $\{C_1, C_2, \cdots, C_k\}$ be the ground-truth clusters and $\{C'_1, C'_2, \cdots, C'_k\}$ be the obtained clusters from the algorithm; the purity is equal to $\frac{1}{n-z}\sum^k_{j=1}\max_{1\leq l\leq k}|C'_j\cap C_l|$. In general, the precisions and the purities achieved by our uniform sampling approaches and the baseline algorithms  are relatively close. 
The numerical results are shown in Table~\ref{tab-2} and \ref{tab-3}.

\subsection{Scalability and Sampling Ratio}
We consider the scalability first. We enlarge the data size $n$ from $10^5$ to $10^7$, and illustrate the results in Figure~\ref{fig-exp-size-sup}. We fix the sampling size $|S| = 10^{3}$.   For Algorithm~\ref{alg-kc1} and~\ref{alg-km}, we set $\tau=2$. We can see that our algorithms are several orders of magnitude faster than the baselines when $n$ is large. Actually, since our approach is just simple uniform sampling, the advantage over the non-uniform sampling approaches will be more significant when $n$ becomes large. 


We then study the influence of the sampling ratio $|S|/n$ to  our algorithms. 
 We vary $|S|/n$ from $1\times 10^{-3}$ to $10\times 10^{-3}$ and run our algorithms on the synthetic datasets. We show the results (averaged across $20$ trials) in Figure~\ref{fig-exp-sample-sup}. 
 We can see that the trends tend to be  ``flat'' when the ratio $|S|/n>4\times 10^{-3}$.

\subsection{Other Influence Factors}

For completeness, we also consider several other influence factors. We fix the sampling ratio $|S|/n=5\times 10^{-3}$ as Section~\ref{sec-exp-part1}. 
As discussed before, the performances of Algorithm~\ref{alg-kc2} and Algorithm~\ref{alg-km2} depend on the ratio  $\epsilon_1/\epsilon_2$. An interesting question is that how about their stabilities in terms of $\epsilon_1/\epsilon_2$ (in particular, when $\epsilon_1/\epsilon_2\leq 1$). 
We repeat the experiments $20$ times on the synthetic datasets and compute the obtained average objective values and standard deviations. 
In Figure~\ref{fig-stability-sup} (the two figures in the first line), we can see that their performances are actually quite stable when $\epsilon_1/\epsilon_2\geq 2$. This also agrees with our previous theoretical analysis, that is,  larger $\epsilon_1/\epsilon_2$ is more friendly to uniform sampling.

In Figure~\ref{fig-stability-sup} (the two figures in the second line), we study the influence of $\hat{z}$  on the performances (averaged across $20$ times). We vary $\hat{z}/\tilde{z}$ from $1.1$ to $2$, where $\tilde{z}=\frac{\epsilon_2}{k}|S|$ is the expected number of outliers contained in $S$. 
We also illustrate the experimental results on the synthetic datasets with  varying $k$ from $8$ to $24$, and $z$ from $0.2\%n$ to $10\% n$, in Figure~\ref{fig-exp-k-sup} and Figure~\ref{fig-exp-outlier-sup} respectively. 
 Overall, we observe that the influences from   these parameters to the clustering qualities of our algorithms are relatively limited. 
And  in general,  our algorithms are considerably faster than the baseline algorithms.

\section{Future Work}
\label{sec-future}
In this paper, we study the effectiveness of uniform sampling for center-based clustering with outliers problems.  
Following this work, an interesting question  is that whether the  significance measure (or some other realistic assumptions) can be applied to analyze uniform sampling for other robust optimization problems, such as {\em PCA with outliers}~\citep{DBLP:journals/jacm/CandesLMW11} and {\em projective clustering with outliers}~\citep{feldman2011unified}. 

\footnotesize{
\subsection*{Acknowledgements}
The research of this work was supported in part by National Key R\&D program of China through grant 2021YFA1000900, the NSFC throught Grant 62272432, and the Provincial NSF of Anhui through grant 2208085MF163.  }


\appendix
\section{Extensions}
\label{sec-extension} 
\subsection{For $k$-median clustering with outliers}
\label{sec-extension-1} 
The results of Theorem~\ref{the-km} and Theorem \ref{the-km2} can be easily extended to $k$-median clustering with outliers in Euclidean space by using almost the same idea, where the only difference is that we can directly use triangle inequality in the proofs ({\em e.g.,} the inequality (\ref{for-km-2-1}) is replaced by $||\tilde{q}-h_{\tilde{j}_q}||\leq ||\tilde{q}-q||+||q-h_{j_q}||$). As a consequence, the coefficients $\alpha$ and $\beta$ are reduced to be $\big(1+(1+c)\frac{1+\delta}{1-\delta}\big)$ and $(1+c)\frac{1+\delta}{1-\delta}$  in Theorem~\ref{the-km}, respectively. Similarly, $\alpha$ and $\beta$ are reduced to be  $\big(1+(1+c)\frac{t}{t-1}\frac{1+\delta}{1-\delta}\big)$ and $(1+c)\frac{t}{t-1}\frac{1+\delta}{1-\delta}$  in  Theorem~\ref{the-km2}.


\subsection{For the general metric}
\label{sec-extension-2} 
To solve the metric $k$-median/means clustering with outliers problems for a given instance $(X, \mu)$, we should keep in mind that the cluster centers can only be selected from the vertices of $X$. However, the optimal cluster centers $O^*=\{o^*_1, \cdots, o^*_k\}$ may not be contained in the sample $S$, and thus we need to modify our analysis slightly. We observe that the sample $S$ contains a set $O'$ of vertices close to $O^*$ with certain probability. Specifically, for each $1\leq j\leq k$, there exists a vertex $o'_j\in O'$ such that $\mu(o'_j, o^*_j)\leq O(1)\times \frac{1}{|C^*_j|}\sum_{p\in C^*_j}\mu(p, o^*_j)$ (or $\big(\mu(o'_j, o^*_j)\big)^2\leq O(1)\times \frac{1}{|C^*_j|}\sum_{p\in C^*_j}\big(\mu(p, o^*_j)\big)^2$) with constant probability (this claim can be easily proved by using the Markov's inequality). Consequently, we can use $O'$ to replace $O^*$ in our analysis, and achieve the similar results as Theorem~\ref{the-km} and Theorem~\ref{the-km2}.


\begin{table*}[ht]  	
    \caption{Precision and Purity on real datasets ($k$-center with outliers)}  
    \label{tab-2} 
	\centering  
	\begin{tabular}{ccccccccc}  
		\toprule  
		\textsc{Datasets}	& \multicolumn{2}{c}{\textsc{Shuttle}}&\multicolumn{2}{c}{\textsc{Kddcup}}&\multicolumn{2}{c}{\textsc{Covtype}}&\multicolumn{2}{c}{\textsc{Poking Hand}}\cr 
		
		\cmidrule(lr){2-3} \cmidrule(lr){4-5}  \cmidrule(lr){6-7} \cmidrule(lr){8-9}  
		
		\textsc{Measure}     & \textsc{Prec} & \textsc{Purity} & \textsc{Prec} & \textsc{Purity} & \textsc{Prec} & \textsc{Purity} & \textsc{Prec} & \textsc{Purity} \cr
		\midrule  

        \textsc{Algorithm~\ref{alg-kc1}} $\tau = 1$ & 0.855 & 0.807 & 0.900 & 0.822 & 0.879 & 0.481 & 0.996 & 0.502 \cr
        \textsc{Algorithm~\ref{alg-kc1}} $\tau = \frac{4}{3}$ & 0.856 & 0.824 & 0.906 & 0.834 & 0.895 & 0.491 & 0.986 & 0.502 \cr
        \textsc{Algorithm~\ref{alg-kc1}} $\tau = \frac{5}{3}$ & 0.860 & 0.843 & 0.912 & 0.851 & 0.902 & 0.488 & 0.992 & 0.501 \cr
    	\textsc{Algorithm~\ref{alg-kc1}} $\tau = 2$ & 0.864 & \textbf{0.869} & 0.929 & 0.848 & \textbf{0.918} & 0.495 & \textbf{0.998} & 0.503 \cr
        \textsc{Algorithm~\ref{alg-kc2}} & 0.856 & 0.851 & 0.936 & \textbf{0.854} & 0.879 & 0.493 & 0.936 & 0.503 \cr
		\textsc{BVX} & \textbf{0.908} & 0.786 & \textbf{0.945} & 0.574 & 0.879 & 0.487 & \textbf{0.998} & 0.501 \cr
        \textsc{MK} &  0.886 & 0.828 & 0.914 & 0.832 & 0.879 & 0.488 & 0.982 & 0.502 \cr
        \textsc{Malkomes} & 0.862 & 0.790 & 0.918 & 0.841 & 0.902 & 0.501 & 0.992 & 0.504 \cr
        \textsc{DYW} & 0.853 & 0.813 & 0.912 & 0.850 & 0.899 & \textbf{0.502} & 0.981 & \textbf{0.510} \cr

		\bottomrule  
	\end{tabular}   
\end{table*} 

\vspace{0.3in}

\begin{table*}[ht]  
    \caption{Precision and Purity on real datasets ($k$-means with outliers)}  
	\label{tab-3} 
	\centering  
	\begin{tabular}{ccccccccc}  
	    \toprule  
		\textsc{Datasets}	& \multicolumn{2}{c}{\textsc{Shuttle}}&\multicolumn{2}{c}{\textsc{Kddcup}}&\multicolumn{2}{c}{\textsc{Covtype}}&\multicolumn{2}{c}{\textsc{Poking Hand}}\cr 
		
		\cmidrule(lr){2-3} \cmidrule(lr){4-5}  \cmidrule(lr){6-7} \cmidrule(lr){8-9}  
		
		\textsc{Measure}     & \textsc{Prec} & \textsc{Purity} & \textsc{Prec} & \textsc{Purity} & \textsc{Prec} & \textsc{Purity} & \textsc{Prec} & \textsc{Purity} \cr
		\midrule  
        \textsc{Algorithm~\ref{alg-km}} $\tau = 1$ & 0.857 & 0.883 & 0.948 & 0.988 & 0.879 & 0.507 & 0.982 & 0.501 \cr
        \textsc{Algorithm~\ref{alg-km}} $\tau = \frac{4}{3}$ & 0.852 & 0.922 & 0.962 & 0.986 & 0.889 & 0.522 & 0.996 & 0.502 \cr
        \textsc{Algorithm~\ref{alg-km}} $\tau = \frac{5}{3}$ & 0.855 & 0.914 & 0.955 & 0.989 & \textbf{0.913} & 0.527 & \textbf{0.999} & 0.505 \cr
        \textsc{Algorithm~\ref{alg-km}} $\tau = 2$ & 0.861 & \textbf{0.944} & \textbf{0.972} & 0.991 & 0.906 & 0.515 & \textbf{0.999} & 0.506 \cr
        \textsc{Algorithm~\ref{alg-km2}}& 0.856 & 0.867 & 0.927 & 0.982 & 0.906 & 0.523 & \textbf{0.999} & \textbf{0.510} \cr
		\textsc{Hybrid} & 0.858 & 0.903 & 0.899 & 0.992 & 0.884 & 0.534 & \textbf{0.999} & 0.501 \cr
        \textsc{$k$-means$--$} & \textbf{0.872} & 0.842 & 0.915 & 0.985 & 0.879 & 0.537 & \textbf{0.999} & 0.502 \cr
        \textsc{LocalSearch} & 0.857 & 0.915 & 0.929 & 0.989 & 0.876 & \textbf{0.541} & 0.998 & 0.501 \cr
        \textsc{DataSummary} & 0.857 & 0.920 & 0.865 & \textbf{0.994} & 0.880 & 0.512 & 0.990 & 0.501 \cr
		\bottomrule  
	\end{tabular}  
 \end{table*}

\label{sec-exp-part1}

\begin{figure*}[h] 
 	\begin{center}

    \begin{subfigure}[b]{0.32\textwidth}
        \centering
        \includegraphics[width=1\textwidth]{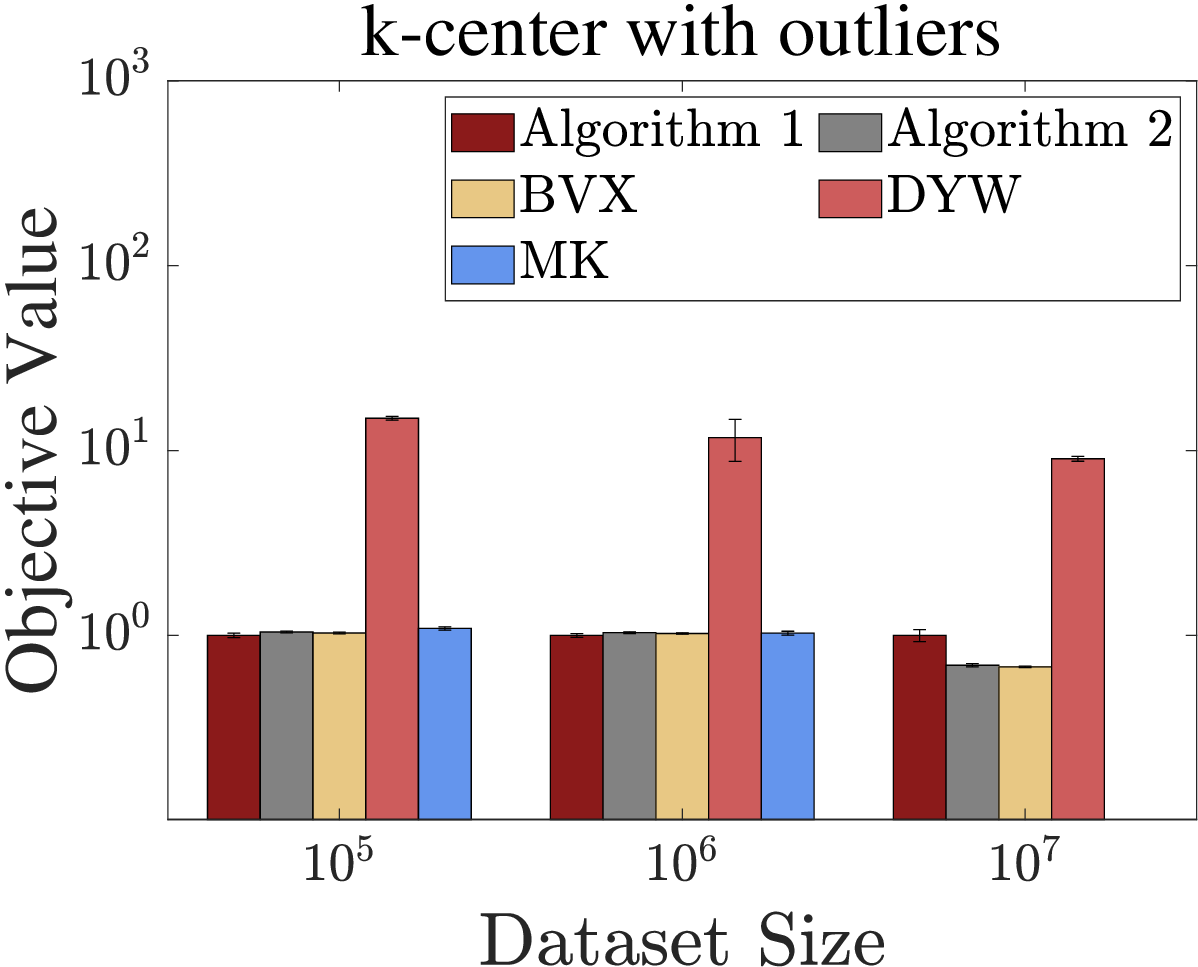} 
        \caption{}
    \end{subfigure}
    \begin{subfigure}[b]{0.32\textwidth}
        \centering
        \includegraphics[width=1\textwidth]{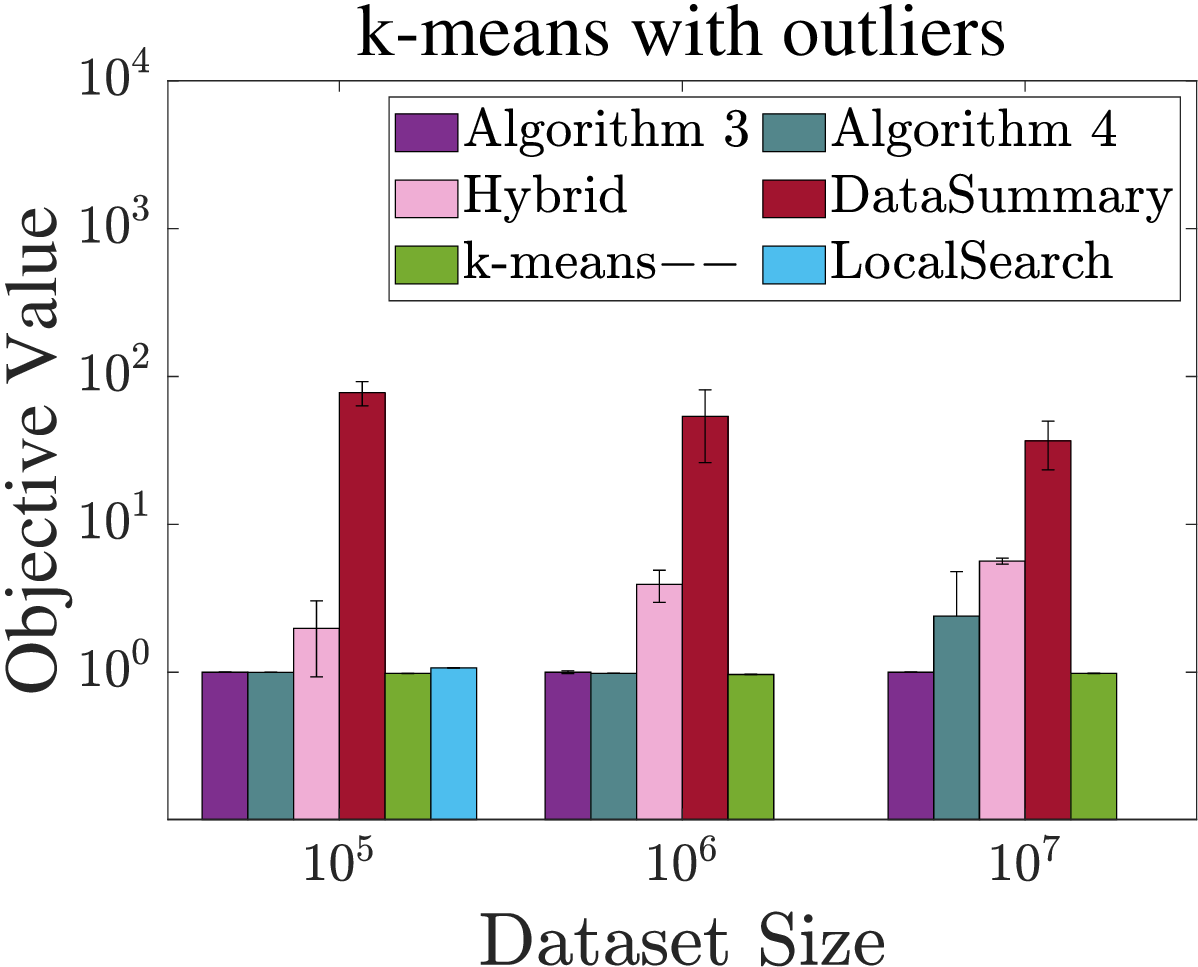} 
        \caption{}
    \end{subfigure}
    \begin{subfigure}[b]{0.32\textwidth}
        \centering
        \includegraphics[width=1\textwidth]{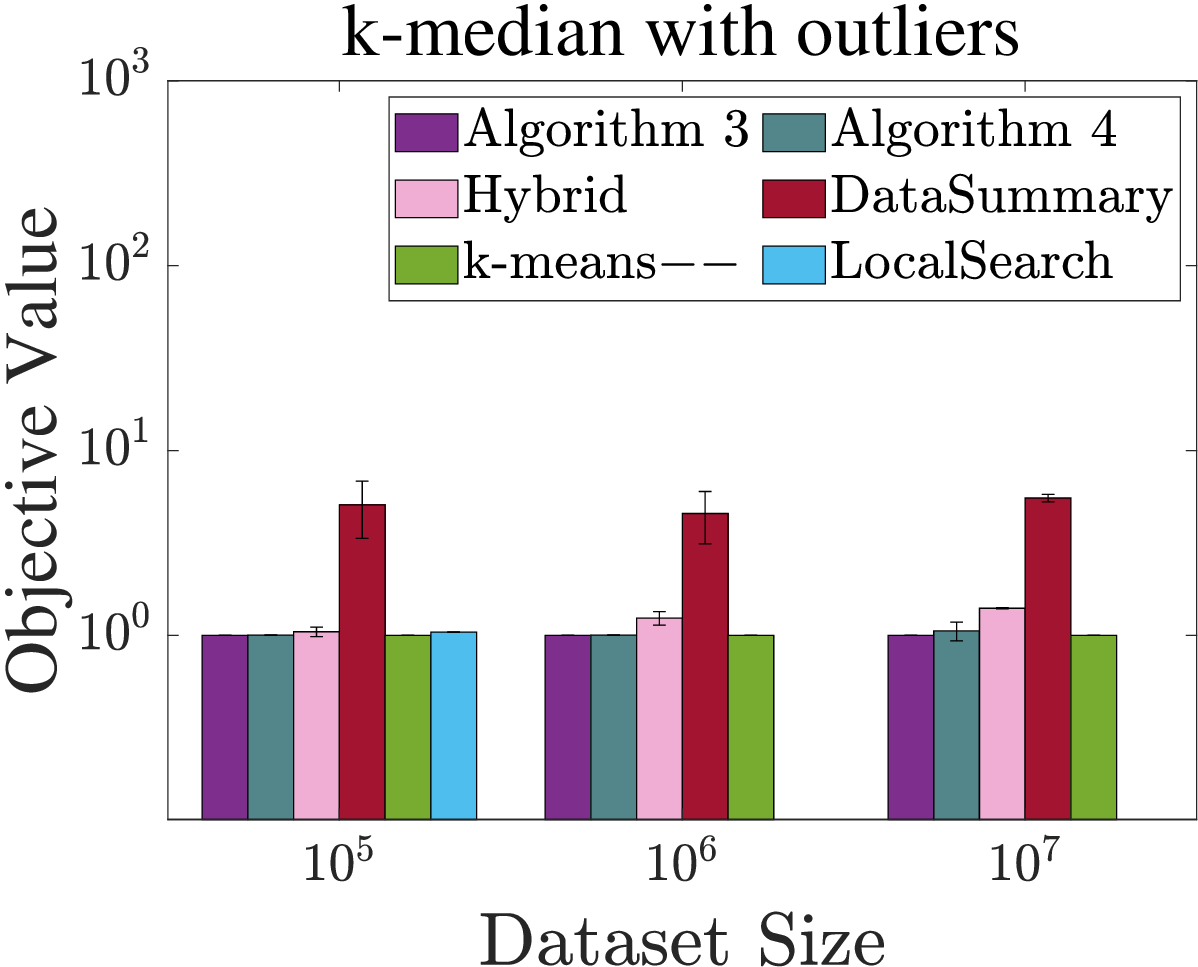}  
        \caption{}
    \end{subfigure}

    \vspace{0.2in}

    \begin{subfigure}[b]{0.32\textwidth}
        \centering
        \includegraphics[width=1\textwidth]{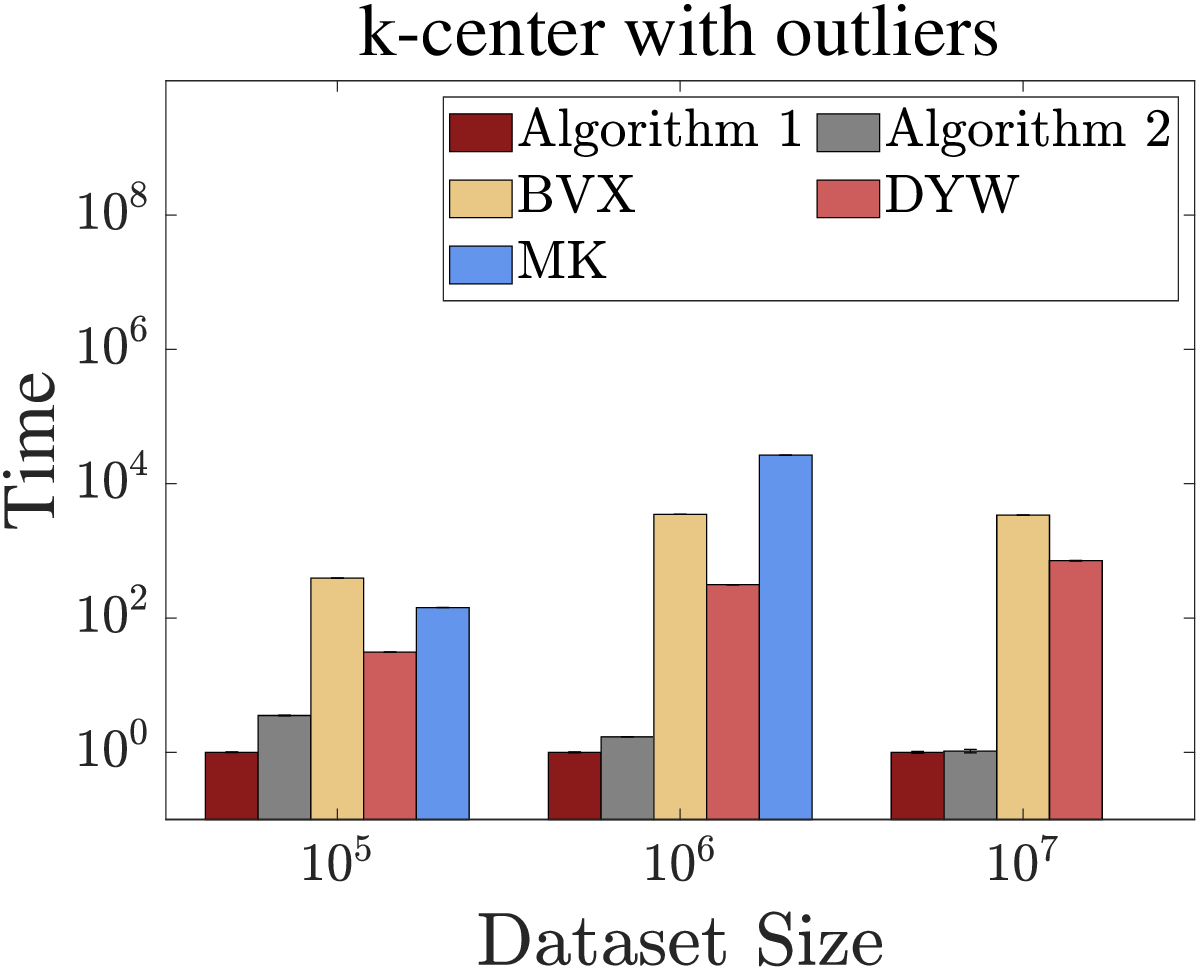}
        \caption{}
    \end{subfigure}
    \begin{subfigure}[b]{0.32\textwidth}
        \centering
        \includegraphics[width=1\textwidth]{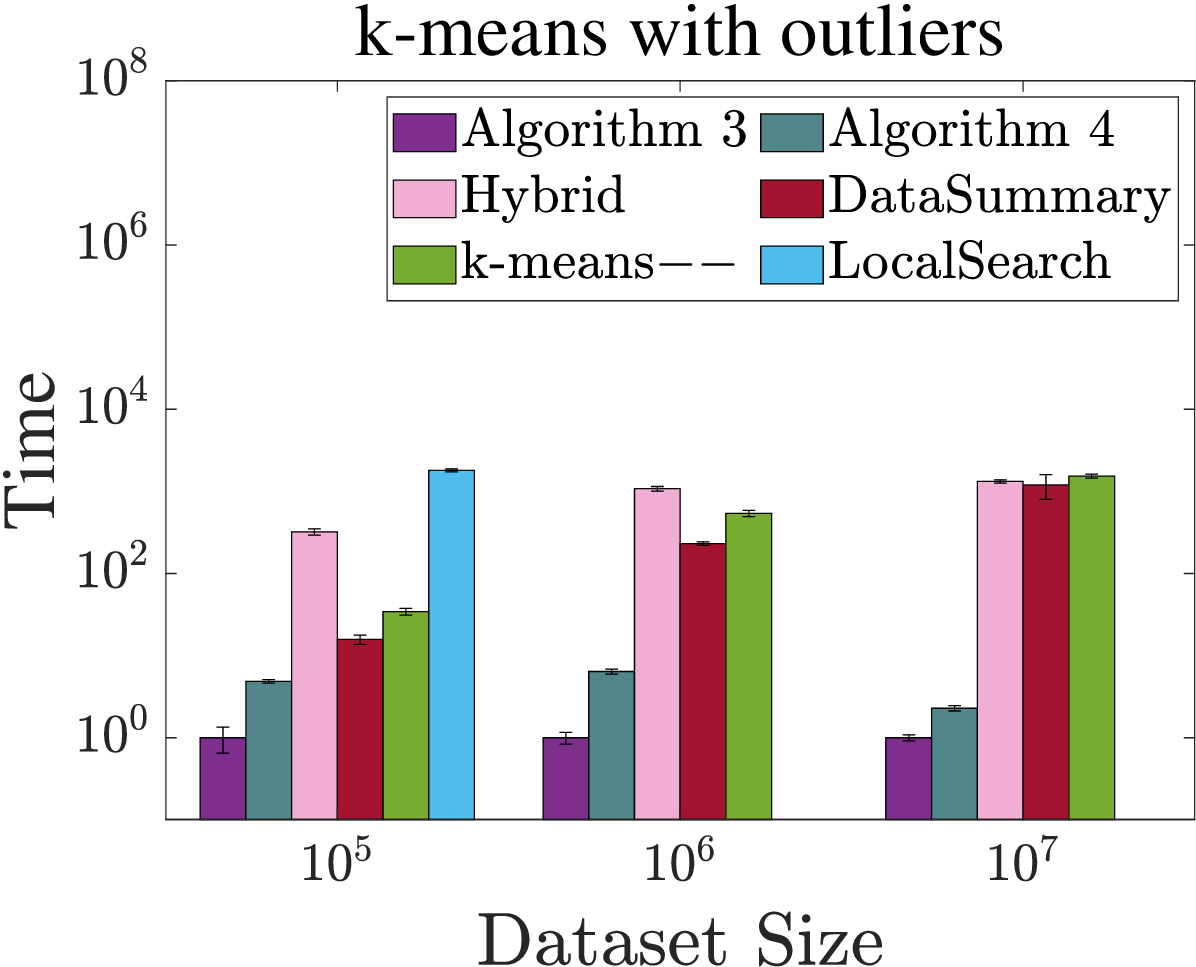}
        \caption{}
    \end{subfigure}
    \begin{subfigure}[b]{0.32\textwidth}
        \centering
        \includegraphics[width=1\textwidth]{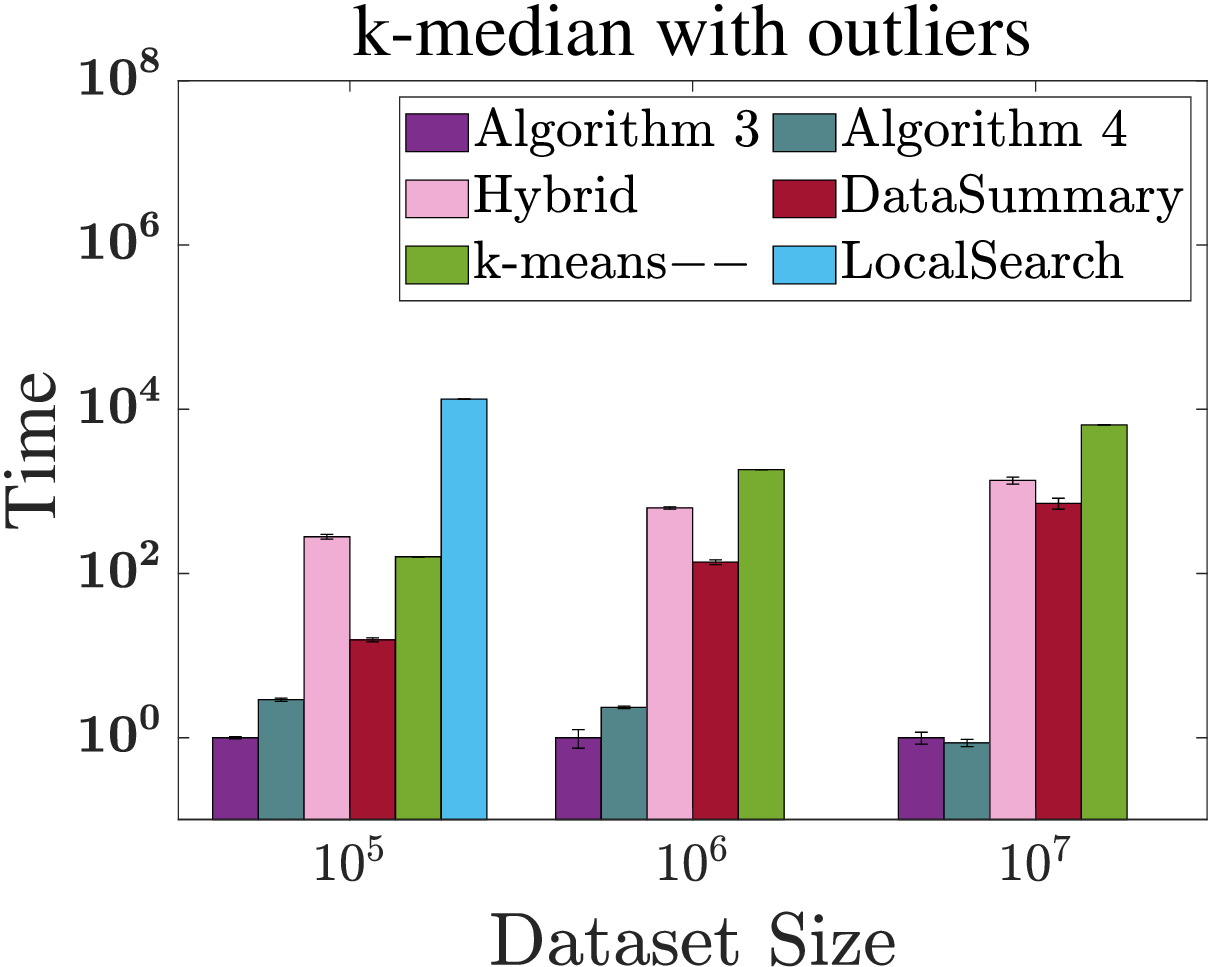} 
        \caption{}
    \end{subfigure}

		 \vspace{-0.1in}
		\caption{The normalized results (objective values and  running times) on the synthetic datasets with varying $n$. We did not run \textsc{Charikar},  \textsc{MK}, \textsc{Malkomes}, and \textsc{LocalSearch} for large $n$  because their running times are too high. To illustrate the comparisons more clearly, we normalize all the results by dividing them over the results of Algorithm~\ref{alg-kc1} with $\tau = 2$ ({\em resp.,} Algorithm~\ref{alg-km} with $\tau = 2$ ) for $k$-center ({\em resp.,} $k$-means/median ) clustering with outliers.  }     
		\label{fig-exp-size-sup}
	\end{center}
\end{figure*}

\begin{figure*}[h] 
	\begin{center}
			 \vspace{-0.1in}

     \begin{subfigure}[b]{0.24\textwidth}
         \centering
         \includegraphics[height=0.73\textwidth]{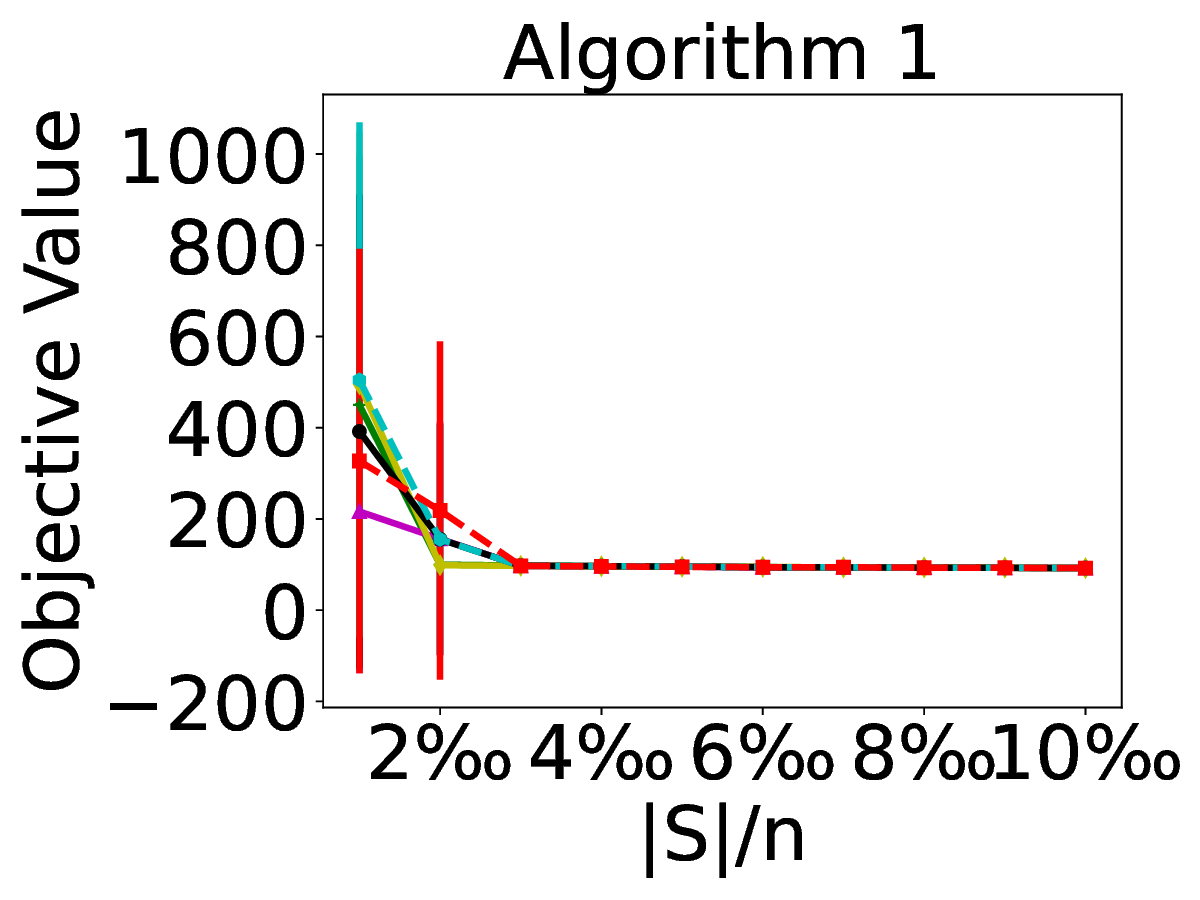} 
         \caption{}
     \end{subfigure}%
     \begin{subfigure}[b]{0.24\textwidth}
         \centering
         \includegraphics[height=0.73\textwidth]{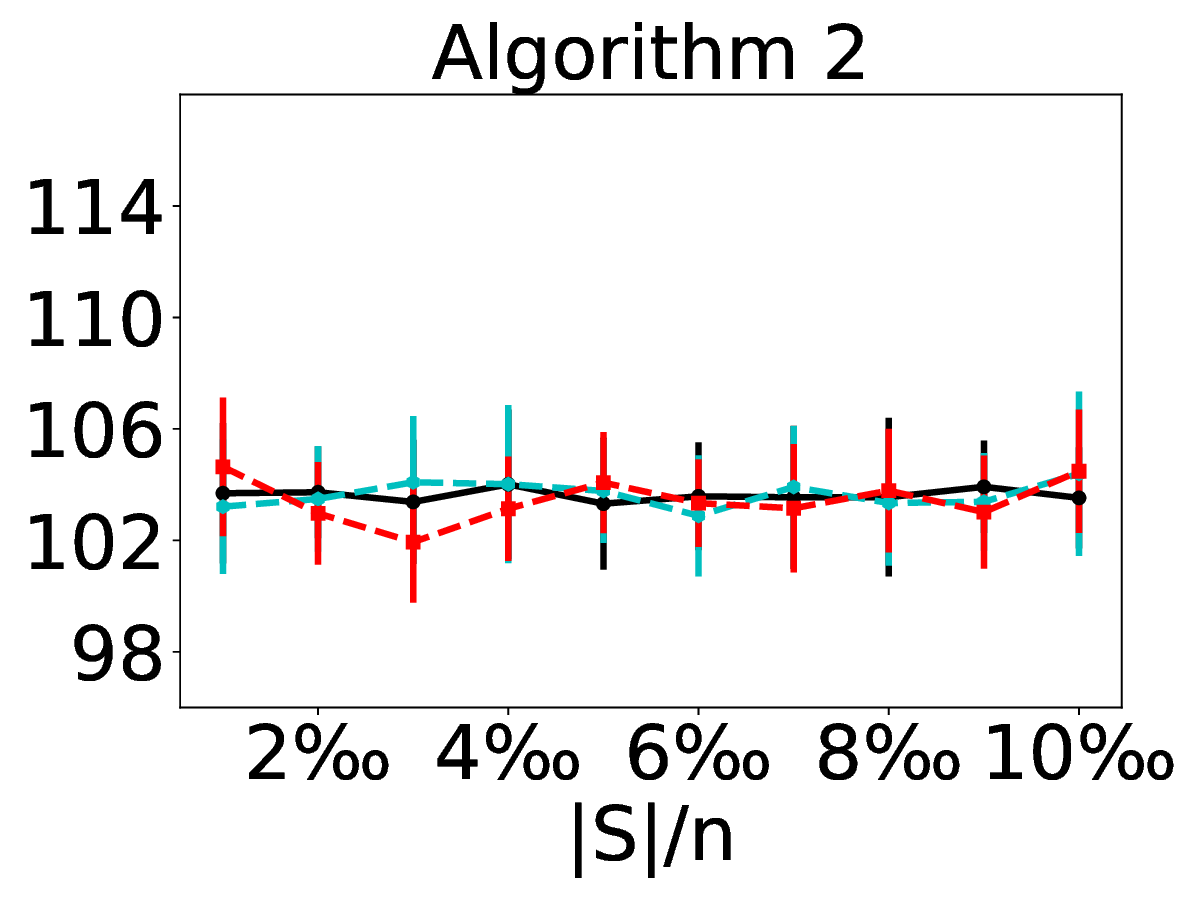} 
         \caption{}
     \end{subfigure}%
     \begin{subfigure}[b]{0.24\textwidth}
         \centering
         \includegraphics[height=0.73\textwidth]{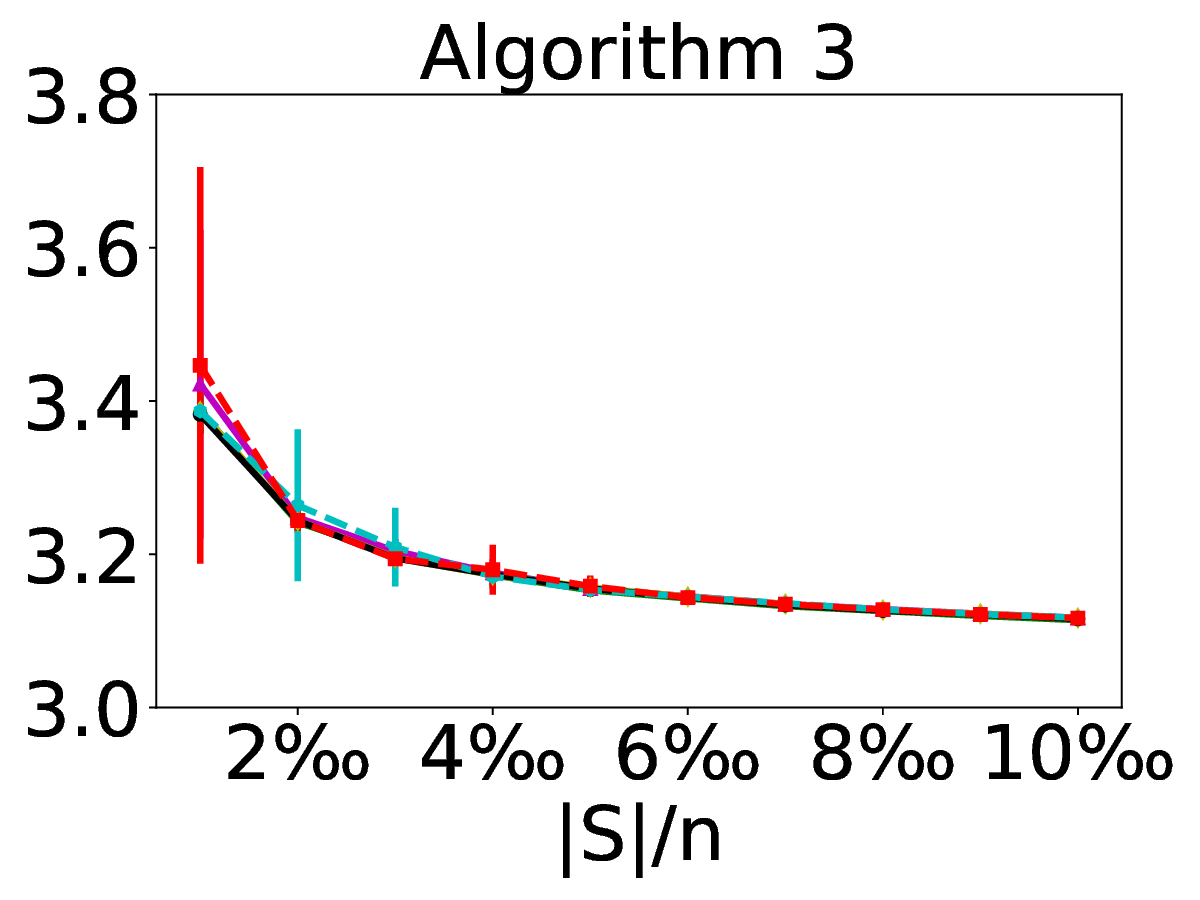}  
         \caption{}
     \end{subfigure}%
     \begin{subfigure}[b]{0.24\textwidth}
         \centering
         \includegraphics[height=0.73\textwidth]{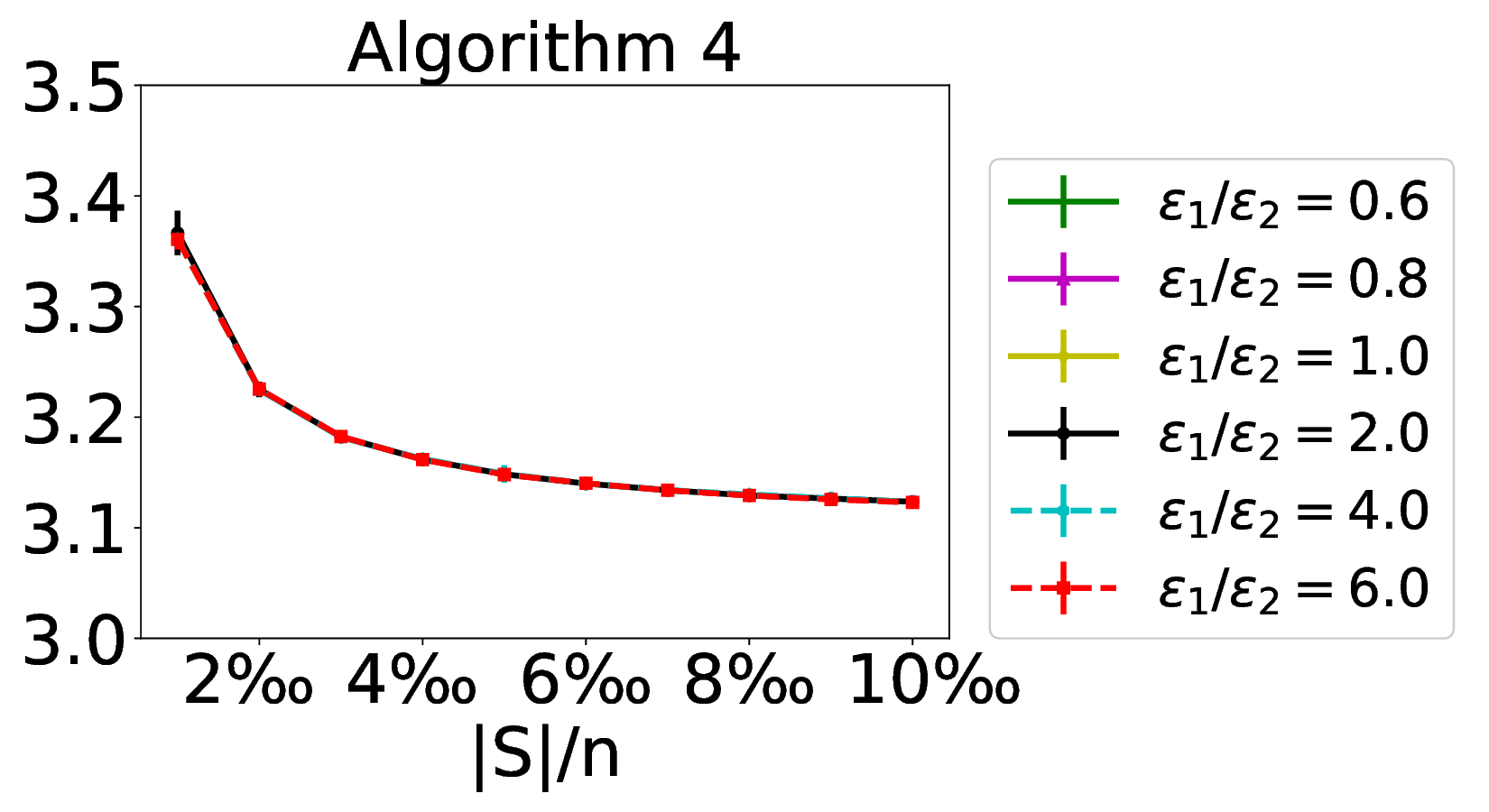}
         \caption{}
     \end{subfigure}%
     
     \vspace{0.2in}
     \begin{subfigure}[b]{0.25\textwidth}
         \centering
         \includegraphics[height=0.73\textwidth]{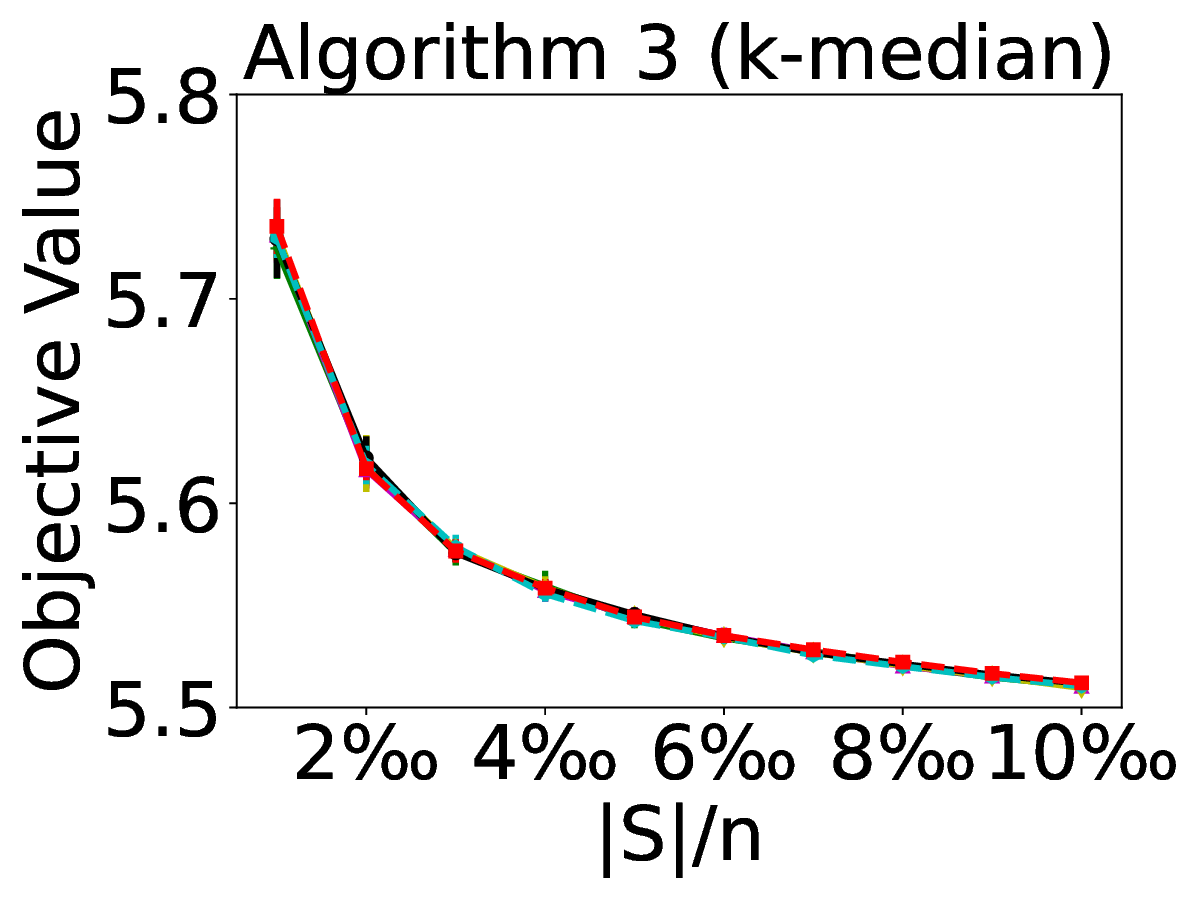} 
         \caption{}
     \end{subfigure}%
     \begin{subfigure}[b]{0.25\textwidth}
         \centering
         \includegraphics[height=0.73\textwidth]{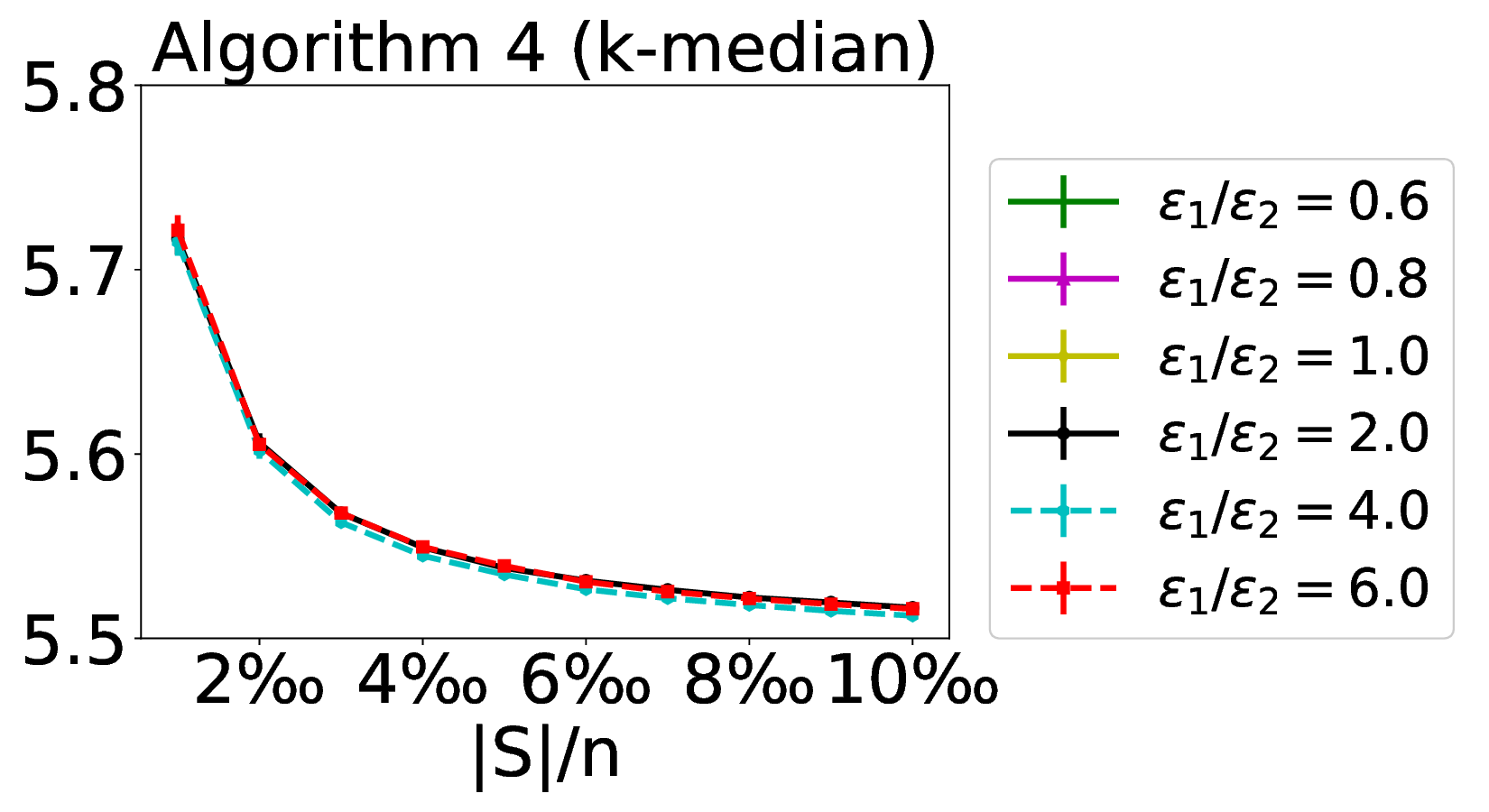}
         \caption{}
     \end{subfigure}%
				 \vspace{-0.1in}
		  \caption{The performances with varying the sampling ratio $|S|/n$ (we run algorithm~\ref{alg-kc2} and \ref{alg-km2} for $\epsilon_1/\epsilon_2>1$). }     
		\label{fig-exp-sample-sup}
	\end{center}
			\vspace{-0.2in}
\end{figure*}


\begin{figure*}[h] 
	\begin{center}


        
        \begin{subfigure}[b]{0.5\textwidth}
            \centering
            \includegraphics[height=0.6\textwidth]{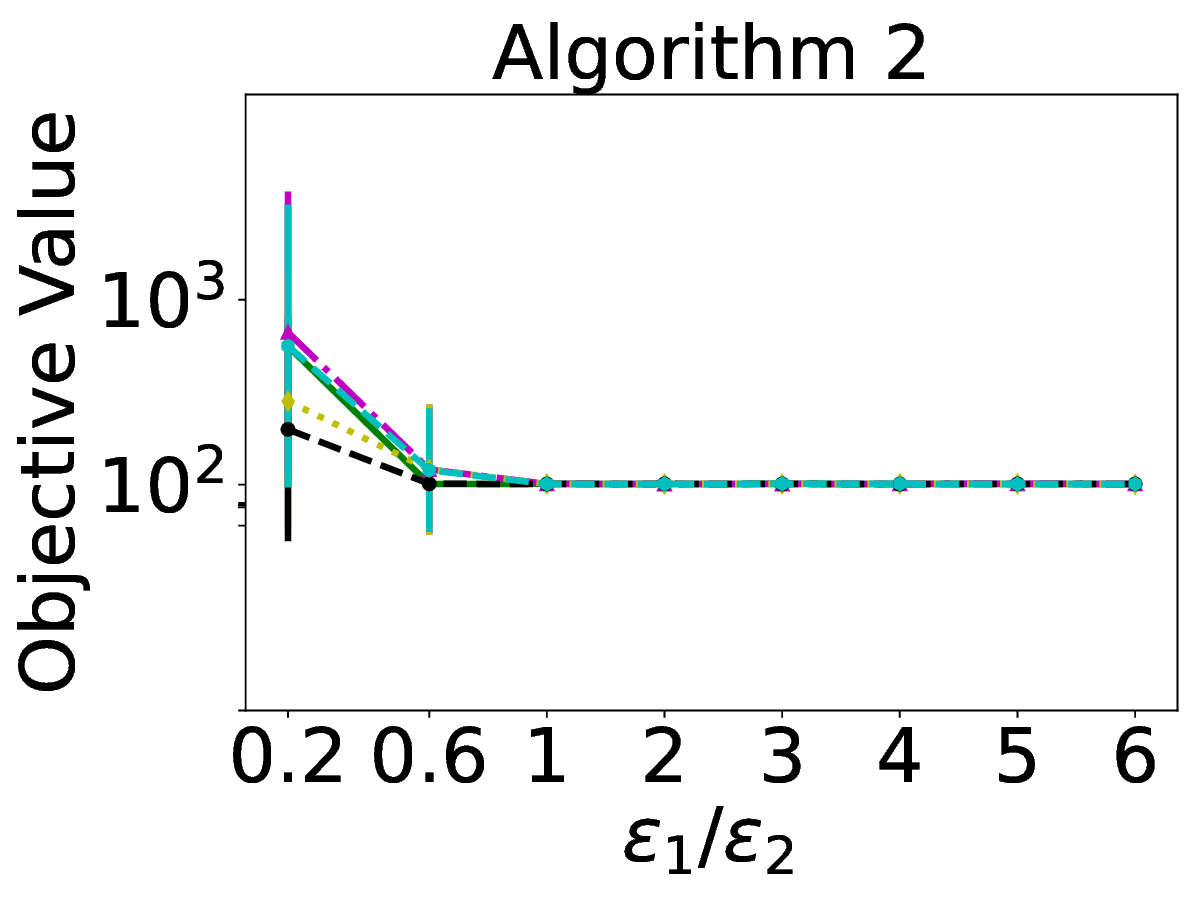} 
            \caption{Varying $\epsilon_1/\epsilon_2$ for Algorithm~\ref{alg-kc2}}
        \end{subfigure}%
        \begin{subfigure}[b]{0.5\textwidth}
            \centering
            \includegraphics[height=0.6\textwidth]{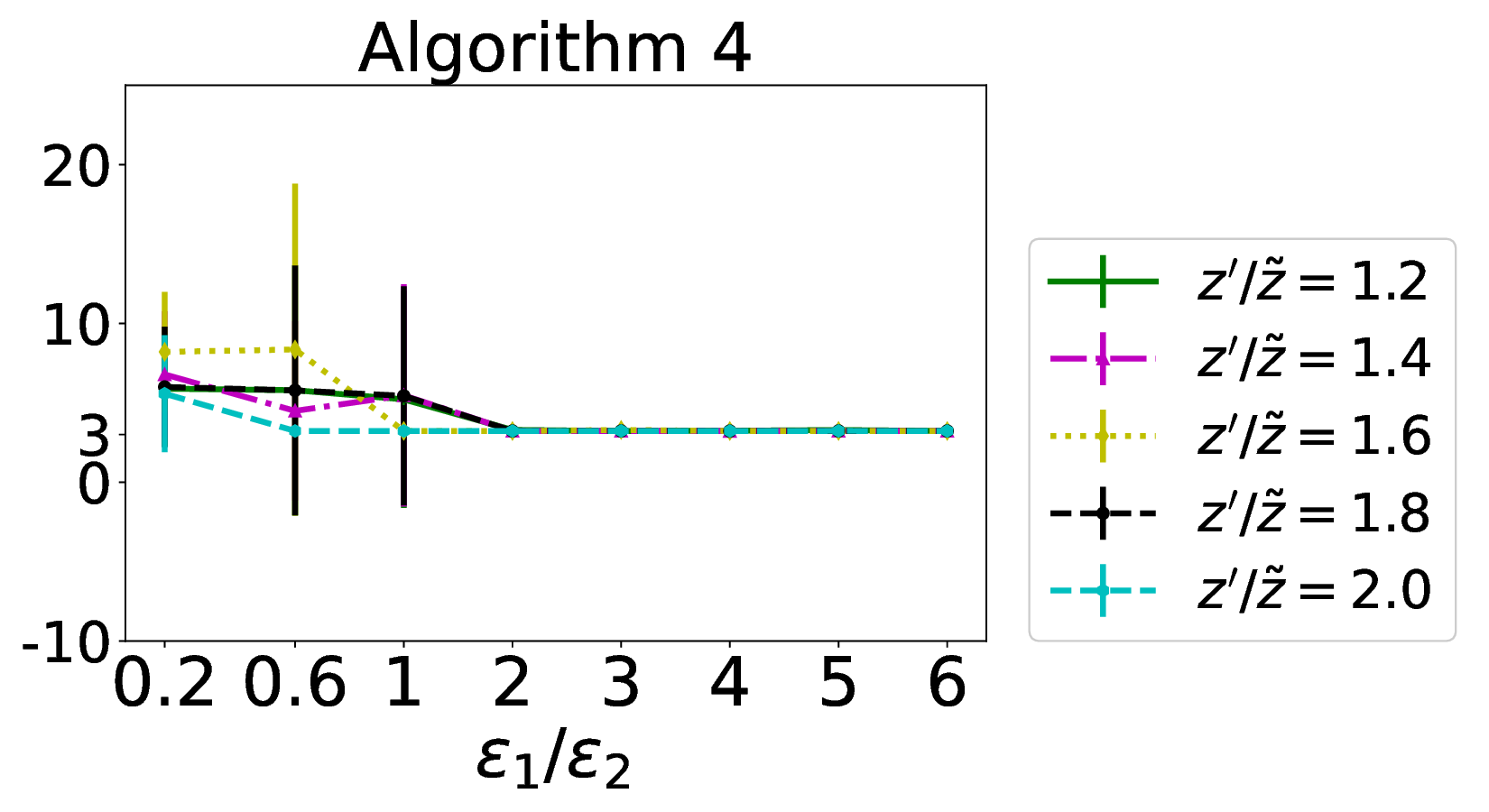} 
            \caption{Varying $\epsilon_1/\epsilon_2$ for Algorithm~\ref{alg-km2}}
        \end{subfigure}%
        \vspace{0.2in}
        \begin{subfigure}[b]{0.5\textwidth}
            \centering
            \includegraphics[height=0.6\textwidth]{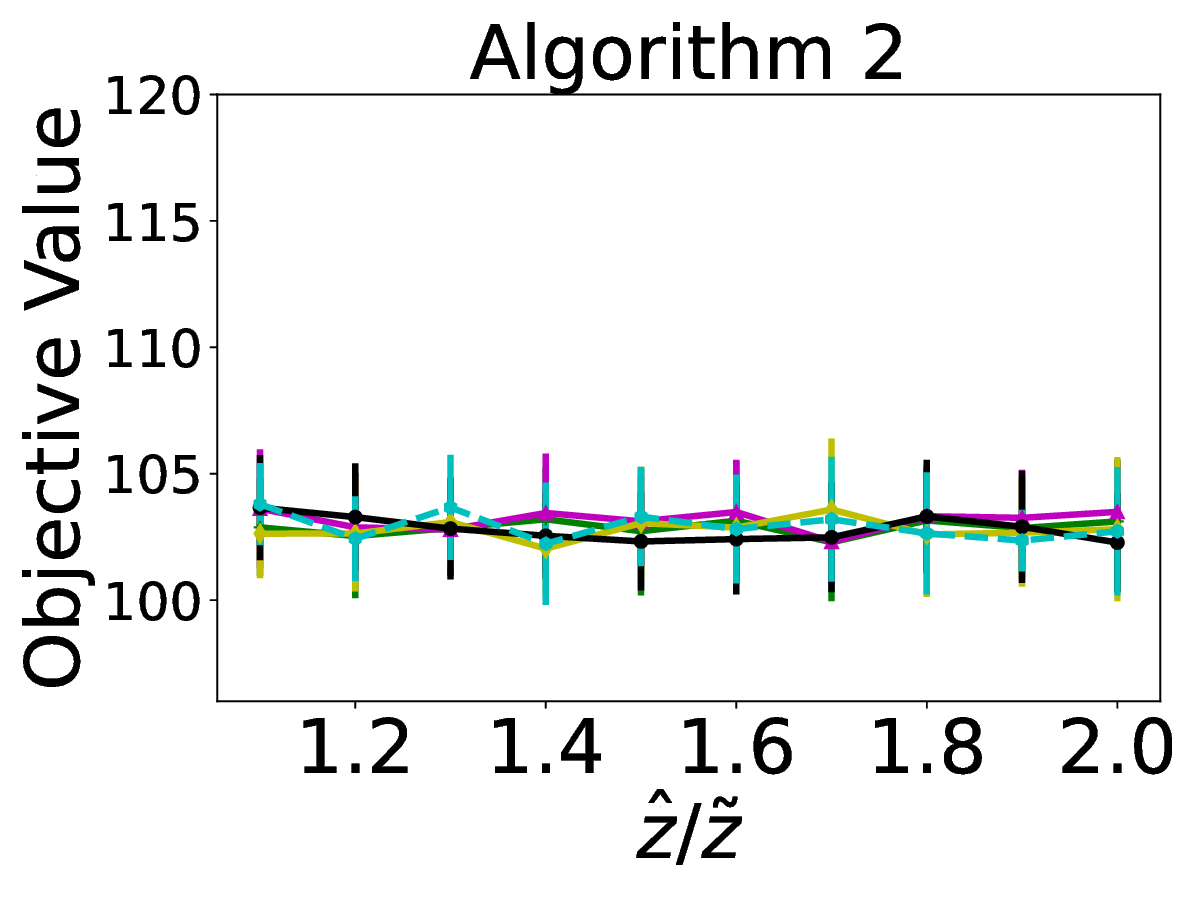} 
            \caption{Varying $\hat{z}/\tilde{z}$  for Algorithm~\ref{alg-kc2}}
        \end{subfigure}%
        \begin{subfigure}[b]{0.5\textwidth}
            \centering
            \includegraphics[height=0.6\textwidth]{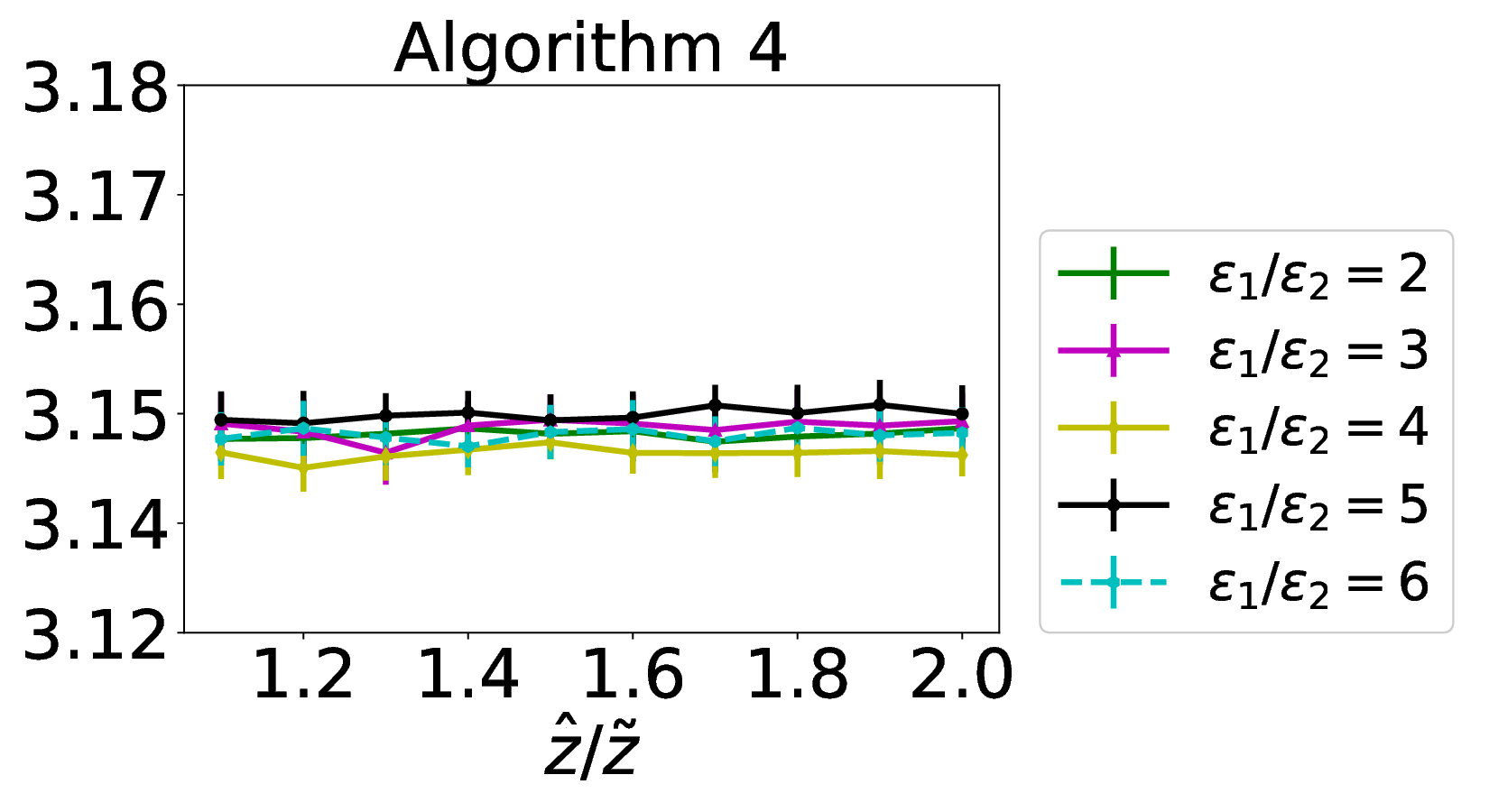} 
            \caption{Varying $\hat{z}/\tilde{z}$  for Algorithm~\ref{alg-km2}}
        \end{subfigure}%
        
		\caption{The  average objective values and standard deviations with varying $\epsilon_1/\epsilon_2$ and $\hat{z}/\tilde{z}$. 
$\tilde{z}=\frac{\epsilon_2}{k}|S|$ is the expected number of outliers contained in $S$.}     
		\label{fig-stability-sup}
	\end{center}
\end{figure*}

\begin{figure*}[h]
	\begin{center}


    \begin{subfigure}[b]{0.38\textwidth}
        \centering
        \includegraphics[width=1\textwidth]{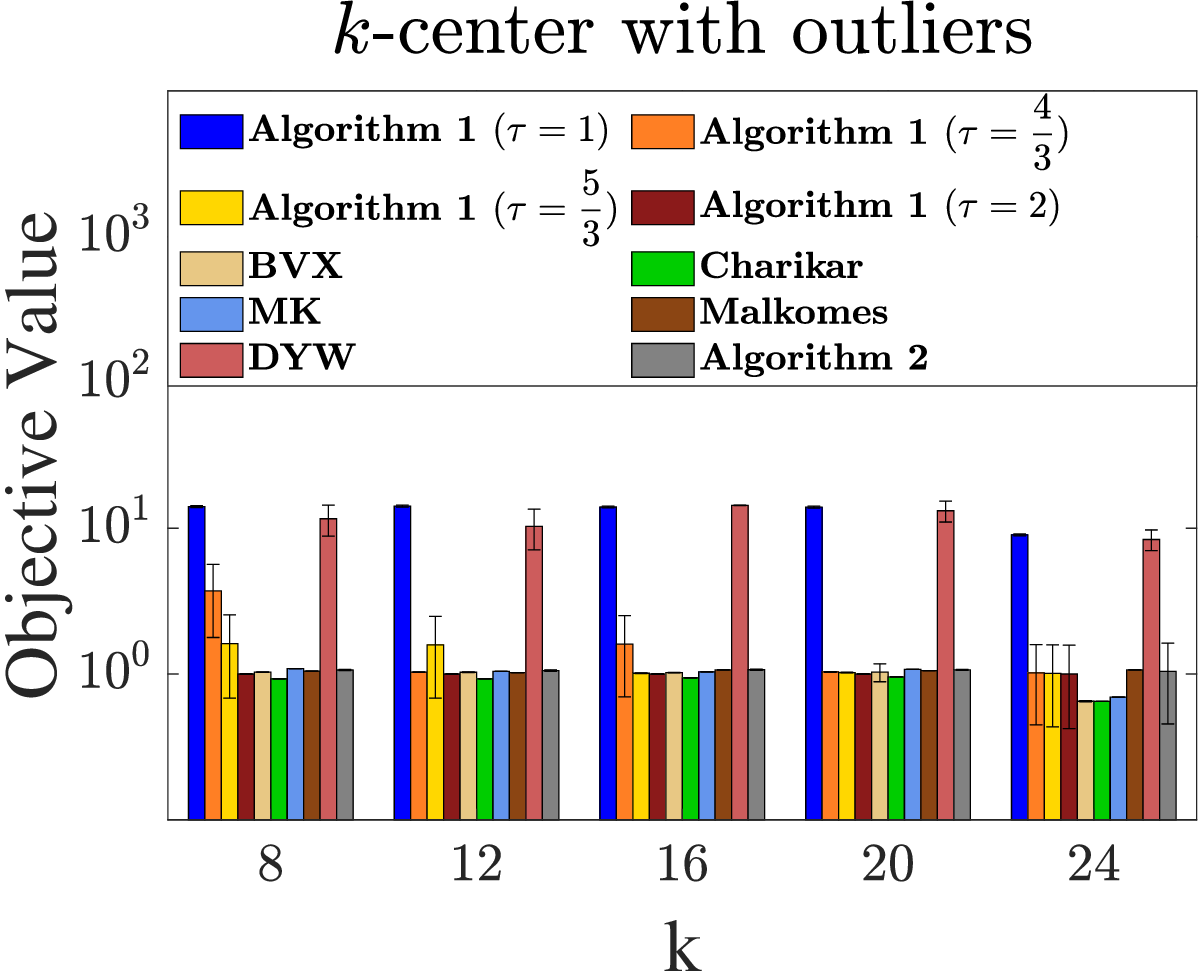}  
        \caption{}
    \end{subfigure}
    \hspace{0.2in}
    \begin{subfigure}[b]{0.38\textwidth}
        \centering
        \includegraphics[width=1\textwidth]{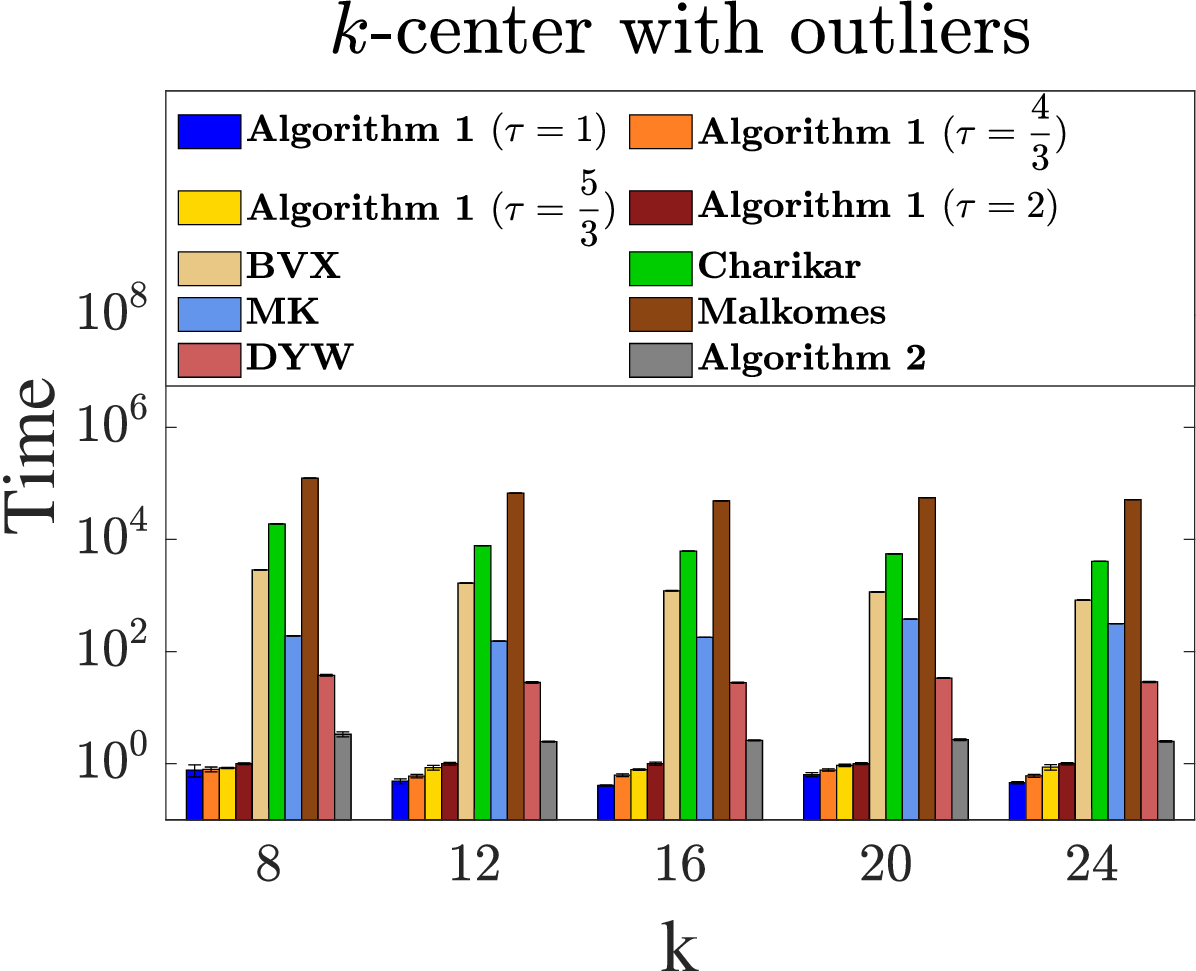} 
        \caption{}
    \end{subfigure}
    
        \vspace{0.1in}
    
    \begin{subfigure}[b]{0.38\textwidth}
        \centering
        \includegraphics[width=1\textwidth]{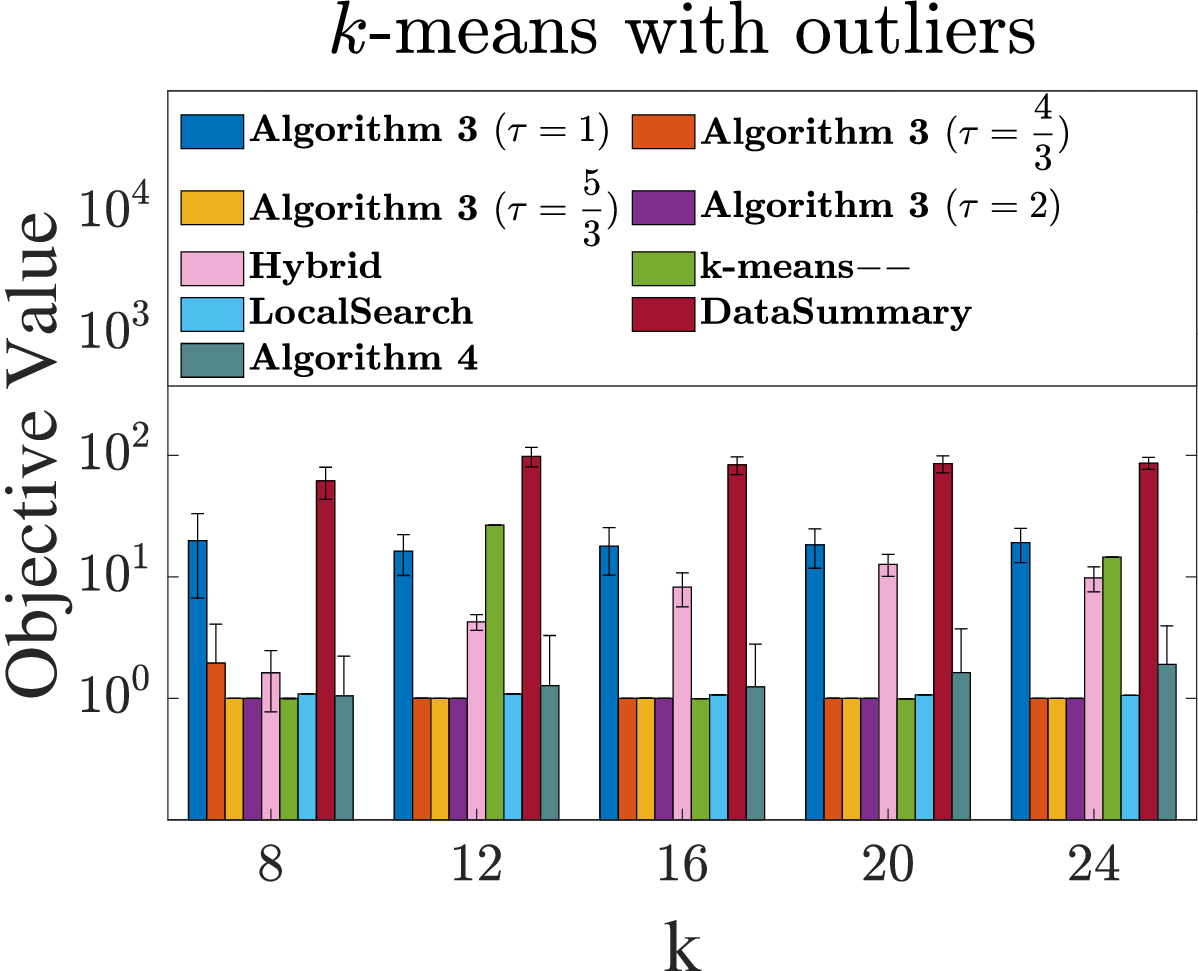}  
        \caption{}
    \end{subfigure}
    \hspace{0.2in}
    \begin{subfigure}[b]{0.38\textwidth}
        \centering
        \includegraphics[width=1\textwidth]{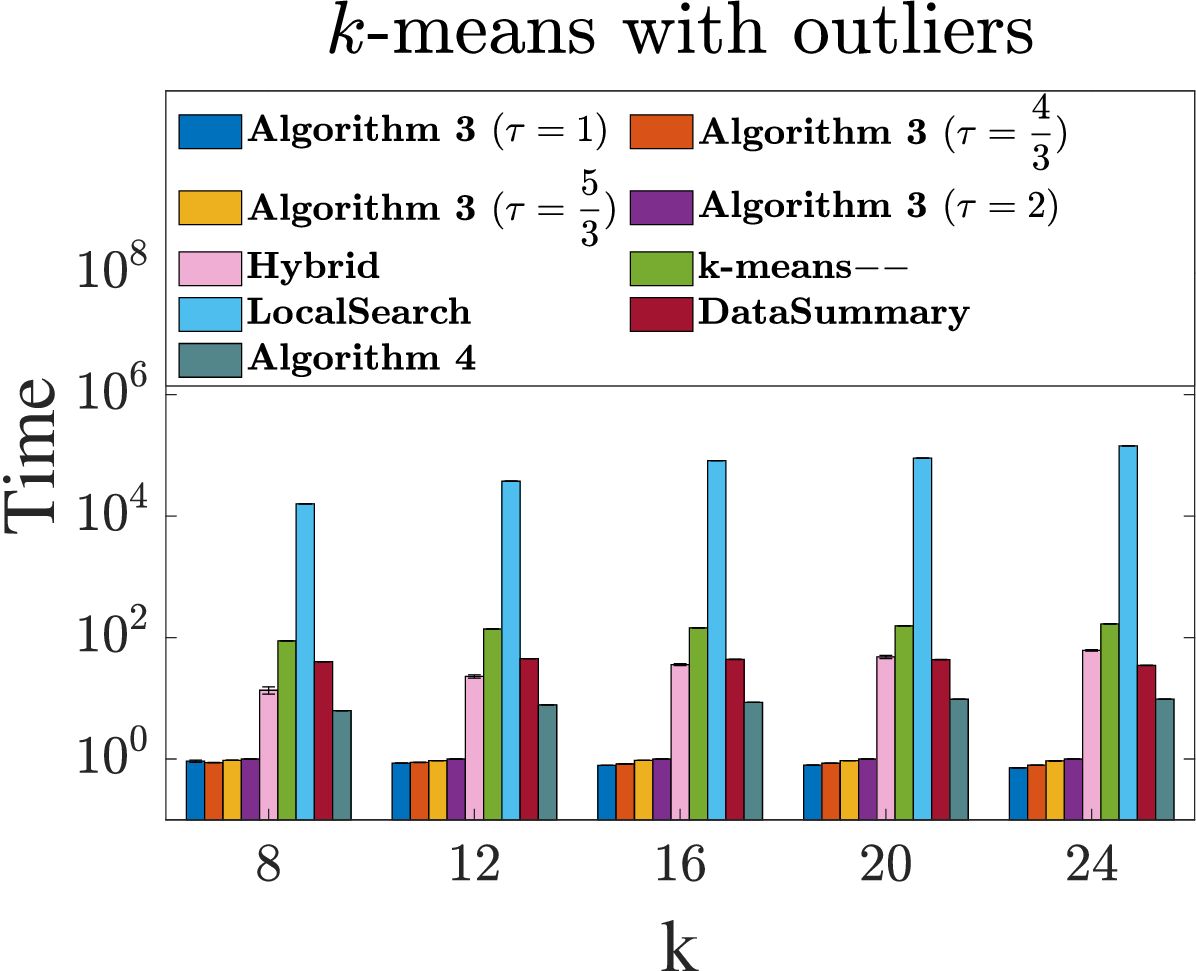}
        \caption{}
    \end{subfigure}
    
        \vspace{0.1in}

    \begin{subfigure}[b]{0.38\textwidth}
        \centering
        \includegraphics[width=1\textwidth]{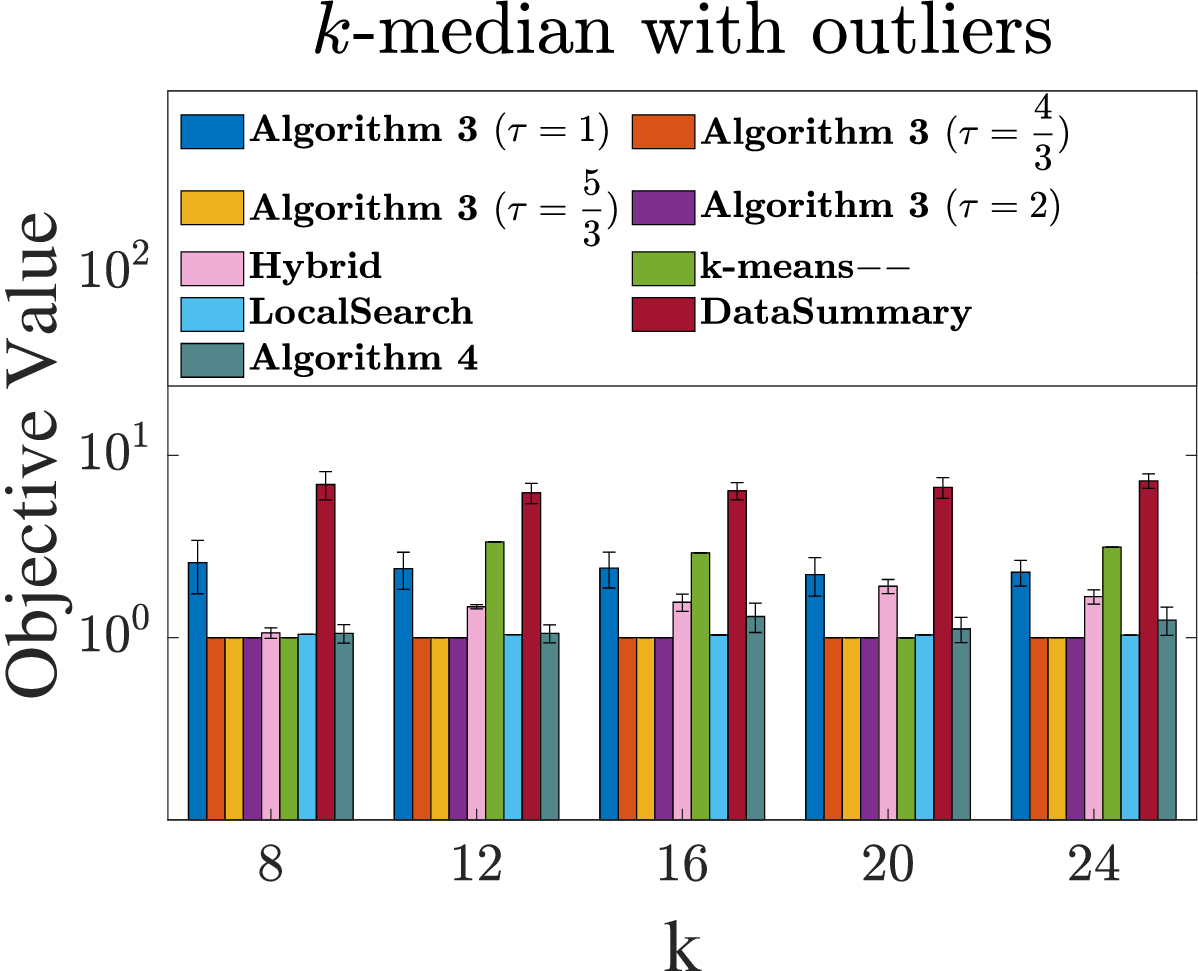}
        \caption{}
    \end{subfigure}
    \hspace{0.2in}
    \begin{subfigure}[b]{0.38\textwidth}
        \centering
        \includegraphics[width=1\textwidth]{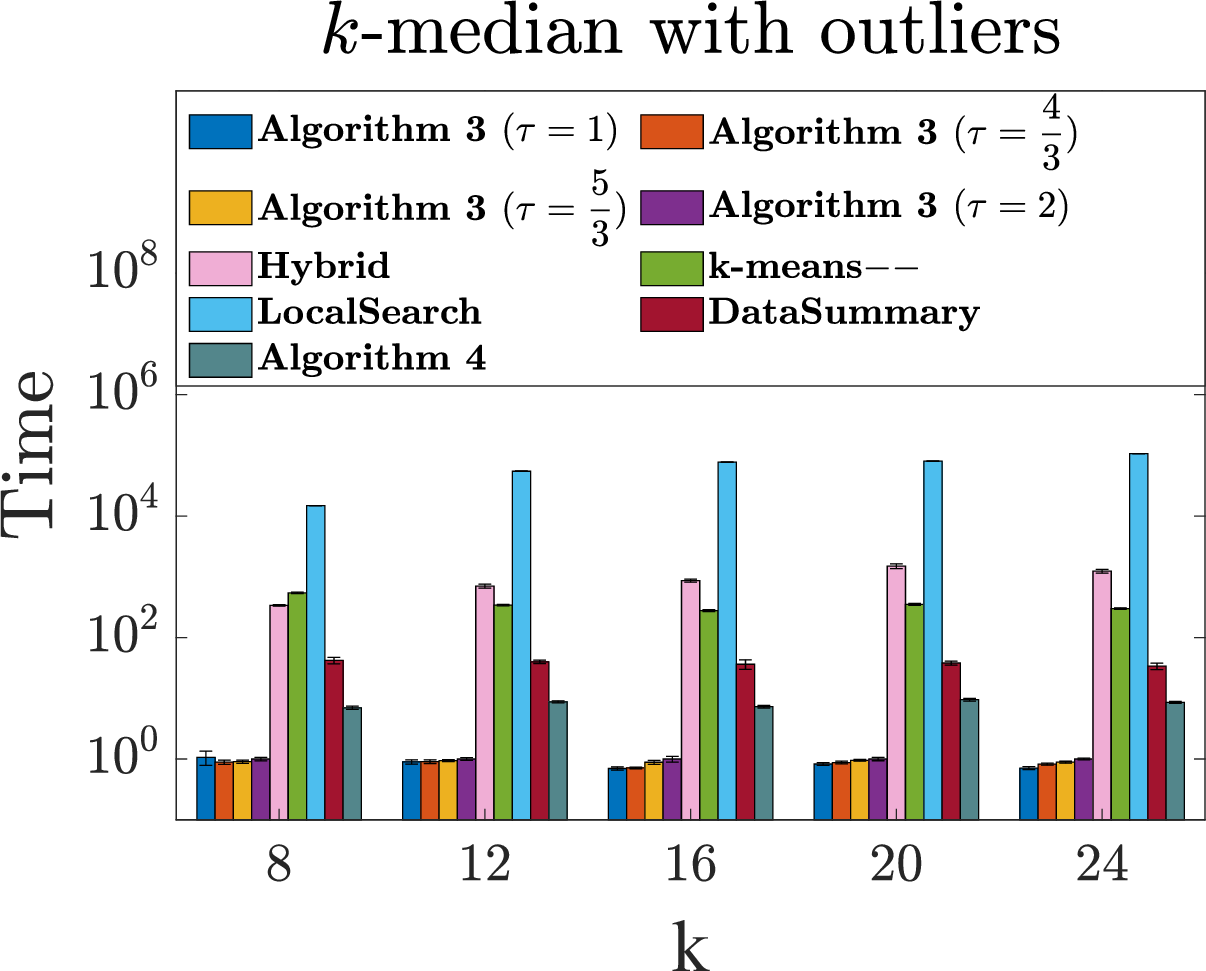} %
        \caption{}
    \end{subfigure}

    \caption{The normalized results (objective values and running times) on the synthetic datasets with varying $k$ (we set $z=2\%n$ and $\epsilon_1/\epsilon_2=2$). To illustrate the comparisons more clearly, we normalize all the results by dividing them over the results of Algorithm~\ref{alg-kc1} with $\tau = 2$ ({\em resp.,} Algorithm~\ref{alg-km} with $\tau = 2$ ) for $k$-center ({\em resp.,} $k$-means/median ) clustering with outliers. }     
	\label{fig-exp-k-sup}
	\end{center}
\end{figure*}

\vspace{-12in}
\begin{figure*}[h]
 	\begin{center}

    \begin{subfigure}[b]{0.38\textwidth}
        \centering
        \includegraphics[width=1\textwidth]{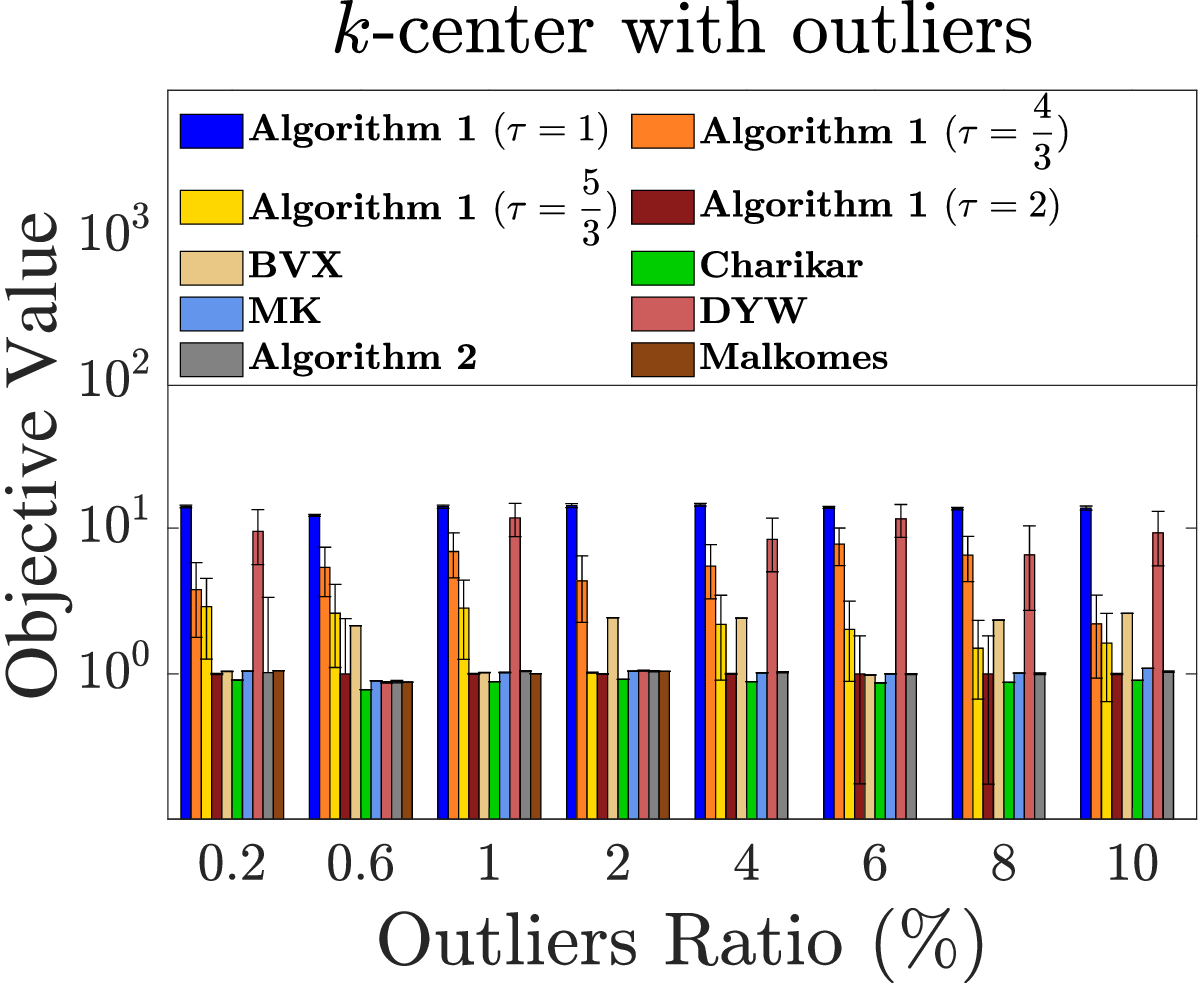}  
        \caption{}
    \end{subfigure}
    \hspace{0.2in}
    \begin{subfigure}[b]{0.38\textwidth}
        \centering
        \includegraphics[width=1\textwidth]{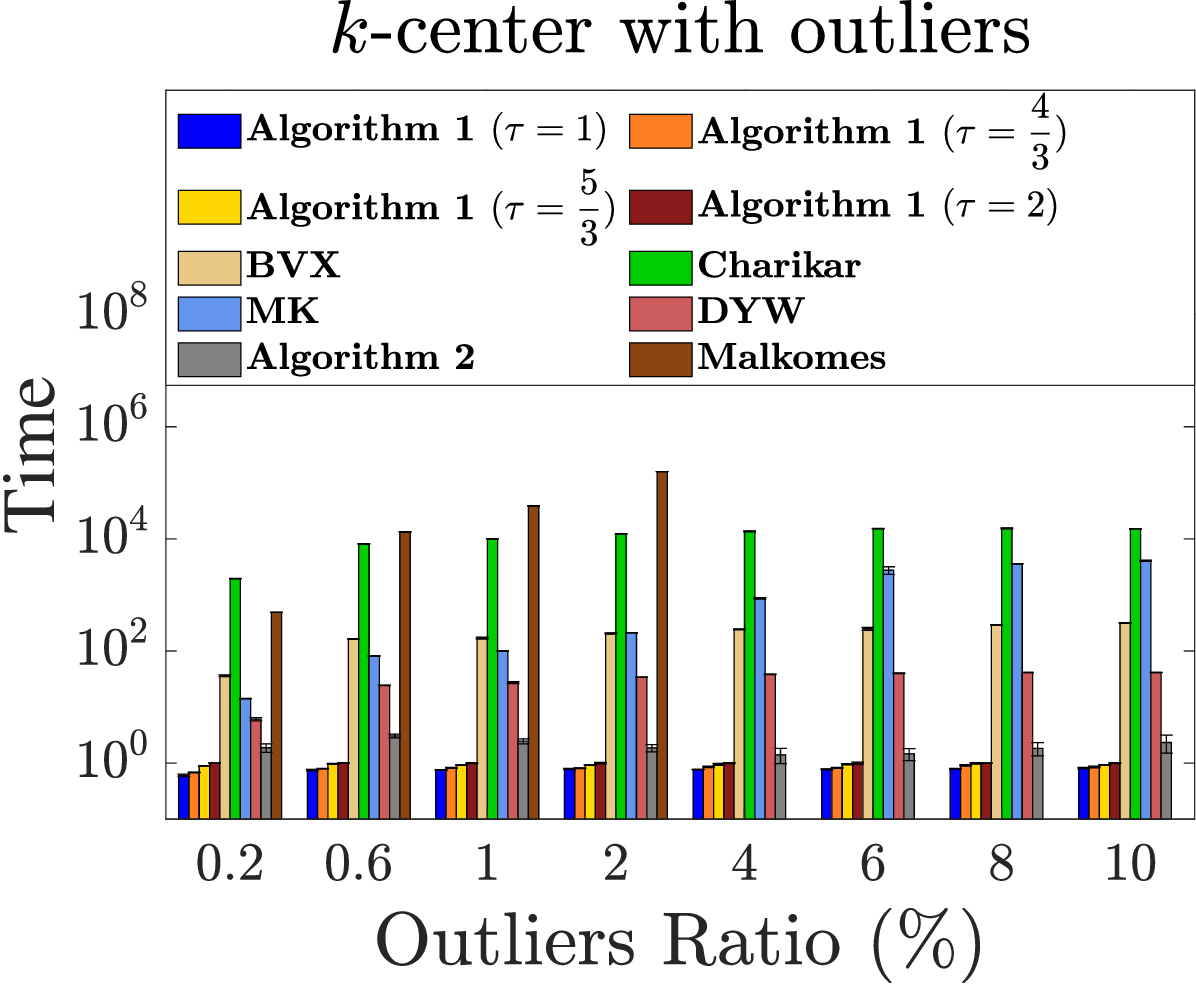} 
        \caption{}
    \end{subfigure}
    
        \vspace{0.1in}
    
    \begin{subfigure}[b]{0.38\textwidth}
        \centering
        \includegraphics[width=1\textwidth]{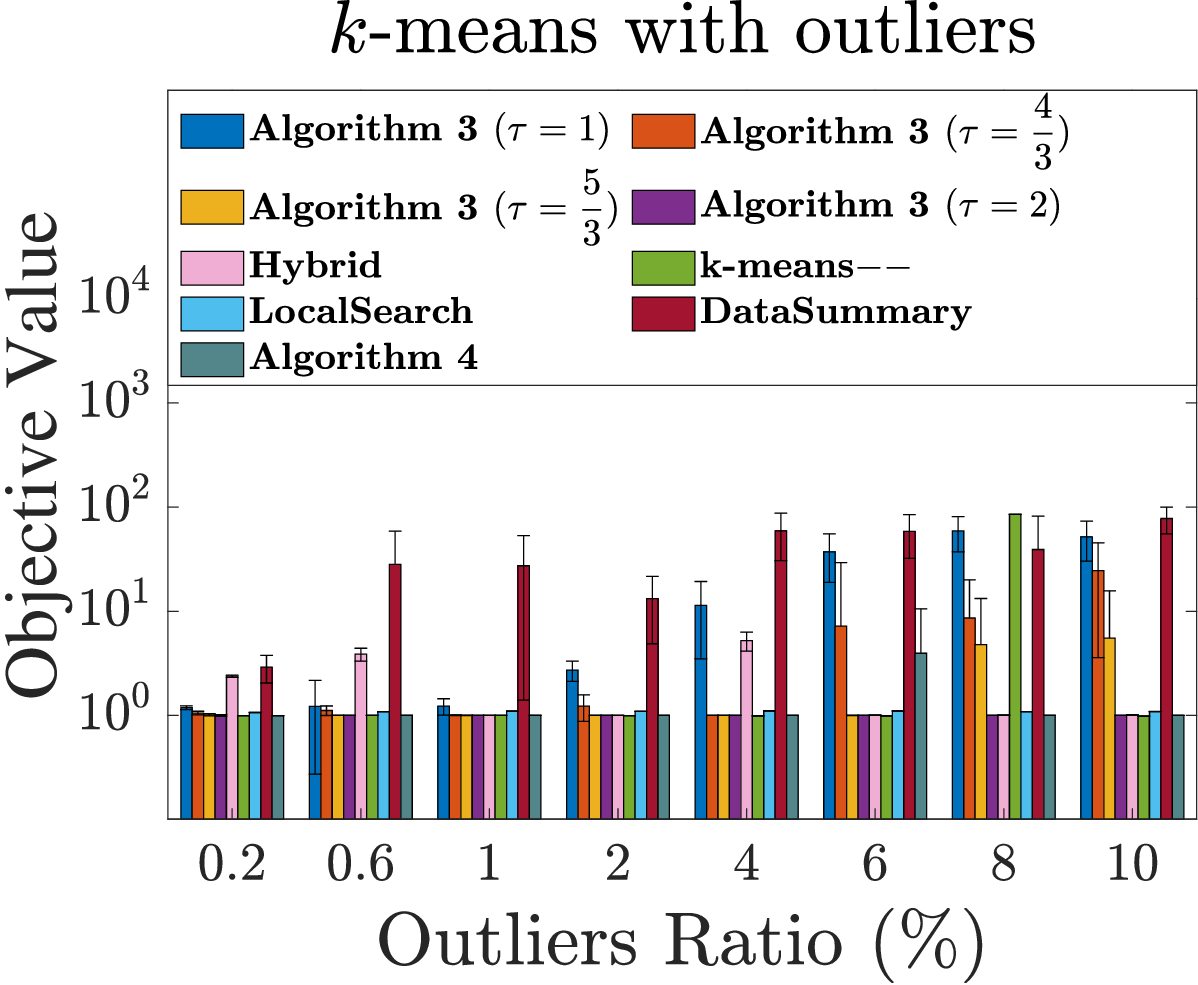}  
        \caption{}
    \end{subfigure}
    \hspace{0.2in}
    \begin{subfigure}[b]{0.38\textwidth}
        \centering
        \includegraphics[width=1\textwidth]{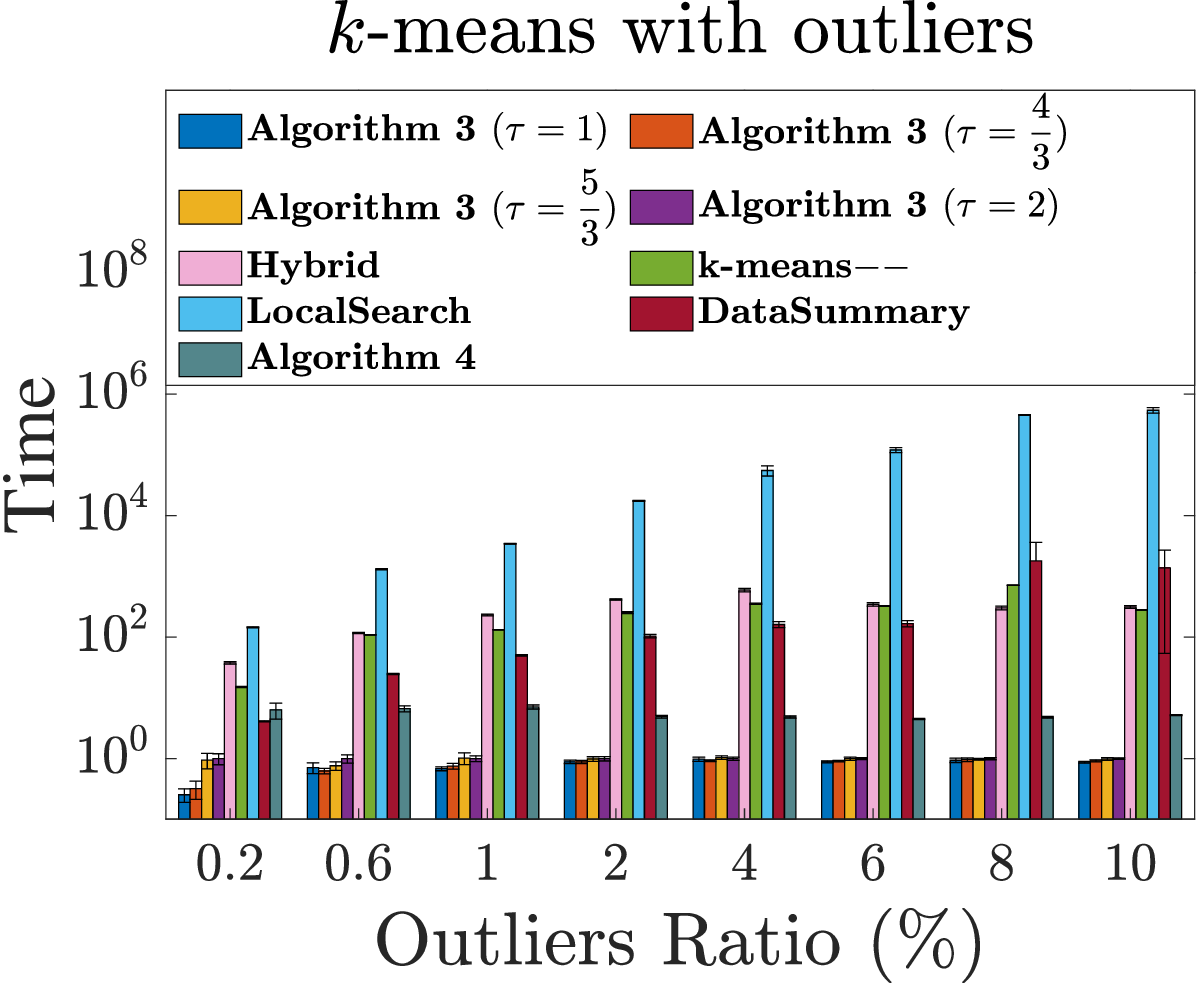}
        \caption{}
    \end{subfigure}
    
        \vspace{0.1in}

    \begin{subfigure}[b]{0.38\textwidth}
        \centering
        \includegraphics[width=1\textwidth]{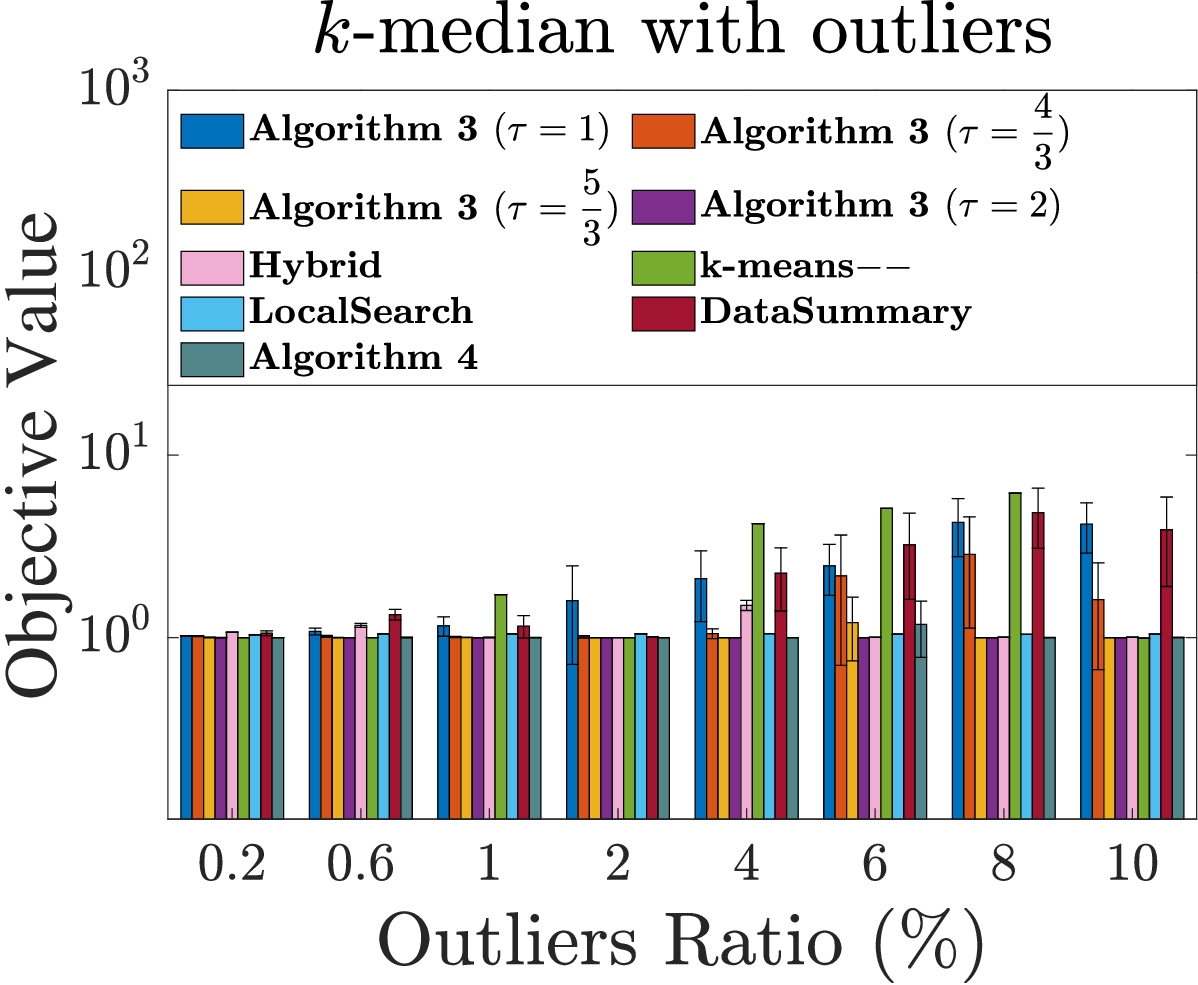}
        \caption{}
    \end{subfigure}
    \hspace{0.2in}
    \begin{subfigure}[b]{0.38\textwidth}
        \centering
        \includegraphics[width=1\textwidth]{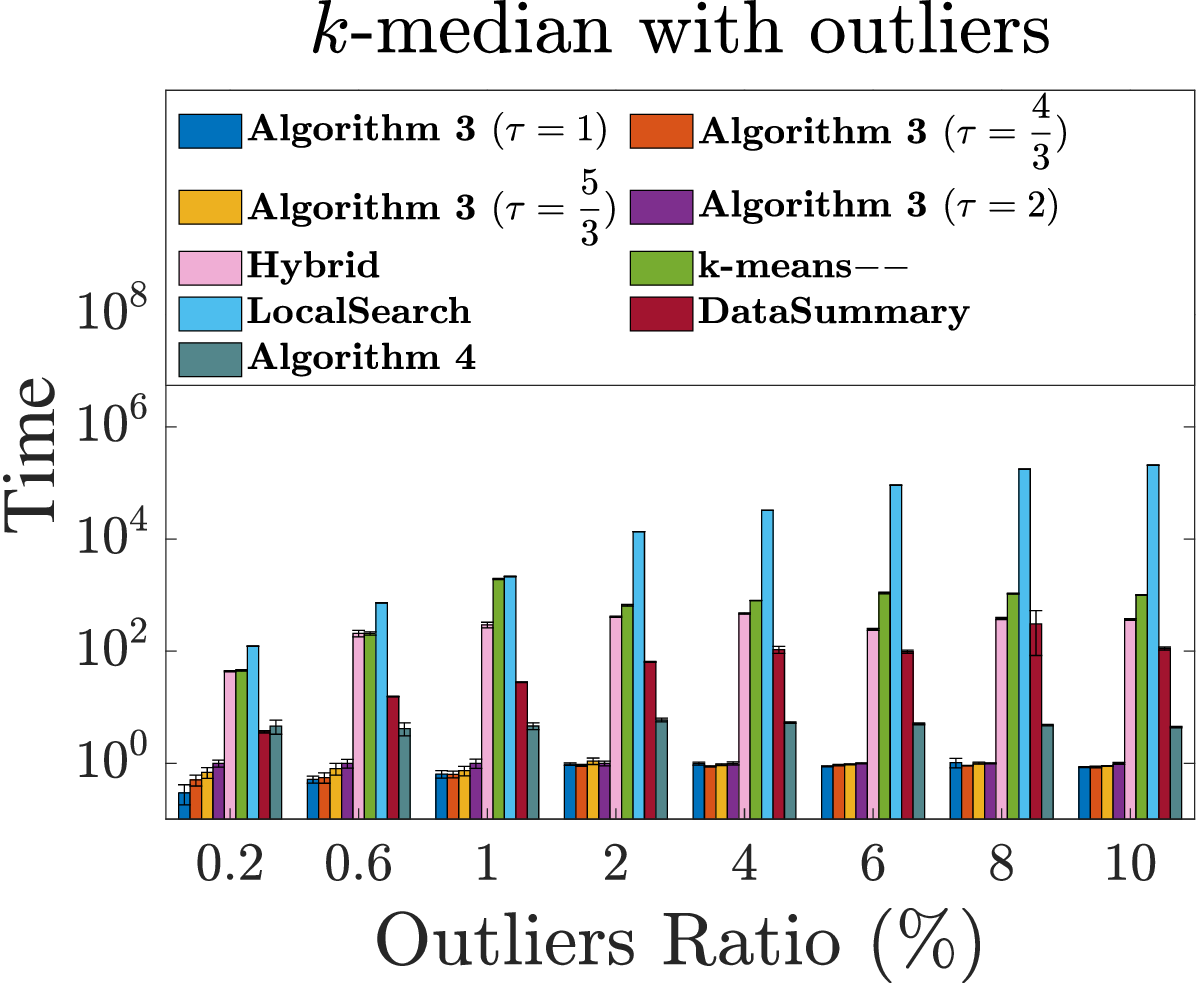} %
        \caption{}
    \end{subfigure}
    
		 \caption{The normalized results (objective values and running times) on the synthetic datasets with varying $z/n$ (we set $k=4$ and $\epsilon_1/\epsilon_2=2$).  We did not run  \textsc{Malkomes} for $z/n\geq4\%$, since it is very slow for large $z$. To illustrate the comparisons more clearly, we normalize all the results by dividing them over the results of Algorithm~\ref{alg-kc1} with $\tau = 2$ ({\em resp.,} Algorithm~\ref{alg-km} with $\tau = 2$ ) for $k$-center ({\em resp.,} $k$-means/median ) clustering with outliers. }     
		\label{fig-exp-outlier-sup}
	\end{center}
\end{figure*}

\FloatBarrier

\clearpage

\makeatletter
\renewcommand{\@biblabel}[1]{#1.}
\makeatother
\bibliographystyle{unsrtnat}

\bibliography{uniform_sampling_clustering3}

\end{document}